\documentclass[twoside,11pt]{article}

\usepackage[preprint]{jmlr2e}

\usepackage{graphicx} %
\usepackage[utf8]{inputenc} %
\usepackage[T1]{fontenc}    %
\usepackage{hyperref}       %
\usepackage{url}            %
\usepackage{booktabs}       %
\usepackage{amsfonts}       %
\usepackage{nicefrac}       %
\usepackage{amsmath}
\usepackage{bbold}
\usepackage{subfig}
\usepackage{amssymb}
\usepackage{appendix}
\usepackage{xcolor} %
\usepackage{algorithm}

\usepackage{appendix}

\newcommand{\op}{\left(}
\newcommand{\cp}{\right)}
\newcommand{\ob}{\left[}
\newcommand{\cb}{\right]}
\newcommand{\R}{\mathbb{R}}
\newcommand{\E}{\mathbb{E}}

\newcommand{\ep}{\varepsilon}

\newcommand{\N}{\mathcal{N}}
\renewcommand{\P}{\mathbb{P}}
\newcommand{\F}{\mathcal{F}}
\renewcommand{\L}{\mathcal{L}}

\newcommand{\td}{\widetilde}
\newcommand{\rv}[1]{\overleftarrow{#1}\!}
\newcommand{\OU}{\textnormal{OU}}

\DeclareMathOperator{\tr}{tr}

\DeclareMathOperator{\cov}{cov}

\usepackage{changes}

\usepackage{lastpage}
\jmlrheading{XX}{2025}{1-\pageref{LastPage}}{11/25; Revised XX}{XX}{XXXXX}{Eliot Beyler and Francis Bach}

\ShortHeadings{Convergence of Deterministic and Stochastic Diffusion-Model Samplers}{Beyler and Bach}
\firstpageno{1}

\begin{document}

\title{Convergence of Deterministic and Stochastic Diffusion-Model Samplers: A Simple Analysis in Wasserstein Distance}

\author{\name Eliot Beyler \email eliot.beyler@inria.fr \\
       \addr INRIA, Ecole Normale Supérieure\\
       PSL Research University\\
       Paris, France
       \AND
       \name Francis Bach \email francis.bach@inria.fr \\
       \addr INRIA, Ecole Normale Supérieure\\
       PSL Research University\\
       Paris, France}

\editor{XXX}

\maketitle

\begin{abstract}%
We provide new convergence guarantees in Wasserstein distance for diffusion-based generative models, covering both stochastic (DDPM-like) and deterministic (DDIM-like) sampling methods.
We introduce a simple framework to analyze discretization, initialization, and score estimation errors. Notably, we derive the first Wasserstein convergence bound for the Heun sampler and improve existing results for the Euler sampler of the probability flow ODE.
Our analysis emphasizes the importance of spatial regularity of the learned score function and argues for controlling the score error with respect to the true reverse process, in line with denoising score matching.
We also incorporate recent results on smoothed Wasserstein distances to sharpen initialization error bounds.
\end{abstract}

\begin{keywords}
Diffusion models, convergence guarantees, Wasserstein distance,  probability flow ODE, Heun method.
\end{keywords}

\section{Introduction}

Diffusion models \citep{sohl-dicksteinDeepUnsupervisedLearning2015,songGenerativeModelingEstimating2019,hoDenoisingDiffusionProbabilistic2020,songScoreBasedGenerativeModeling2021} are now the state-of-the-art methods for learning and sampling a distribution in high dimension, only known from a large data set of empirical samples.
Starting from a sample of a Gaussian distribution, they progressively refine it by following a differential equation involving the \emph{score function}, which is learned from empirical samples through a least-squares denoising objective, a method called \emph{denoising score matching} \citep{vincentConnectionScoreMatching2011}.
This process is either stochastic, with \emph{``DDPM-like''} samplers \citep{hoDenoisingDiffusionProbabilistic2020} that correspond to the Euler-Maruyama discretization of an SDE, or deterministic, with \emph{``DDIM-like''} samplers \citep{songDenoisingDiffusionImplicit2021} that correspond to the Euler discretization of an ODE. Deterministic sampling can be accelerated using higher order Runge–Kutta methods, the most popular of which is Heun scheme \citep{jolicoeur-martineauGottaGoFast2021,karrasElucidatingDesignSpace2022}.

\subsection{Convergence Guarantees for Diffusion Models}
The rising popularity of diffusion models and their impressive empirical performances have prompted a growing interest in providing theoretical convergence guarantees. Apart from early works \citep{bortoliConvergenceDenoisingDiffusion2022,kwonScorebasedGenerativeModeling2022} giving guarantees in Wasserstein distances but with limiting assumptions, most of the literature regarding DDPM-like stochastic samplers has used Kullback–Leibler (KL) divergence or total variation (TV) distance \citep{leeConvergenceScorebasedGenerative2022,leeConvergenceScorebasedGenerative2023,chenSamplingEasyLearning2023,chenImprovedAnalysisScorebased2023,bentonNearlyDlinearConvergence2024,liNonAsymptoticConvergenceDiffusionBased2024,confortiKLConvergenceGuarantees2025,strasmanAnalysisNoiseSchedule2024}. 

But more recently, convergence bounds in Wasserstein distance has gained attention.
First works used a limited framework (assuming a log-concave target distribution as done by \cite{tangContractiveDiffusionProbabilistic2024,brunoDiffusionbasedGenerativeModels2025,strasmanAnalysisNoiseSchedule2024,yuAdvancingWassersteinConvergence2025,gaoWassersteinConvergenceGuarantees2025,strasmanAnalysisNoiseSchedule2024}) or do not tackle the discretization error \citep{mimikos-stamatopoulosScorebasedGenerativeModels2024}.
More recent contributions tackle these limitations by working with weakly log-concave distributions \citep{gentiloni-silveriLogConcavityScoreRegularity2025} or semiconvexity assumptions on the data distribution and potentials with discontinuous gradients \citep{brunoWassersteinConvergenceScorebased2025a}.
Yet, all these works make an assumption on the learning error that do not correspond to what is minimized by the learning algorithm used to estimate the score.
Concurrently with our work, \cite{wangWassersteinBoundsGenerative2025} derived convergence guarantees for a stochastic diffusion-model sampler using the correct formulation of the score learning error.

On the other side, convergence guarantees for DDIM-like deterministic samplers is still more limited, though the number of results has grown recently. Early works gave bounds in TV/KL, with limitations such as having access to the true score \citep{chenRestorationDegradationLinearDiffusions2023} or adding a non deterministic corrector step \citep{chenProbabilityFlowODE2023}.
More recent works give convergence guarantees in TV without this restrictive assumptions \citep{liNonAsymptoticConvergenceDiffusionBased2024}, and also tackle the high order Runga-Kutta methods \citep{liAcceleratingConvergenceScoreBased2024,huangConvergenceAnalysisProbability2025,huangFastConvergenceHighOrder2025}.

However, for these deterministic samplers, results in Wasserstein distance remain limited, with \citet{gaoConvergenceAnalysisGeneral2025} that give results in a limited framework (log-concave distributions), and with additional assumptions (regularity in time of the score function), and do tackle higher order methods. Other works give Wasserstein convergence results, but for flow matching models, as \citet{bentonErrorBoundsFlow2024} (who do not tackle the discretization error) and \citet{dingCharacteristicLearningProvable2025}.

While all the works cited above consider the score learning error an exogenous factor and incorporate it into the bound, it should be noted that other studies provide statistical guarantees for this learning error \citep[][]{chenScoreApproximationEstimation2023,okoDiffusionModelsAre2023,azangulovConvergenceDiffusionModels2025,dingCharacteristicLearningProvable2025}. Such statistical guarantees will not be addressed in this work.

Recent works \citep{huraultDenoisingScoreMatching2025,pierretDiffusionModelsGaussian2025} also provide precise Wasserstein convergence estimates for Gaussian distributions, for which closed-form expressions can be derived.

\subsection{Wasserstein Distance}
They are many ways to access the similarity of two probability distributions, but we argue that the Wasserstein distance is the most suited one.
KL divergence and TV distance only depend on the ratio between densities of the probability distributions.
In particular, they are ill defined if the target distribution does not admit a density with respect to the Lebesgue measure, which would typically be the case under the \emph{manifold hypothesis} \citep[see, e.g.,][]{tenenbaumGlobalGeometricFramework2000,bengioRepresentationLearningReview2013,feffermanTestingManifoldHypothesis2016}. They also do not incorporate any notion of distance in the sampling space, therefore they cannot differentiate between a generated sample slightly outside the support of the data distribution and one far away. Note finally that the Wasserstein distance is connected to the Fréchet inception distance (FID) \citep{heuselGANsTrainedTwo2017} widely used in image generation.

\subsection{Contributions} In this work, we make the following contributions:
\begin{itemize}
\begin{samepage}
    \item We develop a simple framework to study the convergence of diffusion models, improving the state-of-the-art Wasserstein convergence guarantees for Euler discretization of the probability flow ODE and proving the first one for Heun sampler. 
    At the same time, it allows to get Wasserstein bounds for the Euler-Maruyama (DDPM) sampler of diffusion models, similar to the existing literature, but with simpler derivations.
\end{samepage}
    \item In particular, we discuss the assumption on the learning error commonly made in works on Wasserstein convergence guarantees, and the need for additional Lipchitz-continuity assumptions on the spatial regularity of the learned score function.
    \item We improve the control of initialization error in Wasserstein distance, using the result of \citet{chenAsymptoticsSmoothedWasserstein2022} on asymptotics of smoothed Wasserstein distances.
    \item We also prove convergence of order $1$ in the step size for the Euler-Maruyama (DDPM) sampler with accurate score, matching the optimal rate of convergence for the Euler-Maruyama discretization of SDEs with additive noise (i.e., constant diffusion coefficients).
\end{itemize}

\subsection{Notation}
For $Y\in \R^d$ a random variable, we denote $\L(Y)$ its distribution, and, when it exists, $p_Y$ its density with respect to the Lebesgue measure. 
For $Y$ a random variable with finite second order moment, we denote $\Vert Y\Vert_{L_2} = \op\E[\Vert Y\Vert^2]\cp^{1/2}$ its $L_2$-norm, where $\Vert\cdot\Vert$ is the Euclidean norm on $\R^d$. 
For $\mu, \nu$ two probability distributions on $\R^d$, the Wasserstein-$2$ distance is defined as
$$
 W_2(\mu,\nu) = \op\inf_{\Gamma}\int \Vert x_1 - x_2\Vert^2 d\Gamma(x_1,x_2)\cp^{1/2},
 $$
where $\Gamma$ is taken among all probability distributions on $\R^d\times\R^d$ with first marginal $\mu$ and second marginal $\nu$ \citep[see, e.g.,][]{peyreComputationalOptimalTransport2019}.
For $Y_1, Y_2$ two random variables, we denote $Y_1\bot Y_2$ if $Y_1$ and $Y_2$ are independent.
We also write $\nabla$ the gradient operator, $\nabla\cdot$ the divergence operator, $\Delta$ the Laplacian operator, always with respect to the space variable.
We denote $B(x,R)$ the closed ball of center $x$ and radius $R$.
For a matrix $A\in \R^{d\times d}$, we write $\Vert A\Vert_\textnormal{op}$ the operator norm of $A$, defined by $\Vert A \Vert_\textnormal{op} = \sup_{x\neq 0}\frac{\Vert Ax\Vert}{\Vert x \Vert}$, and $\Vert A\Vert_\textnormal{F}$ its Frobenius norm, defined by $\Vert A \Vert_\textnormal{F}^2 = \tr\ob AA^\top\cb$. 
We denote $\preccurlyeq$ the Loewner order on symmetric matrices ($A\preccurlyeq B$ if $B-A$ is positive semi-definite). 

Throughout the entire paper, $X$ we denote the random variable of interest, $X_t$ the forward noising process, $\rv X_t$ the backward (reverse) stochastic process, $x_t$ the reverse deterministic process following the probability flow ODE, and finally $\hat X, \hat X_n$ the empirical outputs and steps of the sampling algorithms.

\section{Preliminaries: Algorithms}
\subsection{Defining Diffusion Processes}
There are several ways to define diffusion models, which correspond to different time parameterizations and scalings. Here, we use the simplest one, for which the diffusion process corresponds simply to the constant-speed heat equation \citep[see, e.g.,][]{evansPartialDifferentialEquations2022}.
This choice makes the expression of the forward (\ref{eq:sde_forward}) and backward (\ref{eq:reverse_SDE}-\ref{eq:reverse_ODE}) processes much simpler, and other choices of time parameterizations can be retrieved by a change of time and rescaling.
We discuss the impact of noise schedules and how our analysis can be adapted to a more general setting in Appendix~\ref{sct:param_time}.

Starting from the random variable of interest $X$, we progressively add Gaussian noise to it with, for $t\in\R_+$:
\begin{equation*}
X_t = X + B_t,
\end{equation*}
where $B_t$ is a Brownian motion (in particular, the marginal distribution is $B_t \sim \N(0,tI)$). We will denote $p_t = p_{X_t}$ the density of $X_t$. $X_t$ verifies the following SDE:
\begin{equation}
\label{eq:sde_forward}
\left\{
\begin{array}{l}
dX_t =  dB_t, \\
X_0 = X.
\end{array}\right.
\end{equation}
 
The idea is that for large time $T$, we will have $X_T \approx B_T$, which is easy to sample from. Then we can go back to $X_0 = X$ through a reverse process given by the following proposition \citep{songScoreBasedGenerativeModeling2021}, a special case of a result by \citet{andersonReversetimeDiffusionEquation1982}.

\begin{proposition}
\label{prop:reverse_SDE}
We define a process $\rv X_t$, for $t \in [0,T]$, with the following stochastic differential equation:
\begin{equation}
\label{eq:reverse_SDE}
\left\{
\begin{array}{l}
d\rv{X}_t = \nabla\log \rv p_{t}(\rv{X}_t)dt + dW_t, \\
\rv{X}_0 = X_T,
\end{array}\right.
\end{equation}
where $\rv p_t = p_{T-t}$ and $W_t$ is a Brownian motion. Then, $\rv{X}_t$ has the same marginal distributions as $X_{T-t}$, i.e., $\forall t \in [0,T], \L(\rv{X}_t) = \L( X_{T-t})$.
\end{proposition}

We can also define a reverse ODE for the forward process, also known as the \emph{probability flow ODE} \citep{songScoreBasedGenerativeModeling2021}.

\begin{proposition}
\label{prop:reverse_ODE}
We define a process $(x_t)_{t\in[0,T]}$ by
\begin{equation}
\label{eq:reverse_ODE}
\left\{\begin{array}{rl}
    \frac{d x_t}{dt} &= \frac{1}{2} \nabla \log \rv p_{t}(x_t),\\
    x_{0} &= X_{T}.
\end{array}\right.
\end{equation}
where $\rv p_t = p_{T-t}$. Then, $x_t$ has the same marginal distributions as $X_{T-t}$, i.e., $\forall t \in [0,T], \L(x_t)= \L(X_{T-t})$.
\end{proposition}
For this process, the only source of randomness lies in the initialization $X_T$. It is worth noting that there exists a continuum of reverse-time dynamics interpolating between the SDE and ODE formulations \citep{huangVariationalPerspectiveDiffusionBased2021,zhangDiffusionNormalizingFlow2021}, given for any $\lambda \geq 0$ by
$$
d\rv{X}_t = \frac{1+\lambda}{2}\nabla\log \rv p_{t}(\rv{X}_t)dt + \sqrt{\lambda}dW_t.
$$
The theoretical results presented in this work can be readily extended to this family of processes.

\subsection{Euler Discretization and Sampling Algorithms} Following \citet{songScoreBasedGenerativeModeling2021}, we obtain sampling algorithms by discretizing the SDE (\ref{eq:reverse_SDE}) or the ODE (\ref{eq:reverse_ODE}).
This requires knowing the \emph{score} $\nabla\log p_t$. In practice, it is learned with a neural network $s_\theta: \R\times\R^d\rightarrow\R^d$ of parameter $\theta$ using denoising score matching \citep{vincentConnectionScoreMatching2011}, which implicitly minimizes $\Vert s_\theta(t,X_t) - \nabla\log p_t (X_t)\Vert_{L_2}$.

Then, we fix a large time $T$ such that $X_T \approx B_T$ and a number of sampling steps $N$.
To avoid possible irregularity, an early stopping time $\epsilon\geq 0$ is added.
We write $h = \frac{T-\epsilon}{N}$ the step size, and for $n=0,\dots, N$, $t_n = hn$.
We start by sampling $\hat{X}_0 \sim \N(0,TI)$, then discretizing~(\ref{eq:reverse_SDE}) with the Euler–Maruyama method gives Algorithm~\ref{alg:SDE} (similar to DDPM from \citet{hoDenoisingDiffusionProbabilistic2020}).
Similarly, we get Algorithm~\ref{alg:ODE} (similar to DDIM from \citet{songDenoisingDiffusionImplicit2021}) by discretizing~(\ref{eq:reverse_ODE}) with the Euler method. As (\ref{eq:reverse_ODE}) is an ODE, it is common to use higher-order Runge–Kutta methods, and in particular Heun second-order method \citep{jolicoeur-martineauGottaGoFast2021,karrasElucidatingDesignSpace2022}, to accelerate convergence with respect to the step size. This is summed up in Algorithm~\ref{alg:Heun}.

\noindent\begin{minipage}{0.49\textwidth}
\begin{algorithm}[H]
\caption{Sampling (SDE -- Euler-Maruyama)}
\label{alg:SDE}
\textbf{Initialization:} $s_\theta, T, N, \epsilon$

\textbf{Set:} $h = \frac{T-\epsilon}{N}$, and for $n = 0, \dots, N$, $t_n = nh$

\quad--  $\hat{X_0} \leftarrow \N(0,T\cdot I)$

\textbf{For $n = 1, \dots, N$:}

\quad-- $\hat{X}_n \leftarrow \hat{X}_{n-1} + h s_\theta(T-t_{n-1},\hat{X}_{n-1})+\N(0,hI)$

\textbf{Return:} $\hat{X}_N$
\end{algorithm}
\end{minipage}
\hfill
\begin{minipage}{0.49\textwidth}
\begin{algorithm}[H]
\caption{Sampling (ODE -- Euler)}
\label{alg:ODE}
\textbf{Initialization:} $s_\theta, T, N,\epsilon$

\textbf{Set:} $h = \frac{T-\epsilon}{N}$, and for $n = 0, \dots, N$, $t_n = nh$

\quad--  $\hat{X_0} \leftarrow \N(0,T\cdot I)$

\textbf{For $n = 1, \dots, N$:}

\quad-- $\hat{X}_n \leftarrow \hat{X}_{n-1} + \frac{h}{2}s_\theta(T-t_{n-1},\hat{X}_{n-1})$

\textbf{Return:} $\hat{X}_N$
\end{algorithm}
\end{minipage}

\begin{center}
\begin{minipage}{0.65\textwidth}
\begin{algorithm}[H]
\caption{Sampling (ODE -- Heun)}
\label{alg:Heun}
\textbf{Initialization:} $s_\theta, T, N,\epsilon$

\textbf{Set:} $h = \frac{T-\epsilon}{N}$, and for $n = 0, \dots, N$, $t_n = nh$

\quad--  $\hat{X_0} \leftarrow \N(0,T\cdot I)$

\textbf{For $n = 1, \dots, N$:}

\quad-- $\hat{Y}_n \leftarrow  \hat{X}_{n-1} + \frac{h}{2} s_\theta(T-t_{n-1},\hat{X}_{n-1})$

\quad-- $\hat{X}_n \leftarrow  \hat{X}_{n-1} + \frac{h}{4} \op s_\theta(T-t_{n-1},\hat{X}_{n-1}) + s_\theta(T-t_{n},\hat{Y}_{n})\cp$

\textbf{Return:} $\hat{X}_N$
\end{algorithm}
\end{minipage}
\end{center}

\section{Controlling the Different Sources of Error in Diffusion Models}

We work in the following framework:
\begin{center}
\begin{minipage}{0.65\textwidth}
\begin{itemize}
\item[\textbf{Assumption 1.}] The target distribution has support in $B(0,R)$, i.e., $X\in B(0,R)$ almost surely.
\end{itemize}
\end{minipage}
\end{center}

Assuming a bounded support allows to control conditional moments of the probability distribution, that appear when computing the score $\nabla \log p_t$ and its derivatives with respect to space and time. Beside this, we make no other assumption on the regularity of the target distribution.
In particular, our framework can be applied under the \emph{manifold hypothesis}, for which the data lies on a low-dimensional manifold and does not admit a density with respect to the Lebesgue measure.
Our result can also be used to tackle a more general framework, proposed by \citet{saremiChainLogConcaveMarkov2023}, also used by \citet{greniouxStochasticLocalizationIterative2024} and \citet{dingCharacteristicLearningProvable2025}:
\begin{center}
\begin{minipage}{0.65\textwidth}
\begin{itemize}
\item[\textbf{Assumption 1'.}] There exists a random variable $Z$ and $\tau >0$ such that $Z\in B(0,R)$ almost surely and ${X = Z + \N(0,\tau I).}$
\end{itemize}
\end{minipage}
\end{center}
With this assumption, early stopping is not needed. Considering a process $Z_t$ defined by (\ref{eq:sde_forward}) starting from $Z$ rather that $X$, we have $Z_\tau = X$. 
Running the sampling algorithm for data distribution $X$ without early stopping ($\epsilon = 0$) corresponds to using early stopping  with $\epsilon = \tau$ for $Z_t$, as $X_0 = Z_\tau = X$. Note that we introduce the process $Z_t$ only for the purpose of theoretical analysis, while the algorithm is in fact run on $X_t$ with $\epsilon = 0$.

There are different sources of error that arises from approximating the true reverse processes (\ref{eq:reverse_SDE}) and (\ref{eq:reverse_ODE}) by Algorithms~\ref{alg:SDE}-\ref{alg:Heun}: the \emph{discretization error}, the \emph{initialization error} (using $B_T$ instead of $X_T$), the \emph{early stopping error} (stopping the reverse process at time $T-\epsilon$ instead of $T$), the \emph{score approximation error} (using $s_\theta$ instead of $\nabla\log p_t$). Moreover, the propagation of errors from previous steps to subsequent ones needs to be controlled. In this section, we present results addressing these different types of error, except for the score approximation, which we incorporate as an an exogenous factor and which is discussed in Section~\ref{sct:discussion_learning_score}.

\subsection{Control of the Spatial Regularity of the Score and Propagation of Errors}
We want to control the spatial regularity of the score, which plays an important role in the control of the propagation of errors.
\begin{lemma}
\label{lemma:bound_spatial_regularity}
Under Assumption 1, for all $t>0$ and $x\in \R^d$, we have
\begin{equation}
\label{eq:bound_hessian}
- \frac{1}{t} I\preccurlyeq  \nabla^2\log p_t(x) \preccurlyeq  \op - \frac{1}{t} + \frac{R^2}{t^2}\cp I.
\end{equation}
In particular, $\Vert\nabla^2\log p_t(x)\Vert_\textnormal{op} \leq 
C_t = \max \op \frac{1}{t}, \left|\frac{R^2}{t^2}- \frac{1}{t}\right|\cp.$ Moreover, denoting $f_{t,h}: x \mapsto x + h\nabla\log p_t(x)$, for $h\leq t$, $f_{t,h}$ is $L_{t,h}$-Lipchitz-continuous, with
\begin{equation}
\label{eq:Lipchitz_constant}
L_{t,h} = 1 + h\op\frac{R^2}{t^2} - \frac{1}{t}\cp.
\end{equation}
In particular, for $t> t^* = R^2$, $f_{t,h}$ is contractive ($L_{t,h} < 1$).
\end{lemma}

From these bounds, we can distinguish three different time regimes, as illustrated in Figure~\ref{fig:key_times_t}.
Near time $t=0$, without regularity assumption on the distribution of $X$, we cannot control the regularity of $\nabla^2 \log p_t$ and the bounds diverge. 
The \emph{early stopping} time $\epsilon$ is introduced to circumvent this issue. 
Then for $t> t^* =R^2$, $\frac{R^2}{t^2} - \frac{1}{t} <0$, meaning that all eigenvalues of $\nabla^2\log p_t(x)$ are negative, hence $p_t$ is strongly log-concave. 
This also corresponds to a change in regime for $f_{t,h}$, the function by which errors are propagated, which becomes contractive ($L_{t,h}<1$). We use this observation to limit the accumulation of error, as done by \citet{gentiloni-silveriLogConcavityScoreRegularity2025}, with exact formulas to bound how these coefficients multiply given in Appendix~\ref{sct:proofs:bound_propagation_error}.

\begin{figure}
    \centering
	\includegraphics[width=\linewidth]{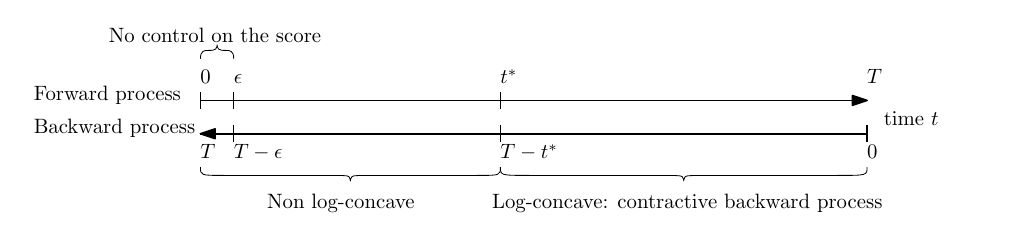}
    \vspace{-.8cm}
	\caption{Illustration of the 3 different characteristic times in the diffusion process.}
	\label{fig:key_times_t}
\end{figure}

\textbf{Remark.} In many analysis of diffusion models in Wasserstein distance \citep{tangContractiveDiffusionProbabilistic2024,brunoDiffusionbasedGenerativeModels2025,strasmanAnalysisNoiseSchedule2024,yuAdvancingWassersteinConvergence2025,gaoWassersteinConvergenceGuarantees2025} authors make the assumption that the target distribution is log-concave, hence they are only looking at the part of the diffusion process for $t\geq t^*$, which is the easy part of sampling as the backward process is contractive. 

\textbf{Tightness.} A toy example \citep[see, e.g.,][]{saremiChainLogConcaveMarkov2023} is to take $d=1$ and $X$ a mixture of two Dirac masses, $X = \frac{\delta_{-R} + \delta_{-R}}{2}$. In particular $X \in B(0,R)$ almost surely, and we can compute
$$
\nabla^2 \log p_t(x) = -\frac{1}{t} + \frac{R^2}{t^2}\op\cosh\op\frac{xR}{t}\cp\cp^{-2}.
$$
For $x=0$, we get $\nabla^2 \log p_t(0) = -\frac{1}{t} + \frac{R^2}{t^2}$ and for $x\rightarrow\infty$, $\nabla^2 \log p_t(x) \rightarrow -\frac{1}{t}$, hence all bounds in Lemma~\ref{lemma:bound_spatial_regularity} are tight.

\subsection{Control of the Discretization Error}
\label{sct:discretization_error}

In this section, we take a look at the error induced by the discretization of the continuous processes (\ref{eq:reverse_SDE}) and (\ref{eq:reverse_ODE}). Note that related work on Wasserstein guarantees for the probability flow ODE \citep[][Assumption 3.2]{gaoConvergenceAnalysisGeneral2025} assumes that the score is Lipchitz-continuous with respect to time. We believe that such an assumption, in addition to being unverifiable, is not needed as it can be directly deduced from assumptions on the target distribution.

Integrating the reverse ODE (\ref{eq:reverse_ODE}) between $t$ and $t+h$ gives 
$$
x_{t+h} = x_t +  \frac{1}{2}\int_{t}^{t+h} \nabla\log \rv p_s(x_s)ds.
$$
The Euler discretization replaces the integral by its approximation by Euler method, i.e., for one step, $h\nabla\log \rv p_{t}(x_{t})$.
The corresponding error is controlled by Lemma~\ref{lemma:bound_discretization_ODE} and scales as $O(h^2)$ similar to the deterministic case.
\begin{lemma}
\label{lemma:bound_discretization_ODE}
Under Assumption 1, for $\epsilon\leq R^2$, and $t,h\geq 0$, $t+h\leq T - \epsilon$, we have
$$
\left\Vert \int_{t}^{t+h} \nabla\log \rv p_s(x_s)ds - h\nabla\log \rv p_{t}(x_{t}) \right\Vert_{L_2} \leq2\sqrt{d}\frac{R^3}{\epsilon^{3/2}}h \int_{t}^{t+h} \frac{1}{(T-u)^{3/2}} du.
$$
\end{lemma}

Similarly, we tackle Heun discretization by approximating the integral with the trapezoidal rule,  i.e., for one step, $\frac{h}{2}\op\nabla\log \rv p_{t}(x_{t})+\nabla\log \rv p_{t+h}(x_{t+h})\cp$.
The corresponding error is controlled by Lemma~\ref{lemma:bound_discretization_Heun}. Once again, the dependency in $O(h^3)$ is the same as the deterministic case, and gives a better convergence rate in the step size. However, it involves controlling derivatives of higher order,  hence a larger multiplicative constant.
\begin{lemma}
\label{lemma:bound_discretization_Heun}
Under Assumption 1, for $\epsilon\leq R^2$, and $t,h\geq 0$, $t+h\leq T - \epsilon$, we have
\begin{multline*}
\left\Vert \int_{t}^{t+h} \nabla\log \rv p_s(x_s)ds - \frac{h}{2}\op\nabla\log \rv p_{t}(x_{t})+\nabla\log \rv p_{t+h}(x_{t+h})\cp \right\Vert_{L_2} \\\leq 33d\frac{R^5}{\epsilon^{5/2}}h^2\int_{t}^{t+h} \frac{1}{(T-u)^{5/2}} du. 
\end{multline*}
\end{lemma}

Finally, integrating the reverse SDE (\ref{eq:reverse_SDE}) between $t$ and $t+h$ gives 
$$
\rv X_{t+h} = \rv X_t + \int_t^{t+h}  \nabla\log\rv{p}_s(\rv X_s)ds + \underbrace{\int_t^{t+h} dW_s}_{\sim \N(0,hI)}.
$$
The Euler-Maruyama discretization approximates the first integral above with Euler method, leading to 
$$
\hat X_{t+h} = \rv X_t + h\nabla\log\rv{p}_t(\rv X_t) + Z,
$$
with $Z \sim  \N(0,hI)$. Here we are free to choose a representation of $Z$ to get a particular coupling between the true process and its approximation. We take $Z = \int_t^{t+h} dW_s\sim \N(0,hI)$ such that the Gaussian noises cancel out when computing the difference. This leads to the error term controlled by Lemma~\ref{lemma:bound_discretization_SDE}.
\begin{lemma}
\label{lemma:bound_discretization_SDE}
Under Assumption 1, for $\epsilon\leq R^2$, and  $t,h>0$, $t + h \leq T-\epsilon$, we have
$$
 \int_t^{t+h} \nabla\log\rv{p}_s(\rv X_s)ds - h\nabla\log\rv{p}_t(\rv X_t)= \int_t^{t+h}\int_t^s \nabla^2\log \rv{p}_u(\rv X_u)\cdot dW_u ds,
$$
and
$$
\left\Vert \int_t^{t+h} \nabla\log\rv{p}_s(\rv X_s )ds - h\nabla\log\rv{p}_t(\rv X_t)\right\Vert_{L_2}\leq \frac{\sqrt{2d}}{2}\frac{R^2}{\epsilon} \sqrt{h} \int_t^{t+h}\frac{1}{T-s}ds.
$$
\end{lemma}

\subsection{Control of Initialization Error}
For the control of the initialization error, we start by giving a result in a more general setting, that can be applied to other time parametrizations of the diffusion process found in the literature.
\begin{proposition}
\label{prop:init_error}
Let $Y = \alpha X + \beta Z$, with $\alpha,\beta > 0$, $Z\sim \N(0,I)$ and $\hat{Y} = \gamma Z'$ with  $\gamma \geq \beta$ and $Z'\sim \N(0,I)$. Then 
$$
W_2(\L(Y),\L(\hat{Y})) \leq \alpha W_2\op \L(X),\N\op 0,\frac{\gamma^2 -\beta^2}{\alpha^2}I\cp\cp.
$$
Assume moreover that $\E[X] = 0$, and that for some some $\xi>0$, $\E\ob e^{\xi X^2}\cb < \infty$\footnote{This second technical assumption ensures that $X$ has sufficiently light tails, which is immediately verified under Assumption 1 or~1'.}, then we have the following asymptotic behavior, as $\beta/\alpha \rightarrow+\infty$,
$$
W_2(\L(Y),\L(\hat{Y})) \sim  \frac{\alpha^2}{2\beta} \left\Vert \Sigma - \frac{\gamma^2 -\beta^2}{\alpha^2}I\right\Vert_\textnormal{F},
$$
with $\Sigma = \E\ob XX^\top\cb$.
\end{proposition}

All other works on Wasserstein convergence bounds for diffusion models use a bound on initialization error derived form the first one of Proposition~\ref{prop:init_error}, and would get a better dependency in $T$  by using the second one. Indeed we use a result by \citet{chenAsymptoticsSmoothedWasserstein2022} on the asymptotic of Wasserstein distances for smoothed densities, and to our knowledge, the use of this result in the diffusion model literature is new\footnote{We only found a reference by \citet{reevesInformationTheoreticProofsDiffusion2025}, where the authors use a result of the same paper on the KL divergence, rather than on the Wasserstein distance.}.

In our setting, Proposition~\ref{prop:init_error} gives the following control on the initialization error:
\begin{corollary}
\label{cor:init_error}
Assume that $\Vert X \Vert_{L_2}<\infty$, then for $\hat{X}_0 \sim \N(0,TI)$, we have 
$$
W_2(\L(X_T),\L(\hat{X}_0)) \leq W_2(\L(X), \delta_0) = \Vert X\Vert_{L_2}.
$$
Suppose moreover that Assumption 1 holds and $\E[X] = 0$, then we have the following asymptotic behavior, as $T\rightarrow+\infty$,
$$
W_2(\L(X_T),\L(\hat{X}_0)) \sim \frac{1}{2} \frac{\Vert \Sigma\Vert_\textnormal{F}}{\sqrt{T}},
$$
with $\Sigma = \E\ob XX^\top\cb$. In particular, for $T$ large enough (depending only on $\L(X)$\textnormal{)}, we have 
$$W_2(\L(X_T),\L(\hat{X}_0)) \leq \frac{R^2}{\sqrt{T}}.$$
\end{corollary}

\subsection{Control of Early Stopping Error}
We can finally control the early stopping error with the following lemma.
\begin{lemma}
\label{lemma:bound_early_stopping_error}
For $\epsilon\geq0$, we have
$$
W_2(\L(X),\L(X_\epsilon)) \leq \sqrt{d\epsilon} .
$$
\end{lemma}

\section{Sketch of the Proof Strategy and Discussion on Score Error Assumptions}
\label{sct:discussion_learning_score}

The strategy of the proof is to define an initial coupling between $X_T$, the initialization of the exact reverse process, and $\hat X_0$, the initialization of the algorithm, then to follow at each step how the error evolves in $L_2$ distance.

\begin{figure}
    \centering
    \subfloat[][]{
    \includegraphics[width=.48\linewidth]{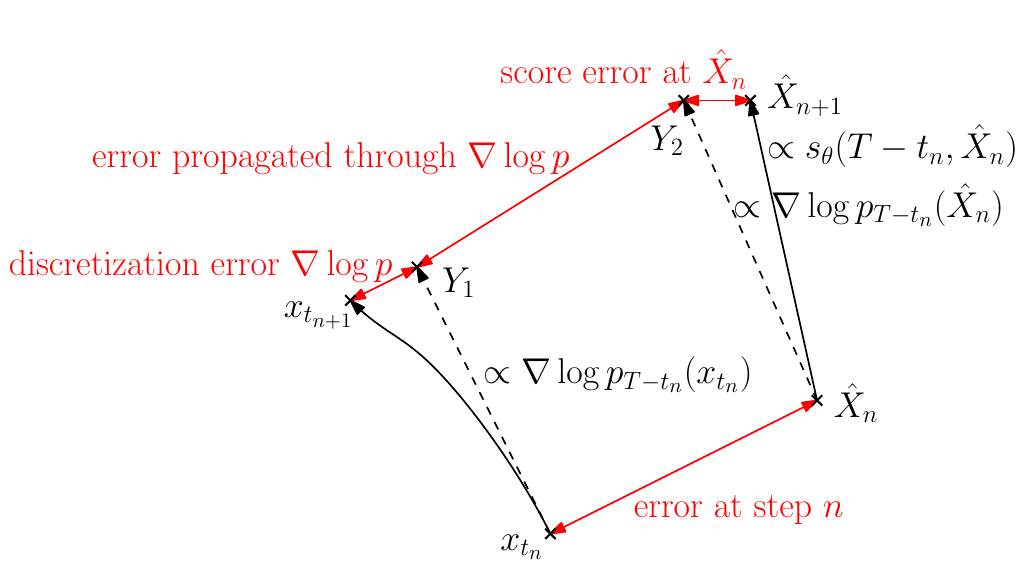}
    }
    \subfloat[][]{
    \includegraphics[width=.48\linewidth]{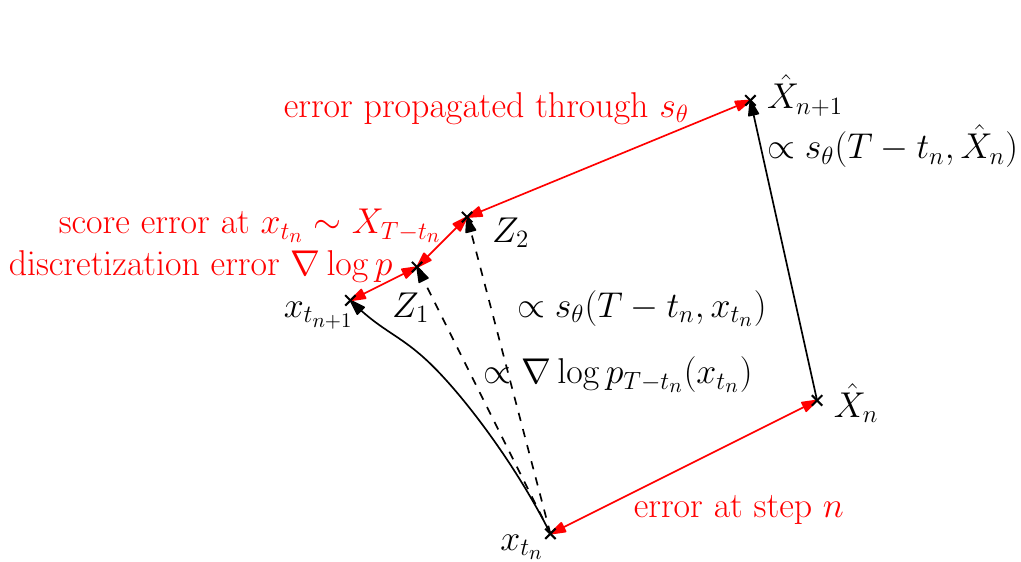}
    }
    \caption{Illustration of two different ways to decompose error at each step.}
	\label{fig:error_decomp}
\end{figure}

At a step $n+1$, the error can be decomposed as the discretization error, and the score error and the error propagated from the previous step $n$. To illustrate, we give a decomposition for the Euler discretization of the ODE in Figure~\ref{fig:error_decomp}(a). It introduces two intermediate points
$$
Y_1 = x_{t_{n}} + \frac{h}{2}\nabla\log \rv p_{t_{n}}(x_{t_{n}}), \quad Y_2 = \hat X_{{n}} + \frac{h}{2}\nabla\log \rv p_{t_{n}}(\hat X_{{n}}).
$$
This leads to the following decomposition of error:
\begin{align*}
x_{t_{n+1}} &- \hat X _{n+1} = \\
&\quad(x_{t_n} - Y_1) &&\text{discretization error of $\nabla\log\rv p_t$,}\\
&+ (Y_1 -Y_2) &&\text{error propagated from step $n-1$ through $I + \frac{h}{2} \nabla\log\rv p_{t_{n-1}}$,}\\
&+ (Y_2 - \hat X _n) &&\text{score approximation error \textbf{evaluated at the empirical process $\hat X_{n}$.}}
\end{align*}
This decomposition of error, used in earlier works on convergence of diffusion models in Wasserstein distance  \citep{gentiloni-silveriLogConcavityScoreRegularity2025,gaoWassersteinConvergenceGuarantees2025,gaoConvergenceAnalysisGeneral2025,yuAdvancingWassersteinConvergence2025}, has the advantage that we can control the regularity of $\nabla \log p_t$ in time $t$ and space $x$, hence we can control the discretization error (Lemmas~\ref{lemma:bound_discretization_ODE}-\ref{lemma:bound_discretization_SDE}), and the propagation  of error from step $n-1$ to step $n$ (Lemma~\ref{lemma:bound_spatial_regularity}).
In particular, we do not have to make any regularity assumption on the learned score $s_\theta$. We only need to control its $L_2$-error with respect to the true score $\nabla\log p_t$ \textbf{evaluated at the \emph{empirical process} $\hat X_n$}:
\begin{equation}
\label{eq:hyp_score_empirical}
\Vert \nabla \log p_{t_n}(\hat{X}_n) - s_\theta(t_n, \hat{X}_n)\Vert_{L_2} \leq \ep_{\text{score}},
\end{equation}
for $n = 0,\dots, N$.
We believe that this choice, although it simplifies the proofs, is not the right one, as $\hat X_n$ depends on $s_\theta$ itself, and as it does not reflect the error that is minimized implicitly during learning.\footnote{
\citet{brunoWassersteinConvergenceScorebased2025a} use a different decomposition of error.
They introduce an auxiliary process of the form 
$d\rv X^\theta_t = s_\theta(T-t,\rv X^\theta_t)dt + dW_t$ and evaluate the score approximation error with respect to this process $\left\Vert \nabla \log \rv p_{t_n}(\rv X_{t_n}^\theta) - s_\theta(T-t_n,\rv X_{t_n}^\theta)\right\Vert_{L_2}$.
Similarly to (\ref{eq:hyp_score_empirical}) used in other works, this control of the learning error does not correspond to what is minimized during training.
}
The error that is implicitly minimized by denoising score matching is \textbf{the error evaluated at the \emph{true process} $X_t$} \citep[][]{vincentConnectionScoreMatching2011}, hence the appropriate choice is to make an hypothesis of the form:
\begin{equation}
\label{eq:hyp_score_true_process}
\left\Vert \nabla \log p_{t_n}(X_{t_n}) - s_\theta(t_n,X_{t_n})\right\Vert_{L_2} \leq \ep_{\text{score}}.
\end{equation}
Moreover, the first decomposition hides the fact that regularity assumptions on $s_\theta$ are really needed, as shown in Section~\ref{sct:lipschitz_assumption}.

In this work, we prefer to use another decomposition, illustrated in Figure~\ref{fig:error_decomp}(b), which introduces  two different intermediate points
$$
Z_1 = x_{t_{n}} + \frac{h}{2}\nabla\log\rv p_{t_{n}}( x_{t_{n}}), \quad Z_2 = x_{t_{n}} + \frac{h}{2}s_\theta(T-t_{n},x_{t_{n}}).
$$
Then
\begin{align*}
&x_{t_n} - \hat X _n =\\
&\quad (x_{t_n} - Z_1) &&\text{discretization error of $\nabla\log\rv p_t$,}\\
&+ (Z_1 -Z_2) &&\text{score approximation error \textbf{evaluated at the true reverse process $x_{t_{n-1}}$},}\\
&+ (Z_2 - \hat X _n) &&\text{error propagated from step $n-1$ through $I + \frac{h}{2} s_\theta(T-t_{n-1},\cdot)$.}
\end{align*}
This second decomposition allows to control the $L_2$-error on the learned score \textbf{evaluated at the true (reverse) process $x_t$}, in line with denoising score matching.
However, as the error from step $n-1$ is propagated to step $n$ through $I + \frac{h}{2} s_\theta(T-t_{n-1},\cdot)$, it means that we need to control the spatial regularity to of $s_\theta$. The propagation of error through  $I + \frac{h}{2} \nabla\log\rv p_{t}$ was controlled by Lemma~\ref{lemma:bound_spatial_regularity}, hence it is natural to make the hypothesis that $s_\theta$ verifies the same properties. 

\begin{center}
\begin{minipage}{0.7\textwidth}
\begin{itemize}
\item[\textbf{Assumption 2.}] For $t >0$, $x \mapsto s_\theta(t,x)$ is $C_{t}$-Lipchitz with 
$$
C_t = \max\op\frac{1}{t}, \left|\frac{R^2}{t^2}-\frac{1}{t}\right|\cp,
$$
and for all $h\leq \epsilon \leq t$, the mapping $x \mapsto x + h s_\theta(t,x)$ is $L_{t,h}$-Lipchitz with 
$$L_{t,h} = 1 + h\op\frac{R^2}{t^2} - \frac{1}{t}\cp.$$
\end{itemize}
\end{minipage}
\end{center}

As $s_\theta$ approximates $\nabla\log p_t$, it is reasonable to assume that we can get the same kind of spatial regularity.\footnote{We could assume that we have a slightly weaker Lipchitz constant of the form $L_{t,h} = 1 + h\op\frac{R^2}{t^2} - \frac{1}{t} + \delta(t)\cp$ for some $t \mapsto \delta(t)$. It would add an additional multiplicative term $\exp\op\int_\epsilon^T\delta(t)dt\cp$ in the accumulation of error.}
However, we stress the fact that it is not enforced by the training objective.
Still, we believe that is not an artifact of the proof, and that this kind of regularity is needed to get good convergence of diffusion models.
We believe that future work should be dedicated to exploring in more details which hypotheses on the regularity of $s_\theta$ are needed and whether it is reasonable to think they are enforced in practice by biases in the network architecture and learning algorithm.

Finally, note that \citet{wangWassersteinBoundsGenerative2025} use an approach similar to us, and give Wasserstein convergence guarantees for the Euler-Maruyama discretization of the SDE (\ref{eq:reverse_SDE}). Their Assumption 4 that is similar to our Assumption 2, and they likewise control the score error with respect to the true reverse process. 
Although we pointed out in a first version of this work an error in their original proof that invalidated their bounds, the authors have since released an updated version containing a corrected argument.

\subsection{Why are Lipchitz-continuity Assumptions Important?}
\label{sct:lipschitz_assumption}

In addition to not corresponding to the minimization done in practice, the use of assumptions of the form (\ref{eq:hyp_score_empirical})  and decomposition of error (a) hide the fact that spatial regularity of $s_\theta$ is needed. Indeed, the process $(\hat{X}_n)_n$ is the discretization of the continuous SDE,
$$
d\rv X^\theta_t = s_\theta(T-t,\rv X^\theta_t)dt + dW_t,
$$
or ODE,
$$
\frac{dx^\theta_t}{dt} = \frac{1}{2}s_\theta(T-t,x^\theta_t).
$$
Thus, as the step size decreases, the process $(\hat{X}_n)_n$ can be expected to converge to the continuous paths $(\rv X^\theta_t)_t$ or $(x^\theta_t)_t$, yet these processes are not necessarily finite.

\subsubsection{Explosion of SDEs and ODEs with Non-Lipchitz Drift} ODEs, and SDEs, can explode in finite time if the drift is only locally Lipchitz-continuous and not globally.
$L_2$ control on the score error is not enough to prevent this from happening. Indeed, we can find $L_2$-approximation of the score that are not globally Lipchitz-continuous and for which an explosion occurs. Take for some $\alpha > 0$, $s(t,x) = \nabla\log p_t (x) + \alpha \Vert x\Vert x$. We have,
$$
\Vert s(t,X_t) - \nabla\log p_t (X_t)  \Vert_{L_2} = \alpha \E\ob \Vert X_t\Vert^4\cb^{1/2},
$$
which can be made as small as desired. However, in Appendix~\ref{sct:explosion_ODE} (Proposition~\ref{prop:explosion_ODE}), we prove that the solution to the ODE
$$
\frac{d\rv x^s_t}{dt} = \frac{1}{2}s(t,\rv x^s_t),
$$
explodes in finite time with non-zero probability. More precisely, there exists a random stopping time $\tau \in [0,\infty]$, such that $\P(\tau < \infty)>0$ and for $\tau < \infty$, $\Vert x_t\Vert \xrightarrow[t\rightarrow \tau^-]{} \infty$ almost surely.
Moreover, it verifies that for all $\delta > 0,$ $\P(\tau \leq \delta) > 0$, hence the explosion can happen arbitrarily close to time $t=0$ with non-zero probability, in particular before the stopping time of the reverse process at time $T-\epsilon$.
The proof is based on the fact that $\nabla\log p_t (x)$ exhibits a linear growth in $x$, hence the explosion phenomena due to the quadratic term dominates if the initialization is large enough, which is always the case with non-zero probability as $X_T$ has positive density over $\R^d$.

This explosion means that the process takes the value $\infty$ with non zero probability, in particular all its moments are infinite.
It is known that the same explosion phenomenon occurs for SDE with a quadratic drift term such as $dX_t = X_t^2 dt + dB_t,$ \citep[see, e.g.,][]{karatzasBrownianMotionStochastic1998,bonderContinuityExplosionTime2009}.
Therefore, we conjecture that the same phenomenon could also occur with the reverse SDE of diffusion models if Lipchitz-continuity assumptions on the score network $s_\theta$ are not made.

\subsubsection{Consequences for the Empirical Process} The empirical process $\hat{X}_n$, although finite, will tend towards $\rv X^s_t$, hence its moments to $+\infty$, as $h\rightarrow 0$. The use of assumptions of the form (\ref{eq:hyp_score_empirical}) is problematic as it involves a process that can diverge as the step size decreases.

Finally, note that some works using the TV distance or the KL divergence, \citep[see, e.g.,][]{chenSamplingEasyLearning2023,confortiKLConvergenceGuarantees2025}, do not make the assumption that $s_\theta$ is Lipchitz-continuous with respect to $x$, and still use the good form  (\ref{eq:hyp_score_true_process}) of assumption on the score approximation, yet give valid convergence bounds for the empirical process.
We believe that it is linked to the fact that the TV distance and KL divergence only depend on the ratio of density rather than on the actual values taken by the processes, in particular $(\hat{X}_n)_n$ can take very large values with small probabilities without changing these divergences much. We believe that the fact that we can get valid convergence bounds in KL and TV while the process diverges as $h\rightarrow0$ with non-zero probability, further demonstrates the limitations of this kind of metrics.

\section{Convergence Guarantees for Diffusion Models}
\label{sct:conv_approx_score}

We now give our convergence bounds for diffusion models. We choose to control the score error with an assumption of the form (\ref{eq:hyp_score_true_process}), in accordance with denoising score matching, hence we denote
$$
\ep_{\text{score}}(t) = \Vert \nabla \log p_{t}(X_{t}) - s_\theta(t,X_{t})\Vert_{L_2}.
$$
We present all bounds in Sections~\ref{sct:convergence_ODE_emp_score} to \ref{sct:convergence_SDE_true_score} and discuss them in Sections~\ref{sct:comments} and \ref{sct:comp_previous}.
\subsection{Euler Sampler for the Probability Flow ODE}
\label{sct:convergence_ODE_emp_score}
\begin{proposition}
\label{prop:convergence_ODE_emp_score}
Suppose that Assumptions 1 and 2 hold and that $\E[X] = 0$, then for $h \leq \epsilon\leq R^2$ and $T$ large enough (depending only on $\L(X)$\textnormal{)}, denoting $\hat{X} = \hat{X}_N$ the output of Algorithm~\ref{alg:ODE}, we have
\begin{align}
\label{eq:bound_ODE_emp}
W_2(\L(X),&\L(\hat{X})) 
\leq 
\underbrace{\sqrt{d\epsilon}}_{\textnormal{Early stopping error}}
+
\underbrace{\frac{\sqrt{2\epsilon}}{T}\exp\op\frac{R^2}{2\epsilon}\cp R^2}_{\textnormal{Propagated initialization error}} 
\notag \\ 
&\quad+ 
\underbrace{\sqrt{d}\frac{R^3}{\sqrt{2}\epsilon^{2}}\exp\op\frac{R^2}{2\epsilon}\cp h}_{\textnormal{(Propagated) discretization error}}
+ 
\underbrace{\sqrt{\frac{\epsilon}{2}} \exp\op\frac{R^2}{2\epsilon}\cp  h\sum_{k=1}^{N} \frac{\ep_{\text{score}}(\epsilon + hk)}{\sqrt{\epsilon + hk}}}_{\textnormal{(Propagated) error on score}}.
\end{align}
\end{proposition}

\textbf{Bound without early stopping.} If we replace Assumption 1 by Assumption 1', then we can remove the error associated to early stopping and replace $\epsilon$ by $\tau$ in the bound. Indeed, we can view the diffusion process (\ref{eq:sde_forward}) started from $X$ between time $0$ and $T$ as the process started from $Z$ between time $\tau$ and $T + \tau$. Therefore, running Algorithm~\ref{alg:SDE} on $X$ with $\epsilon=0$ (no early stopping) and time horizon $T$ is equivalent to running Algorithm~\ref{alg:SDE} on $Z$ with $\epsilon=\tau$ and time horizon $T+\tau$. This gives the following bound under Assumption 1':
\begin{multline*}
W_2(\L(X),\L(\hat{X})) 
\leq 
\underbrace{\frac{\sqrt{2\tau}}{T+\tau} \exp\op\frac{R^2}{2\tau}\cp R^2}_{\textnormal{Propagated initialization error}} 
\notag 
+ 
\underbrace{\sqrt{d}\frac{R^3}{\sqrt{2}\tau^{2}}\exp\op\frac{R^2}{2\tau}\cp h}_{\textnormal{(Propagated) discretization error}}
\\+ 
\underbrace{\sqrt{\frac{\tau}{2}} \exp\op\frac{R^2}{2\tau}\cp  h\sum_{k=1}^{N} \frac{\ep_{\text{score}}(hk)}{\sqrt{\tau + hk}}}_{\textnormal{(Propagated) error on score}}.    
\end{multline*}

\textbf{Initialization error.} Note that even if we do not assume $\E[X]=0$, we can still use the first bound of Corollary~\ref{cor:init_error}, leading to the term $\sqrt{\frac{2\epsilon}{T}} \exp\op\frac{R^2}{2\epsilon}\cp R$ for the propagated initialization error. This bound is valid for any $T\geq 0$, but we lose $1/2$ order in the rate of convergence with respect to $T$.

\medskip
\textbf{Propagated score error.} If we assume a uniform bound over the $L_2$-error on the learned score:
$$
\forall t\in[\epsilon,T], \ep_{\text{score}}(t) \leq \ep_{\text{score}},
$$
then bounding the sum by an integral, we get
$$
\sqrt{\frac{\epsilon}{2}}\exp\op\frac{R^2}{2\epsilon}\cp h\sum_{k=1}^{N}\frac{\ep_{\text{score}}(\epsilon +hk) }{\sqrt{\epsilon +hk}} \leq \sqrt{2\epsilon}\exp\op\frac{R^2}{2\epsilon}\cp\ep_{\text{score}}\sqrt{T}.
$$
More generally, if $t\mapsto\ep_{\text{score}}(t)$ is continuous, we have the limit:
$$
\sqrt{\frac{\epsilon}{2}}\exp\op\frac{R^2}{2\epsilon}\cp h\sum_{k=1}^{N}\frac{\ep_{\text{score}}(\epsilon +hk) }{\sqrt{\epsilon +hk}} \xrightarrow[h\rightarrow0]{} \frac{1}{\sqrt{2}}\exp\op\frac{R^2}{2\epsilon}\cp \int_\epsilon^T\sqrt{\frac{\epsilon}{t}}\ep_{\text{score}}(t)dt.
$$

\subsection{Heun Sampler for the Probability Flow ODE}
\label{sct:convergence_Heun_emp_score}

For Heun sampler, we will also need the hypothesis that for all $t\in [\epsilon,T], x \mapsto s_\theta(t,x)$ is $L$-Lipchitz. Note that a consequence of Lemma~\ref{lemma:bound_spatial_regularity} is that $x \mapsto \nabla\log p_t(x)$ is $C_t$-Lipchitz with 
$C_t = \max \op \frac{1}{t}, \left|\frac{R^2}{t^2}- \frac{1}{t}\right|\cp.$
In particular, assuming $\epsilon\leq  R^2$, for $t\in [\epsilon,T], C_t \leq \frac{R^2}{\epsilon^2}$, so it is reasonable to assume that $L \approx \frac{R^2}{\epsilon^2}$. 
\begin{proposition}
\label{prop:convergence_Heun_emp_score}
Suppose that Assumptions 1 and 2 hold, and that $\E[X] = 0$, then, for $\epsilon\leq R^2$, $h \leq \epsilon/2$ and $T$ large enough (depending only on $\L(X)$\textnormal{)}, denoting $\hat{X} = \hat{X}_N$ the output of Algorithm~\ref{alg:Heun}, we have
\begin{multline}
\label{eq:bound_Heun_emp}
W_2(\L(X),\L(\hat{X}))
\leq
\underbrace{\sqrt{d\epsilon}}_{\textnormal{Early stopping error}}
+
\underbrace{\frac{\sqrt{2\epsilon}}{T}\exp\op\frac{R^2}{\epsilon} +\frac{hR^2}{4\epsilon^2}\cp  R^2}_{\textnormal{Propagated initialization error}} 
\\
+ 
\underbrace{\frac{17dR^5}{\sqrt{2}\epsilon^{4}}\exp\op\frac{R^2}{\epsilon} +\frac{hR^2}{4\epsilon^2}\cp h^2}_{\textnormal{(Propagated) discretization error}}
\\
+ \underbrace{\frac{\sqrt{\epsilon}}{2\sqrt{2}}\exp
\op\frac{R^2}{\epsilon} +\frac{hR^2}{4\epsilon^2}\cp \op h\sum_{k=0}^{N-1}\frac{\ep_{\text{score}}(\epsilon +hk) }{\sqrt{\epsilon +h(k+1)}} +  \op 1 + \frac{hR^2}{2\epsilon^2}\cp h\sum_{k=1}^{N}\frac{\ep_{\text{score}}(\epsilon +hk) }{\sqrt{\epsilon +hk}}\cp}_{\textnormal{(Propagated) error on score}}.
\end{multline}
\end{proposition}

\textbf{Bound without early stopping.} Replacing Assumption 1 by Assumption 1', we get the following bound without early stopping error:
\begin{multline*}
W_2(\L(X),\L(\hat{X}))
\leq
\underbrace{\frac{\sqrt{2\tau}}{T+\tau}\exp\op\frac{R^2}{\tau} + \frac{hR^2}{4\tau^2}\cp R^2}_{\textnormal{Propagated initialization error}}
+
\underbrace{\frac{17dR^5}{\sqrt{2}\tau^{4}}\exp\op\frac{R^2}{\tau} +\frac{hR^2}{4\tau^2}\cp h^2}_{\textnormal{(Propagated) discretization error}}
\\
+ \underbrace{\frac{\sqrt{\tau}}{2\sqrt{2}}\exp\op\frac{R^2}{\tau} + \frac{hR^2}{4\tau^2}\cp \op h\sum_{k=0}^{N-1}\frac{\ep_{\text{score}}(hk) }{\sqrt{\tau +h(k+1)}}
+ \op 1 + \frac{hR^2}{2\tau^2}\cp h\sum_{k=1}^{N}\frac{\ep_{\text{score}}(hk) }{\sqrt{\tau +hk}}\cp}_{\textnormal{(Propagated) error on score}}.
\end{multline*}

\textbf{Initialization error.} As before, even if we do not assume $\E[X]=0$, we can still use the first bound of Corollary~\ref{cor:init_error}, leading to the term $\sqrt{\frac{2\epsilon}{T}} \exp\op\frac{R^2}{\epsilon} +\frac{hR^2}{4\epsilon^2}\cp  R$ for the propagated initialization error.

\medskip
\textbf{Propagated score error.} Similarly to Section~\ref{sct:convergence_ODE_emp_score}, if we assume a uniform bound over the $L_2$-error on the learned score, we get
\begin{multline*}
\frac{\sqrt{\epsilon}}{2\sqrt{2}}\exp
\op\frac{R^2}{\epsilon} +\frac{hR^2}{4\epsilon^2}\cp \op h\sum_{k=0}^{N-1}\frac{\ep_{\text{score}}(\epsilon +hk) }{\sqrt{\epsilon +h(k+1)}} +  \op 1 + \frac{hR^2}{2\epsilon^2}\cp h\sum_{k=1}^{N}\frac{\ep_{\text{score}}(\epsilon +hk) }{\sqrt{\epsilon +hk}}\cp\\
\leq 
\sqrt{\frac{\epsilon}{2}}\exp\op\frac{R^2}{\epsilon} +\frac{hR^2}{4\epsilon^2}\cp  \op 2 + \frac{hR^2}{2\epsilon^2}\cp\ep_{\text{score}}\sqrt{T}.
\end{multline*}
More generally, if $t\mapsto\ep_{\text{score}}(t)$ is continuous, we have the same limit as for Euler discretization:
\begin{multline*}
\frac{\sqrt{\epsilon}}{2\sqrt{2}}\exp
\op\frac{R^2}{\epsilon} +\frac{hR^2}{4\epsilon^2}\cp \op h\sum_{k=0}^{N-1}\frac{\ep_{\text{score}}(\epsilon +hk) }{\sqrt{\epsilon +h(k+1)}} +  \op 1 + \frac{hR^2}{2\epsilon^2}\cp h\sum_{k=1}^{N}\frac{\ep_{\text{score}}(\epsilon +hk) }{\sqrt{\epsilon +hk}}\cp\\
\xrightarrow[h\rightarrow0]{} 
\frac{1}{\sqrt{2}}\exp\op\frac{R^2}{\epsilon}\cp
\int_\epsilon^T\sqrt{\frac{\epsilon}{t}}\ep_{\text{score}}(t)dt.
\end{multline*}

\subsection{SDE Sampler}
\label{sct:convergence_SDE_emp_score}
\begin{proposition}
\label{prop:convergence_SDE_emp_score}
Suppose that Assumptions 1 and 2 hold and that $\E[X] = 0$, then, for $h\leq \epsilon\leq R^2$ and $T$ large enough (depending only on $\L(X)$\textnormal{)}, denoting $\hat{X} = \hat{X}_N$ the output of Algorithm \ref{alg:SDE}, we have
\begin{align}
\label{eq:bound_SDE_emp}
W_2&(\L(X),\L(\hat{X)})
\leq
\underbrace{\sqrt{d\epsilon}}_{\textnormal{Early stopping error}}
+
\underbrace{\frac{2\epsilon}{T^{3/2}} \exp\op\frac{R^2}{\epsilon}\cp R^2}_{\textnormal{Propagated initialization error}} 
\notag \\
&\quad+ 
\underbrace{\frac{\sqrt{2d}}{2}\frac{R^2}{\epsilon}\exp\op\frac{R^2}{\epsilon}\cp \sqrt{h}}_{\textnormal{(Propagated) discretization error}}
+ 
\underbrace{2\epsilon \exp\op\frac{R^2}{\epsilon}\cp  h\sum_{k=1}^{N} \frac{\ep_{\text{score}}(\epsilon + hk)}{{\epsilon + hk}}}_{\textnormal{(Propagated) error on score}}.
\end{align}
\end{proposition}

\textbf{Bound without early stopping.} Replacing Assumption 1 by Assumption 1', we get the following bound without early stopping error:
\begin{multline*}
W_2(\L(X),\L(\hat{X)})
\leq
\underbrace{\frac{2\tau}{(T+\tau)^{3/2}} \exp\op\frac{R^2}{\tau}\cp R^2}_{\textnormal{Propagated initialization error}} 
\notag \\
+
\underbrace{\sqrt{2d}\frac{R^2}{\tau}\exp\op\frac{R^2}{\tau}\cp\sqrt{h}}_{\textnormal{(Propagated) discretization error}}
+ 
\underbrace{2\tau \exp\op\frac{R^2}{\tau}\cp  h\sum_{k=1}^{N} \frac{\ep_{\text{score}}(hk)}{{\tau + hk}}}_{\textnormal{(Propagated) error on score}}.
\end{multline*}

\medskip
\textbf{Initialization error.} As before, even if we do not assume $\E[X]=0$, we can still use the first bound of Corollary~\ref{cor:init_error}, leading to the term $\frac{2\epsilon}{T} \exp\op\frac{R^2}{\epsilon}\cp R$ for the propagated initialization error.

\textbf{Propagated score error.} Similarly to Section~\ref{sct:convergence_ODE_emp_score}, if we assume a uniform bound over the $L_2$-error on the learned score, bounding the sum by an integral, we get
$$
{2\epsilon}\exp\op\frac{R^2}{\epsilon}\cp h\sum_{k=1}^{N}\frac{\ep_{\text{score}}(\epsilon +hk) }{{\epsilon +hk}}
\leq  {2\epsilon}\exp\op\frac{R^2}{\epsilon}\cp  \log\op\frac{T}{\epsilon}\cp\ep_{\text{score}}.
$$
More generally, if $t\mapsto\ep_{\text{score}}(t)$ is continuous, we have the limit:
$$
{2\epsilon}\exp\op\frac{R^2}{\epsilon}\cp h\sum_{k=0}^{N-1}\frac{\ep_{\text{score}}(\epsilon +h(k+1)) }{{\epsilon +hk}}
\xrightarrow[h\rightarrow0]{} 
2\exp\op\frac{R^2}{\epsilon}\cp\int_\epsilon^T\frac{\epsilon}{{t}}\ep_{\text{score}}(t) dt.
$$

\subsection{Convergence of the Euler-Maruyama Sampler Under True Score Assumption}

We finally prove convergence of order $1$ in the step size for the Euler-Maruyama sampler with accurate score, matching the optimal rate of convergence for the Euler-Maruyama discretization of SDEs with additive noise, i.e., constant diffusion coefficients \citep[see, e.g.,][Chapter 10.2]{kloedenNumericalSolutionStochastic1992}. The pivotal aspect of the proof that enables this improved convergence rate is that the discretization error at each step is independent from the errors at the previous steps.

\label{sct:convergence_SDE_true_score}
\begin{proposition}
\label{prop:convergence_SDE_true_score}
Suppose that Assumption 1 holds, that $\E[X] = 0$ and that for all $t \in [\epsilon,T],x\in \R^d, s_\theta(t,x) = \nabla\log p_t(x)$, then, for $h\leq \epsilon\leq R^2$ and $T$ large enough (depending only on $\L(X)$\textnormal{)}, denoting $\hat{X} = \hat{X}_N$ the output of Algorithm \ref{alg:SDE}, we have
\begin{equation}
\label{eq:bound_SDE}
W_2(\L(X),\L(\hat{X}))
\leq
\underbrace{\sqrt{d\epsilon}}_{\textnormal{Early stopping error}}
+
\underbrace{\frac{2\epsilon}{T^{3/2}} \exp\op\frac{R^2}{\epsilon}\cp R^2}_{\textnormal{Propagated initialization error}} 
+
\underbrace{\sqrt{\frac{2d}{3}}\frac{R^2}{\epsilon^{3/2}}\exp\op\frac{R^2}{\epsilon}\cp h.}_{\textnormal{(Propagated) discretization error}}
\end{equation}
\end{proposition}

\textbf{Bound without early stopping.} Replacing Assumption 1 by Assumption 1', we get the following bound without early stopping error:
$$
W_2(\L(X),\L(\hat{X}))
\leq
\underbrace{\frac{2\tau}{(T+\tau)^{3/2}} \exp\op\frac{R^2}{\tau}\cp R^2}_{\textnormal{Propagated initialization error}} 
+
\underbrace{\sqrt{\frac{2d}{3}}\frac{R^2}{\tau^{3/2}}\exp\op\frac{R^2}{\tau}\cp h.}_{\textnormal{(Propagated) discretization error}}
$$

\textbf{Initialization error.} As before, even if we do not assume $\E[X]=0$, we can still use the first bound of Corollary~\ref{cor:init_error}, leading to the term $\frac{2\epsilon}{T} \exp\op\frac{R^2}{\epsilon}\cp R$ for the propagated initialization error.

\medskip
\textbf{Relationship with the work of \cite{wangWassersteinBoundsGenerative2025}.}
\citet{wangWassersteinBoundsGenerative2025} give Wasserstein convergence guarantees for the Euler-Maruyama discretization of the SDE (\ref{eq:reverse_SDE}) with a discretization error in $h$ and this with an inaccurate score network.
Although we pointed out in a first version of this work an error in their original proof that invalidated their bounds, the authors have since released an updated version containing a corrected argument.
Their revised proof relies on the same key observation as Proposition~\ref{prop:convergence_SDE_true_score} - that the discretization error at each step is independent of the errors from previous steps. To handle the case of an inaccurate score network, they introduce an additional technique that could be adapted to our framework, although it yields a looser estimate of the propagated score error than Proposition~\ref{prop:convergence_SDE_emp_score}.

\subsection{Discussion}
\label{sct:comments}

Under Assumption 1', we give simplified bounds by focusing only on the key parameters which are the time $T$ from which we reverse the forward process, the step size $h$ and the score error $\ep_{\text{score}}(t)$. We hide other constants depending on the regularity of the data distribution, and assume that $h$ is small enough such that we can replace the sum by an integral in the propagated score error. We get the following bounds:
\begin{align*}
\textbf{Euler (ODE):} &&W_2(\L(X),\L(\hat{X})) &\lesssim \frac{1}{T+\tau} + h + \int_0^T\sqrt{\frac{\tau}{t+\tau}}\ep_{\text{score}}(t)dt\\
\textbf{Heun (ODE):} &&W_2(\L(X),\L(\hat{X})) &\lesssim \frac{1}{T+\tau} + h^2 + \int_0^T\sqrt{\frac{\tau}{t+\tau}}\ep_{\text{score}}(t)dt\\
\textbf{Euler-Maruyama (SDE):} &&W_2(\L(X),\L(\hat{X})) &\lesssim \frac{1}{(T+\tau)^{3/2}} + \sqrt{h} + \int_0^T{\frac{\tau}{t+\tau}}\ep_{\text{score}}(t)dt
\end{align*}

For all our bounds, we observe desirable convergence properties. The initialization error goes to $0$ as $T\rightarrow\infty$, the discretization error goes to $0$ as  the step size $h\rightarrow0$, and the score approximation error goes to zero as $\ep_\textnormal{score} \rightarrow0$.

For the discretization error, we have a bound in $O(\sqrt{h})$ for the Euler-Maruyama sampler, in $O(h)$ for the Euler discretization of the probability ODE and in $O(h^2)$ for Heun discretization. These different convergence rates are consistent with empirical observations \citep[see, e.g.,][]{karrasElucidatingDesignSpace2022,yangDiffusionModelsComprehensive2023}, the Euler-Maruyama sampler usually requiring more calls to the score function $s_\theta$ (i.e., neural
function evaluations, NFE) than the deterministic samplers, among which Heun sampler is the fastest. Interestingly, when the score is known precisely (Proposition~\ref{prop:convergence_SDE_true_score}), we find  that the rate of the Euler-Maruyama is similar to the Euler deterministic sampler, in $O(h)$.

Finally, we observe that the error made during the sampling process, the initialization error and the score approximation error, are contracted by a factor of $\sqrt{\frac{\tau}{t+\tau}}\leq 1$ for the deterministic samplers (Euler and Heun) and by a stronger factor of $\frac{\tau}{t+\tau} \leq 1$ for the Euler-Maruyama sampler.
As $h$ goes to zero (i.e., for a large number of steps), the discretization error becomes negligible compared to initialization and score errors. This observation may help explain empirical findings \citep[see, e.g.,][]{xuRestartSamplingImproving2023}: deterministic samplers (Euler and Heun for the ODE) tend to perform better with a small number of function evaluations (NFE), whereas stochastic samplers (Euler–Maruyama for the SDE) are more effective at high NFE, as they provide stronger contraction of the initialization error and score errors.
 
\subsection{Comparison with Previous Works}
\label{sct:comp_previous}
We compare our results with previous or concurrent works on Wasserstein convergence guarantees for diffusion models. We focus on the Euler-Maruyama sampler for the reverse ODE, as there are, to our knowledge, no previous bounds in Wasserstein distance for Heun sampler and only bounds in the limited log-concave setting for the Euler sampler of the probability flow ODE. We compare to \citet{gentiloni-silveriLogConcavityScoreRegularity2025,brunoWassersteinConvergenceScorebased2025a,wangWassersteinBoundsGenerative2025} that give Wasserstein bounds using assumption on the data distribution that generalize beyond the log-concave case.

\subsubsection{Assumptions on the Data Distribution} We compare our Assumption 1' on the data distribution with the assumptions of \citet{gentiloni-silveriLogConcavityScoreRegularity2025,brunoWassersteinConvergenceScorebased2025a,wangWassersteinBoundsGenerative2025}.
In all cases, the idea is to generalize convergence result beyond the log-concave case, while keeping some form of regularity needed to get convergence results.
While Assumption 1' can in fact be seen as a particular case of assumptions made in previous work, we believe that it make proofs much simpler while remaining general enough to model real world data and understanding the mechanisms at stake.

Assumption 1 is very general as it does not assume anything about the data distribution apart from bounded support. It is therefore needed to add regularity. We believe that a simple and natural way to model this extra regularity is to add Gaussian noise to the distribution, either through early stopping (with an extra error term) or equivalently under Assumption 1'. 

This extra regularity is comparable to the assumptions of \citet{gentiloni-silveriLogConcavityScoreRegularity2025,brunoWassersteinConvergenceScorebased2025a,wangWassersteinBoundsGenerative2025}. 
Indeed, Assumption 1' can be seen as a special case of the Gaussian tail (Assumption 2) of \citet{wangWassersteinBoundsGenerative2025}, as the authors prove in their Theorem 3.7.
It can also be seen as a special case of Assumption 2 of \citet{brunoWassersteinConvergenceScorebased2025a}. Indeed, for $X = Z +\N(0,\tau^2 I)$, we have with Lemma~\ref{lemma:score_moment}:
$$
h(x) = -\nabla\log p_X(x) = \frac{x}{\tau^2} - \frac{\E[Z|X=x]}{\tau^2},
$$
leading to
\begin{align*}
\langle h(x)-h(\bar x), x -\bar x\rangle 
&= \frac{\Vert x - \bar x\Vert^2}{\tau^2} - \frac{\langle\E[Z|X=x] - \E[Z|X=\bar x],x-\bar x\rangle}{\tau^2}\\
&\geq \frac{\Vert x - \bar x\Vert^2}{\tau^2} - \frac{2R\Vert x - \bar x\Vert}{\tau^2}
\geq \frac{\Vert x - \bar x\Vert^2}{2\tau^2},
\end{align*}
for $\Vert x - \bar x\Vert \geq 4R$, hence we get Assumption 2(iii) with $\mu =\frac{1}{2\tau^2}$. And with Lemma~\ref{lemma:bound_spatial_regularity}, we have
\begin{align*}
\langle h(x)-h(\bar x), x -\bar x\rangle 
&= -\int_0^1 (x-\bar x)^T \nabla ^2 \log p_{\tau^2}(x + t(x-\bar x))(x-\bar x)dt\\
&\geq - \op\frac{R^2}{\tau^4}-\frac{1}{\tau^2}\cp \Vert x - \bar x\Vert^2,
\end{align*}
hence we get Assumption 2(ii) with $K =\frac{R^2}{\tau^4}-\frac{1}{\tau^2}$.

Finally, \citet{brunoWassersteinConvergenceScorebased2025a} prove the equivalence between their Assumption 2(ii)-(iii) and the weak-convexity assumption (\textbf{H1}(ii)) of \citet{gentiloni-silveriLogConcavityScoreRegularity2025}, while we get the one-sided Lipchitz assumption \textbf{H1}(i) using Lemma~\ref{lemma:bound_spatial_regularity},
\begin{align*}
\langle h(x)-h(\bar x), x -\bar x\rangle 
&= -\int_0^1 (x-\bar x)^T \nabla ^2 \log p_{\tau^2}(x + t(x-\bar x))(x-\bar x)dt\\
&\leq \frac{1}{\tau^2} \Vert x - \bar x\Vert^2.
\end{align*}

\subsubsection{Bounds}
These three works use the Ornstein–Uhlenbeck (OU) process instead of (\ref{eq:sde_forward}):
$$
\left\{
\begin{array}{l}
d X_t^\OU =  -X_t^\OU + \sqrt{2}dB_t, \\
X_0^\OU = X.
\end{array}\right.
$$
In Appendix~\ref{sct:param_time}, we discuss how different parametrizations of the forward process lead to different scalings and noise schedules and how it impacts the different quantities of interest. In particular, the OU process has marginals $X_t^\OU \sim s^\OU(t)(X + \N(0,\sigma^\OU(t)^2I))$ with
\begin{align*}
s^\OU(t) &= e^{-t}\\
\sigma^\OU(t) &= \sqrt{e^{2t}-1}.
\end{align*}

The work of \citet{gentiloni-silveriLogConcavityScoreRegularity2025} gives a bound of the form (focusing on the parameters $T$, $h$ and the score error):
$$
W_2(\L(X),\L(\hat{X})) \lesssim e^{-T}+ T\sqrt{h} + \ep T,
$$
where the score error is controlled with the uniform bound\footnote{For comparison purposes, we put aside the fact that \citet{gentiloni-silveriLogConcavityScoreRegularity2025} and \citet{brunoWassersteinConvergenceScorebased2025a} use a control of the score with respect to the empirical process.}
\begin{equation}
\label{eq:control_score_uni}
\ep_\textnormal{score}^\OU(t) = \Vert\nabla\log p_t^\OU(X_t^\OU)-s_\theta^\OU(t,X_t) \Vert_{L_2} \leq \ep.
\end{equation}
Writing the bound of \citet{brunoWassersteinConvergenceScorebased2025a} in a simplified form,\footnote{
We have removed the term corresponding to the early stopping error, replaced the term $e^{-2 \int_\epsilon^T \beta_t^{OS,K,\mu}dt-\epsilon}$ by $e^{-2T}$, using that $\beta_t^{OS,K,\mu} \leq 1$ (the simplified bound given here is therefore over-optimistic compared to the true complete bound given by the authors).
We do not give the full expressions of the constants $C_1(T)$ and $C_2(T)$ that are complicated and involve hypotheses on the generation process that are not comparable to what is done in the other works.
} we have
$$
W_2(\L(X),\L(\hat{X})) \lesssim e^{-2T}+ C_1(T)\sqrt{h} + C_2(T)\sqrt{\int_0^T \ep_\textnormal{score}^\OU(t)^2 dt},
$$
where $C_1(T)$ and $C_2(T)$ are constants depending on time $T$ and other parameters controlling the regularity of the process.
And finally the bound given by \citet{wangWassersteinBoundsGenerative2025} can be rewritten as
$$
W_2(\L(X),\L(\hat{X})) \lesssim e^{-T}+ h + \sqrt{\int_0^T \ep_\textnormal{score}^\OU(t)^2 dt}.
$$

To be able to compare these bounds with our own, we choose to express each bound in term of noise levels, with $\sigma_{\max} = \sigma(T)$, $N$ the number of sampling steps and $\td\ep_{\text{score}}(\sigma^2)$ the normalized $L_2$-error on the score at noise level $\sigma^2$, defined by 
$$
\td\ep_{\text{score}}(\sigma^2) = \left\Vert\nabla\log p\op X + \sigma(t)Z;\sigma(t)^2\cp - \td f_\theta( X + \sigma(t)Z,\sigma(t)^2)\right\Vert_{L_2},
$$
where $Z \sim \N(0,I)$ independent from $X$ and $p\op x; \sigma^2\cp$ is the density of the normalized variable at noise level $\sigma^2$: $X + \sigma Z$.
As described in Appendix~\ref{sct:param_time}, different time parametrization leads to different scaling of the score error and we have
$$
\ep_\textnormal{score}^\OU(t) = \frac{1}{s^\OU(t)} \td\ep_\textnormal{score}(\sigma^\OU(t)^2) = \sqrt{\sigma^\OU(t)^2 +1} \td\ep_\textnormal{score}(\sigma^\OU(t)^2).
$$
The bound of \citet{gentiloni-silveriLogConcavityScoreRegularity2025} becomes
$$
W_2(\L(X),\L(\hat{X})) \lesssim \frac{1}{\sigma_{\max}}+ \frac{\log(\sigma_{\max})^{3/2}}{N^{1/2}} + \ep \log(\sigma_{\max})
$$
where the normalized score error is controlled by
\begin{equation}
\label{eq:control_normalized_score}
\td \ep_\textnormal{score}(\sigma^2) \leq \frac{\ep}{\sqrt{\sigma^2+1}}.
\end{equation}
The bound of \citet{brunoWassersteinConvergenceScorebased2025a} becomes
$$
W_2(\L(X),\L(\hat{X})) \lesssim \frac{1}{\sigma_{\max}^2}
+ \td C_1(\sigma_{\max})\sqrt{\frac{\log(\sigma_{\max})}{N}}
+ \td C_2(\sigma_{\max})\sqrt{\int_0^{\sigma_{\max}^2} \td\ep_\textnormal{score}(\sigma^2)^2 d\sigma^2}.
$$
and finally the bound of \citet{wangWassersteinBoundsGenerative2025} becomes
$$
W_2(\L(X),\L(\hat{X})) 
\lesssim
\frac{1}{\sigma_{\max}^2}
+ \frac{\log(\sigma_{\max})}{N}
+ \sqrt{\int_0^{\sigma_{\max}^2} \td\ep_\textnormal{score}(\sigma^2)^2 d\sigma^2}.
$$
For our own bound, as with our time parametrization we simply have $s(t) = 1$ and $\sigma(t)^2 = t$, using the simplified expression of section~\ref{sct:comments}, we have 
$$
W_2(\L(X),\L(\hat{X})) \lesssim \frac{1}{\sigma_{\max}^{3}} + \frac{\sigma_{\max}}{\sqrt{N}} + \int_0^{\sigma_{\max}^2}{\frac{1}{\sigma^2+1}}\td\ep_{\text{score}}(\sigma^2)d\sigma^2,
$$
as $\frac{\tau}{\sigma^2+\tau} \leq \frac{\max(\tau,1)}{\sigma^2+1} \lesssim \frac{1}{\sigma^2 +1}$.

We observe that our control of the score approximation error is stronger that the one of \citet{brunoWassersteinConvergenceScorebased2025a} and \citet{wangWassersteinBoundsGenerative2025}, as it is smaller. Indeed, using Cauchy–Schwarz inequality, we have
\begin{align*}
\int_0^{\sigma_{\max}^2}{\frac{1}{\sigma^2+1}}\td\ep_{\text{score}}(\sigma^2)d\sigma^2
&\leq
\sqrt{\int_0^{\sigma_{\max}^2} \frac{d\sigma^2}{(\sigma^2+1)^{2}}}
\sqrt{\int_0^{\sigma_{\max}^2}\td\ep_{\text{score}}(\sigma^2)^2d\sigma^2 }\\
&\leq
\sqrt{\int_0^{+\infty} \frac{d\sigma^2}{(\sigma^2+1)^2}}
\sqrt{\int_0^{\sigma_{\max}^2}\td\ep_{\text{score}}(\sigma^2)^2d\sigma^2 }\\
&= \sqrt{\int_0^{\sigma_{\max}^2}\td\ep_{\text{score}}(\sigma^2)^2d\sigma^2 }.
\end{align*}
Moreover, under Assumption (\ref{eq:control_score_uni}) (or equivalently  (\ref{eq:control_normalized_score})) of \citet{gentiloni-silveriLogConcavityScoreRegularity2025}, we get
$$
\sqrt{\int_0^{\sigma_{\max}^2} \ep_\textnormal{score}(\sigma^2)^2 dt} \lesssim \sqrt{\log\sigma_{\max}} \ep,
$$
and
$$
 \int_0^{\sigma_{\max}^2}{\frac{1}{\sigma^2+1}}\td\ep_{\text{score}}(\sigma^2)d\sigma^2 \leq \ep \int_0^{+\infty} \frac{d\sigma^2}{(\sigma^2+1)^{3/2}}\lesssim \ep.
$$
Our control of the score approximation error is of order $\ep$ compare to $\log(\sigma_{\max}) \ep$ for \citet{gentiloni-silveriLogConcavityScoreRegularity2025} and  $\sqrt{\log(\sigma_{\max})} \ep$ for \citet{wangWassersteinBoundsGenerative2025}.

We also observe that the order of the initialisation error is $\sigma_{\max}^{-1}$ for \citet{gentiloni-silveriLogConcavityScoreRegularity2025} and \citet{wangWassersteinBoundsGenerative2025}, of order $\sigma_{\max}^{-2}$ for \cite{brunoWassersteinConvergenceScorebased2025a}, while we get $\sigma_{\max}^{-3}$ with our careful control of the propagation of error and by using the second bound of Proposition~\ref{prop:init_error}.

Finally, we observe that for the discretization error, we get a rate in $N^{-1/2}$ similar to \citet{gentiloni-silveriLogConcavityScoreRegularity2025} and \citet{brunoWassersteinConvergenceScorebased2025a}, while \citet{wangWassersteinBoundsGenerative2025} have a rate in $N^{-1}$, as already highlighted in section~\ref{sct:convergence_SDE_true_score}.
Moreover we get a rate is $\sigma_{\max}$ while other works have a rate in $\sqrt{\log\sigma_{\max}}$, $\log\sigma_{\max}$ or $(\log\sigma_{\max})^{3/2}$. 
We believe that this less favorable rate arises from the specific parameterization of the process in (\ref{eq:sde_forward}) compared to the Ornstein–Uhlenbeck process.

Appendix~\ref{sct:param_time:changes} discusses how our analysis can be extended to alternative time parameterizations. In particular, we show that modifying the time parameterization does not affect the order of magnitude of the initialization or score approximation errors. Finally, our framework also permits control the discretization error in the general case; however, a detailed derivation of these results remains for future work.

\section{Conclusion}

Within the simple framework of Assumption 1 (or 1'), we establish new Wasserstein convergence guarantees for diffusion-based generative models, covering both stochastic (DDPM-like) and deterministic (DDIM-like) samplers. Notably, we derive the first Wasserstein convergence bound for the Heun sampler and improve existing results for the Euler sampler of the probability flow ODE.
Our analysis emphasizes the importance of the learned score’s spatial regularity and advocates controlling the score error relative to the true reverse process, consistent with denoising score matching. We also leverage recent results on smoothed Wasserstein distances to refine the initialization error bounds.

There are several avenues to improve this work.
Under Assumption 1, our bounds exhibit exponential dependency on $R$ and $1/\epsilon$.
This is because we do not make any regularity assumption besides bounded support; thus, the distribution could be very rough.
With Assumption 1', we are able to add regularity by adding Gaussian noise, hence getting ride of early stopping and the dependency in $1/\epsilon$. 
We have shown that this framework is a particular case of those studied by 
\citet{gentiloni-silveriLogConcavityScoreRegularity2025},  \citet{brunoWassersteinConvergenceScorebased2025a} and \citet{wangWassersteinBoundsGenerative2025}. More generally, it involves following the spatial and time regularity of the score along the noising process \citep[see, e.g.][]{confortiWeakSemiconvexityEstimates2024,brigatiHeatFlowLogconcavity2024}.
Extending our proofs to these broader frameworks is an interesting direction, although likely at the expense of the current simplicity.

We could also try to improve our bounds by looking at the dimensional structure of the data distribution, for example under the \emph{manifold hypothesis}, as \citet{tangAdaptivityDiffusionModels2024} or \citet{azangulovConvergenceDiffusionModels2025}, to hope to replace the ambient dimension by the intrinsic dimension of the data manifold.

Following the discussion initiated in Appendix ~\ref{sct:param_time}, an interesting direction for future work would be to extend this analysis to alternative time parameterizations (i.e. noise schedules). It would also be insightful to explore how similar bounds could extend to flow matching or stochastic interpolants \citep{liuFlowStraightFast2023,lipmanFlowMatchingGenerative2023,albergoStochasticInterpolantsUnifying2023}.
When the source distribution is a Gaussian, flow matching corresponds to a particular time parametrization of diffusion models, hence the tools developed here could be directly used. Using a more general source distribution distribution will required adapting how we control of the regularity of the velocity fields, under proper assumption on this distribution.

Finally, we believe that our proof for Euler and Heun discretization schemes could be adapted to more general $p$-th order Runga-Kutta schemes on the probability flow ODEs, as it has already been done in TV distance \citep{huangConvergenceAnalysisProbability2025,huangFastConvergenceHighOrder2025}. Indeed, one will simply need to control higher order time derivatives of $\nabla \log \rv p_t (x_t)$, thus involving higher-order conditional moments of $X$ which can be bounded under Assumption 1.

\acks{%
This work has received support from the French government, managed by the National Research Agency, under the France 2030 program with the reference ``PR[AI]RIE-PSAI'' (ANR-23-IACL-0008).
}

\newpage
\appendix
\appendixpage

\section{Proofs}
\subsection{Technical Lemmas}
\subsubsection{Expressing the Score with Conditional Moments of the Probability Distribution}
\label{sct:score_moment}
It is known that we can express derivatives of the log-density $\log p_t$ with conditional moments (see, e.g., for exemple appendix B of \citet{saremiChainLogConcaveMarkov2023} for the gradient and the Hessian). Here, we give expressions for $\nabla\log p_t,$ $\nabla^2\log p_t$ and $\nabla\Delta  \log p_t$.

\begin{lemma}
\label{lemma:score_moment}
For $t>0$,
\begin{align*}
\nabla\log p_t(x) &= \frac{1}{t}(\E[X|X_t = x] - x) = -\frac{1}{t}E[B_t|X_t = x],\\
\nabla^2\log p_t(x) &= - \frac{1}{t}I + \frac{1}{t^2}\cov(X|X_t=x),\\
\nabla\Delta  \log p_t(x) &= \frac{1}{t^3} \E[\Vert X\Vert^2(X- \E[X|X_t=x])|X_t=x] - \frac{2}{t^3} \cov(X|X_t=x) \cdot\E[X|X_t = x].
\end{align*}
\end{lemma}

\begin{proof}
For $x\in \R^d$, we have 
$$
p_t(x) =  \frac{1}{(2\pi t)^{d/2}}\int e^{-\frac{\Vert x_0-x\Vert^2}{2t}}dP_X(x_0).
$$
It leads to:
\begin{align*}
\nabla\log p_t(x) 
&= \frac{\int \frac{x_0-x}{t}e^{-\frac{\Vert x_0-x\Vert^2}{2t}}dP_X(x_0)}{\int e^{-\frac{\Vert x_0-x\Vert^2}{2t}}dP_X(x_0)}\\
&= \frac{1}{t}\op E[X -X_t|X_t = x] \cp = -\frac{1}{t}E[B_t|X_t = x]
= \frac{1}{t}\op E[X|X_t = x] -x\cp.
\end{align*}

\noindent Then we compute 
\begin{align*}
\nabla &E[X|X_t = x] \\
&=\nabla\op\frac{\int x_0e^{-\frac{\Vert x_0-x\Vert^2}{2t}}dP_X(x_0)}{\int e^{-\frac{\Vert x_0-x\Vert^2}{2t}}dP_X(x_0)}\cp\\
&=\frac{\int x_0 \op\frac{x_0-x}{t}\cp^\top e^{-\frac{\Vert x_0-x\Vert^2}{2t}}dP_X(x_0)}{\int e^{-\frac{\Vert x_0-x\Vert^2}{2t}}dP_X(x_0)} - \frac{\int x_0 e^{-\frac{\Vert x_0-x\Vert^2}{2t}}dP_X(x_0)\op\int\frac{x_0-x}{t}e^{-\frac{\Vert x_0-x\Vert^2}{2t}}dP_X(x_0)\cp^\top}{\op \int e^{-\frac{\Vert x_0-x\Vert^2}{2t}}dP_X(x_0)\cp^2} \\
&= \frac{1}{t}\op\E[X(X-X_t)^\top|X_t=x] -\E[X|X_t=x]\E[X-X_t|X_t=x]^\top\cp \\
&= \frac{1}{t}\op\E[XX^\top|X_t=x] -\E[X|X_t=x]\E[X|X_t=x]^\top\cp
= \frac{1}{t}\cov(X|X_t=x).
\end{align*}
This leads to
\begin{align*}
\nabla^2\log p_X(x) =\frac{1}{t}\op \nabla E[X|X_t = x] - I \cp =- \frac{1}{t}I + \frac{1}{t^2}\cov(X|X_t=x).
\end{align*}
From the expression of the Hessian, we get the Laplacian by taking the trace,
\begin{align*}
\Delta\log p_X(x) &= \tr\nabla^2\log p_X(x) \\
&=- \frac{1}{t}d + \frac{1}{t^2}\op\E[\tr(XX^\top)|X_t=x] -\tr(\E[X|X_t=x]\E[X|X_t=x]^\top)\cp \\
&=- \frac{1}{t}d + \frac{1}{t^2}\op\E[\Vert X\Vert^2|X_t=x] -\Vert\E[X|X_t=x]\Vert^2\cp,
\end{align*}
leading to 
$$
\nabla \Delta\log p_X(x) = \frac{1}{t^2}\op\nabla\E[\Vert X\Vert^2|X_t=x] -2 \nabla\E[X|X_t=x]\cdot\E[X|X_t=x]\cp.
$$
We already know that
$$
\nabla E[X|X_t = x]  = \frac{1}{t}\cov(X|X_t=x),
$$
and similarly we compute,
\begin{align*}
\nabla&\E[\Vert X\Vert^2|X_t=x] \\
&=\nabla\op\frac{\int \Vert x_0\Vert^2e^{-\frac{\Vert x_0-x\Vert^2}{2t}}dP_X(x_0)}{\int e^{-\frac{\Vert x_0-x\Vert^2}{2t}}dP_X(x_0)}\cp\\
&=\frac{\int \Vert x_0\Vert^2 \op\frac{x_0-x}{t}\cp e^{-\frac{\Vert x_0-x\Vert^2}{2t}}dP_X(x_0)}{\int e^{-\frac{\Vert x_0-x\Vert^2}{2t}}dP_X(x_0)} - \frac{\int \Vert x_0\Vert^2 e^{-\frac{\Vert x_0-x\Vert^2}{2t}}dP_X(x_0) \int\frac{x_0-x}{t}e^{-\frac{\Vert x_0-x\Vert^2}{2t}}dP_X(x_0)}{\op \int e^{-\frac{\Vert x_0-x\Vert^2}{2t}}dP_X(x_0)\cp^2} \\
&= \frac{1}{t}\op \E[\Vert X\Vert^2(X-X_t)|X_t=x] - \E[\Vert X\Vert^2|X_t=x] \E[X-X_t|X_t=x]\cp\\
&= \frac{1}{t}\op \E[\Vert X\Vert^2X|X_t=x] - \E[\Vert X\Vert^2|X_t=x] \E[X|X_t=x]\cp\\
&= \frac{1}{t}\E[\Vert X\Vert^2(X-\E[X|X_t=x])|X_t=x].
\end{align*}

\noindent It finally leads to 
\begin{align*}
\nabla\Delta  \log p_X(x) &= \frac{1}{t^3} \E[\Vert X\Vert^2(X- \E[X|X_t=x])|X_t=x] - \frac{2}{t^3} \cov(X|X_t=x) \cdot\E[X|X_t = x].
\end{align*}
\end{proof}

We also gives expressions for spatial derivatives of higher orders, but limit ourselves to the case of $d=1$ for simpler notations.
\begin{lemma}
\label{lemma:score_moment_higher_order}
In dimension $d=1$, for $t>0$,
\begin{align*}
\partial_x^4 \log   p_t(x)
&= \frac{1}{t^4}(\E[X^4|X_{t}=x] -4\E[X^3|X_{t}=x]\E[X|X_{t} = x] -3\E[X^2|X_{t}=x]^2
\\&\quad 
+ 12\E[X^2|X_{t}=x]\E[X|X_{t} = x]^2
-6\E[X|X_{t}=x]^4),\\
\partial_x^5 \log   p_t(x)
&= \frac{1}{t^5}(
\E[X^5|X_{t}=x]
-5 \E[X^4|X_{t}=x] \E[X|X_{t}=x]
-10 \E[X^3|X_{t}=x] \E[X^2|X_{t}=x]
\\&\quad
+20 \E[X^3|X_{t}=x] \E[X|X_{t}=x]^2
+30 \E[X^2|X_{t}=x]^2 \E[X|X_{t}=x]
\\&\quad
-60 \E[X^2|X_{t}=x] \E[X|X_{t}=x]^3
+24 \E[X|X_{t}=x]^5
).
\end{align*}
\end{lemma}
\begin{proof}
In dimension $d=1$, the last identity of Lemma~\ref{lemma:score_moment} becomes
\begin{align*}
\partial_x^3 \log   p_t(x) &= \frac{1}{t^3} \E[X^2(X-\E[X|X_{t}=x] )|X_{t}=x] - \frac{2}{t^3} \cov(X|X_{t}=x) \cdot\E[X|X_{t} = x]\\
&= \frac{1}{t^3} \E[X^3|X_{t}=x] - \frac{3}{t^3} \E[X^2|X_{t}=x]\E[X|X_{t} = x] + \frac{2}{t^3} \E[X|X_{t}=x]^3.
\end{align*}
We start by proving that for $k\geq 0$,
$$
\partial_x\op \E[X^k|X_t=x]\cp = \frac{1}{t}\op E[X^{k+1}|X_t=x] -E[X^k|X_t=x]E[X|X_t=x]\cp.
$$
Indeed,
\begin{align*}
\partial_x&\op \E[X^k|X_t=x]\cp\\
&= \partial_x\op \frac{\int x_0^ke^{-\frac{(x_0-x)^2}{2t}}dP_X(x_0)}{\int e^{-\frac{(x_0-x)^2}{2t}}dP_X(x_0)}\cp\\
&= \frac{1}{t}\op\frac{\int (x_0-x)x_0^ke^{-\frac{(x_0-x)^2}{2t}}dP_X(x_0)}{\int e^{-\frac{(x_0-x)^2}{2t}}dP_X(x_0)} - \frac{\int x_0^ke^{-\frac{(x_0-x)^2}{2t}}dP_X(x_0)\int  (x_0-x)e^{-\frac{(x_0-x)^2}{2t}}dP_X(x_0)}{\op\int e^{-\frac{(x_0-x)^2}{2t}}dP_X(x_0)\cp^2}\cp\\
&= \frac{1}{t}\op \E[(X-x)X^k|X_t=x]- \E[X^k|X_t=x]\E[(X-x)|X_t=x]\cp\\
&= \frac{1}{t}\op \E[X^{k+1}|X_t=x]- \E[X^k|X_t=x]\E[X|X_t=x]\cp.
\end{align*}
We can then compute
\begin{align*}
\partial_x^4& \log   p_t(x)\\
&= \frac{1}{t^3}(\partial_x \E[X^3|X_{t}=x] -3 \partial_x\E[X^2|X_{t}=x]\E[X|X_{t} = x]  -3 \E[X^2|X_{t}=x]\partial_x\E[X|X_{t} = x] \\
&\quad +6 \E[X|X_{t}=x]^2 \partial_x\E[X|X_{t}=x])\\
&= \frac{1}{t^4}(\E[X^4|X_{t}=x] -4\E[X^3|X_{t}=x]\E[X|X_{t} = x] -3\E[X^2|X_{t}=x]^2
\\&\quad
+ 12\E[X^2|X_{t}=x]\E[X|X_{t} = x]^2
-6\E[X|X_{t}=x]^4),
\end{align*}
and,
\begin{align*}
\partial_x^5& \log   p_t(x)\\
&= \frac{1}{t^4}(
\partial_x\E[X^4|X_{t}=x] 
-4\partial_x\E[X^3|X_{t}=x]\E[X|X_{t} = x]
-4\E[X^3|X_{t}=x]\partial_x\E[X|X_{t} = x]
\\&\quad
-6\E[X^2|X_{t}=x]\partial_x\E[X^2|X_{t}=x]
+ 12\partial_x\E[X^2|X_{t}=x]\E[X|X_{t} = x]^2
\\&\quad
+ 24\E[X^2|X_{t}=x]\E[X|X_{t} = x]\partial_x\E[X|X_{t} = x]
-24 \E[X|X_{t}=x]^3\partial_x\E[X|X_{t}=x]
),\\
&= \frac{1}{t^5}(
\E[X^5|X_{t}=x]
-5 \E[X^4|X_{t}=x] \E[X|X_{t}=x]
-10 \E[X^3|X_{t}=x] \E[X^2|X_{t}=x]
\\&\quad
+20 \E[X^3|X_{t}=x] \E[X|X_{t}=x]^2
+30 \E[X^2|X_{t}=x]^2 \E[X|X_{t}=x]
\\&\quad
-60 \E[X^2|X_{t}=x] \E[X|X_{t}=x]^3
+24 \E[X|X_{t}=x]^5
).
\end{align*}
\end{proof}

\subsubsection{Propagation of Errors}
\label{sct:proofs:bound_propagation_error}
In this section, we give lemmas controlling the propagation of errors through the steps of the different algorithms, by bounding the product of factors of the form
\begin{equation*}
\tag{\ref{eq:Lipchitz_constant}}
L_{t,h} = 1 + h\op\frac{R^2}{t^2} - \frac{1}{t}\cp.
\end{equation*}
\begin{lemma}
\label{lemma:bound_propagation_error:ODE}
Let $T,\epsilon>0$, $N\geq 1$ number of steps, denote $h = \frac{T-\epsilon}{N}$ and $t_n = nh$.  Assume that $h\leq \epsilon$, then for $n \in \{0,\dots,N\}$, we have
\begin{equation}
\label{eq:bound_propagation_error:ODE}
\prod_{m=n}^{N-1}L_{T-t_m,h/2} \leq  \sqrt{\frac{2\epsilon}{T-t_{n-1}}} \exp\op\frac{R^2}{2\epsilon}\cp,
\end{equation}
denoting $t_{-1} = -h$, and
\begin{equation}
\label{eq:bound_propagation_error:ODE_sum}
\sum_{k=1}^{N}\op\prod_{m=k}^{N-1}L_{m}\cp\int_{t_{k-1}}^{t_k} \frac{1}{(T-u)^{3/2}} du \leq  \frac{1}{\sqrt{2\epsilon}}\exp\op\frac{R^2}{2\epsilon}\cp.
\end{equation}
\end{lemma}

\begin{proof}
For all $u\in\R$, $(1+u) \leq \exp(u)$, hence for $m \in \{0,\dots,N-1\}$, we have
$$
L_{T-t_m,h/2}= 1 + \frac{h}{2}\op\frac{R^2}{(T-t_m)^2} - \frac{1}{T-t_m}\cp \leq \exp\op\frac{h}{2}\op\frac{R^2}{(T-t_m)^2} - \frac{1}{T-t_m}\cp \cp,
$$
leading to 
$$
\prod_{m=n}^{N-1}L_{T-t_m,h/2} \leq \exp\op \frac{h}{2}\sum_{m=n}^{N-1}\op\frac{R^2}{(T-t_n)^2} - \frac{1}{T-t_n}\cp \cp  = \exp\op \frac{h}{2}\sum_{k=1}^{N-n}\op\frac{R^2}{(\epsilon+ hk)^2} - \frac{1}{\epsilon+hk}\cp \cp.
$$
We bound the sums using integrals. As $t \mapsto1/t^2$ is decreasing, 
$$
h\sum_{k=1}^{N-n}\frac{R^2}{(\epsilon+ hk)^2} \leq R^2\int_{\epsilon}^{\epsilon + h(N - n)} \frac{1}{t^2} \leq R^2\int_{\epsilon}^{+\infty} \frac{1}{t^2} = \frac{R^2}{\epsilon},
$$
and $t \mapsto -1/t$ is increasing, and $h\leq \epsilon$, hence,
$$
-h\sum_{k=1}^{N-n} \frac{1}{\epsilon+hk} \leq -\int_{\epsilon+h}^{\epsilon + h(N - n) +h } \frac{1}{t}dt =  -\int_{\epsilon+h}^{T - t_{n-1}} \frac{1}{t} = \log\op\frac{\epsilon +h}{T-t_{n-1}}\cp \leq \log\op\frac{2\epsilon}{T-t_{n-1}}\cp.
$$
It leads to
\begin{equation*}
\tag{\ref{eq:bound_propagation_error:ODE}}
\prod_{m=n}^{N-1}L_{T-t_m,h/2} \leq  \sqrt{\frac{2\epsilon}{T-t_{n-1}}} \exp\op\frac{R^2}{2\epsilon}\cp.
\end{equation*}
Note that it is also valid for $n=N$, as $1\leq \sqrt{\frac{2\epsilon}{\epsilon+h}}\exp\op\frac{R^2}{2\epsilon}\cp$.
Then we have , 
\begin{align*}
\sum_{k=1}^{N}\frac{1}{\sqrt{T-t_{k-1}}}\int_{t_{k-1}}^{t_k} \frac{1}{(T-u)^{3/2}} du
&\leq \sum_{k=1}^{N}\int_{t_{k-1}}^{t_k} \frac{1}{(T-u)^{2}} du 
= \int_{0}^{T-\epsilon} \frac{1}{(T-u)^{2}} du \\
&= \int_{\epsilon}^{T} \frac{1}{u^{2}} du
\leq \int_{\epsilon}^{+\infty} \frac{1}{u^{2}} du = \frac{1}{2\epsilon},
\end{align*}
hence,
\begin{equation*}
\tag{\ref{eq:bound_propagation_error:ODE_sum}}
\sum_{k=1}^{N}\op\prod_{m=k}^{N-1}L_{m}\cp\int_{t_{k-1}}^{t_k} \frac{1}{(T-u)^{3/2}} du \leq  \frac{1}{\sqrt{2\epsilon}}\exp\op\frac{R^2}{2\epsilon}\cp.
\end{equation*}
\end{proof}

\begin{lemma}
\label{lemma:bound_propagation_error:SDE}
Let $T,\epsilon>0$, $N\geq 1$ number of steps, denote $h = \frac{T-\epsilon}{N}$ and $t_n = nh$.  Assume that $h\leq \epsilon$ , then for $n \in \{0,\dots,N\}$, we have
\begin{equation}
\label{eq:bound_propagation_error:SDE}
\prod_{m=n}^{N-1}L_{T-t_m,h} \leq \frac{2\epsilon}{T-t_{n-1}} \exp\op\frac{R^2}{\epsilon}\cp,
\end{equation}
denoting $t_{-1} = -h$,
\begin{equation}
\label{eq:bound_propagation_error:SDE_sum}
\sum_{k=1}^{N}\op\prod_{m=k}^{N-1}L_{T-t_m,h}\cp \int_{t_{k-1}}^{t_k}\frac{1}{T-s}ds  \leq 2\exp\op\frac{R^2}{\epsilon}\cp,
\end{equation}
and
\begin{equation}
\label{eq:bound_propagation_error:SDE_sum_square}
\sum_{k=1}^{N}\op\prod_{m=k}^{N-1}L_{T-t_m,h}\cp^2\int_{t_{k-1}}^{t_k}\frac{1}{(T-s)^2}ds  \leq \frac{4}{3\epsilon}\exp\op\frac{2R^2}{\epsilon}\cp.
\end{equation}
\end{lemma}

\begin{proof}
For all $u\in\R$, $(1+u) \leq \exp(u)$, hence for $m \in \{0,\dots,N-1\}$, we have
$$
L_{T-t_m,h} = 1 + h\op\frac{R^2}{(T-t_m)^2} - \frac{1}{T-t_m}\cp \leq \exp\op h\op\frac{R^2}{(T-t_m)^2} - \frac{1}{T-t_m}\cp \cp,
$$
hence 
$$
\prod_{m=n}^{N-1}L_{T-t_m,h}\leq \exp\op h\sum_{m=n}^{N-1}\op\frac{R^2}{(T-t_n)^2} - \frac{1}{T-t_n}\cp \cp  = \exp\op h\sum_{k=1}^{N-n}\op\frac{R^2}{(\epsilon+ hk)^2} - \frac{1}{\epsilon+hk}\cp \cp.
$$
As $t \mapsto1/t^2$ is decreasing, 
$$
h\sum_{k=1}^{N-n}\frac{R^2}{(\epsilon+ hk)^2} \leq R^2\int_{\epsilon}^{\epsilon + h(N - n)} \frac{1}{t^2} \leq R^2\int_{\epsilon}^{+\infty} \frac{1}{t^2} = \frac{R^2}{\epsilon},
$$
and $t \mapsto -1/t$ is increasing hence,
$$
-h\sum_{k=1}^{N-n} \frac{1}{\epsilon+hk} \leq -\int_{\epsilon+h}^{\epsilon + h(N - n) +h} \frac{1}{t} =  -\int_{\epsilon+h}^{T - t_{n-1}} \frac{1}{t} = \log\op\frac{\epsilon +h}{T-t_{n-1}}\cp \leq \log\op\frac{2\epsilon}{T-t_{n-1}}\cp.
$$
It leads to
\begin{equation*}
\tag{\ref{eq:bound_propagation_error:SDE}}
\prod_{m=n}^{N-1}L_{T-t_m,h} = \leq \frac{2\epsilon}{T-t_{n-1}} \exp\op\frac{R^2}{\epsilon}\cp.
\end{equation*}
Note that it is also valid for $n=N$, as $1\leq \frac{2\epsilon}{\epsilon+h}\exp\op\frac{R^2}{\epsilon}\cp$.
Then we have,
\begin{align*}
\sum_{k=1}^{N} \frac{1}{{T-t_{k-1}}} \int_{t_{k-1}}^{t_k}\frac{1}{T-s}ds
&\leq \sum_{k=1}^{N} \int_{t_{k-1}}^{t_k}\frac{1}{(T-s)^2}ds
= \int_{0}^{T-\epsilon}\frac{1}{(T-s)^2}ds\\
&= \int_{\epsilon}^{T}\frac{1}{u^2}du \leq \int_{\epsilon}^{+\infty}\frac{1}{u^2}du =\frac{1}{\epsilon},   
\end{align*}
hence,
\begin{equation*}
\tag{\ref{eq:bound_propagation_error:SDE_sum}}
\sum_{k=1}^{N}\op\prod_{m=k}^{N-1}L_{T-t_m,h}\cp \int_{t_{k-1}}^{t_k}\frac{1}{T-s}ds   \leq 2\exp\op\frac{R^2}{\epsilon}\cp.
\end{equation*}
Finally,
\begin{align*}
\sum_{k=1}^{N} \frac{1}{({T-t_{k-1}})^2} \int_{t_{k-1}}^{t_k}\frac{1}{(T-s)^2}ds  
&\leq \sum_{k=1}^{N} \int_{t_{k-1}}^{t_k}\frac{1}{(T-s)^4}ds 
= \int_{0}^{T-\epsilon}\frac{1}{(T-s)^4}ds\\
&= \int_{\epsilon}^{T}\frac{1}{u^4}du \leq \int_{\epsilon}^{+\infty}\frac{1}{u^4}du = \frac{1}{3\epsilon^3},
\end{align*}
leading to
\begin{equation*}
\tag{\ref{eq:bound_propagation_error:SDE_sum_square}}
h\sum_{k=1}^{N}\op\prod_{m=k}^{N-1}L_{T-t_m,h}\cp^2\int_{t_{k-1}}^{t_k}\frac{1}{(T-s)^2}ds  \leq \frac{4}{3\epsilon}\exp\op\frac{2R^2}{\epsilon}\cp.
\end{equation*}
\end{proof}

\begin{lemma}
\label{lemma:bound_propagation_error:Heun}
Let $T,\epsilon>0$, $N\geq 1$ number of steps, denote $h = \frac{T-\epsilon}{N}$, $t_n = nh$, and
$$
K_n = \op\prod_{m=n}^{N-1}\frac{L_{T-t_m,h/2}+L_{T-t_{m+1},h/2}}{2} + \frac{h^2R^4}{8\epsilon^2(T-t_{m+1})^2}\cp.
$$
Assume that $h\leq \epsilon/2$ , then for $n \in \{0,\dots,N\}$, we have
\begin{equation}
\label{eq:bound_propagation_error:Heun}
K_n \leq  \sqrt{\frac{2\epsilon}{T-t_{n-1}}} \exp\op\frac{R^2}{\epsilon} +\frac{hR^2}{4\epsilon^2}\cp,
\end{equation}
denoting $t_{-1} = -h$, and
\begin{equation}
\label{eq:bound_propagation_error:Heun_sum}
h\sum_{k=1}^{N}\frac{K_k}{(T-t_{k})^{5/2}} \leq \frac{1}{\sqrt{2}\epsilon^{3/2}} \exp\op\frac{R^2}{\epsilon} +\frac{hR^2}{4\epsilon^2}\cp.
\end{equation}
\end{lemma}

\begin{proof}
Firstly, as $t \mapsto 1/t^2$ is decreasing and as $t \mapsto -1/t$ is increasing, for $m \in \{0,\dots,N-1\}$, we have
\begin{align*}
&\frac{L_{T-t_m,h/2}+L_{T-t_{m+1},h/2}}{2} + \frac{h^2R^4}{8\epsilon^2(T-t_{m+1})^2} \\
&\qquad= 1 + \frac{h}{4}\op\frac{R^2}{(T-t_m)^2} - \frac{1}{T-t_m}\cp + \frac{h}{4}\op\frac{R^2}{(T-t_{m+1})^2} - \frac{1}{T-t_{m+1}}\cp + \frac{h^2R^4}{8\epsilon^2(T-t_{m+1})^2} \\
&\qquad\leq 1 + \op \frac{h}{2} +\frac{h^2R^2}{8\epsilon^2}\cp\frac{R^2}{(T-t_{m+1})^2} - \frac{h}{2}\frac{1}{T-t_{m}}.
\end{align*}

\noindent For all $u\in\R$, $(1+u) \leq \exp(u)$, hence
$$
\frac{L_{T-t_m,h/2}+L_{T-t_{m+1},h/2}}{2} + \frac{h^2R^4}{8\epsilon^2(T-t_{m+1})^2}  \leq \exp\op\op \frac{h}{2} +\frac{h^2R^2}{8\epsilon^2}\cp\frac{R^2}{(T-t_{m+1})^2} - \frac{h}{2}\frac{1}{T-t_{m}}\cp,
$$
leading to 
\begin{align*}
K_n &\leq \exp\op \op \frac{h}{2} +\frac{h^2R^2}{8\epsilon^2}\cp\sum_{m=n}^{N-1}\frac{R^2}{(T-t_{m+1})^2} - \frac{h}{2}\sum_{m=n}^{N-1}\frac{1}{T-t_m} + (N-n)\frac{h^2L^2}{8}\cp  \\
&= \exp\op \op \frac{h}{2} +\frac{h^2R^2}{8\epsilon^2}\cp\sum_{k=0}^{N-n-1}\frac{R^2}{(\epsilon+ hk)^2}  - \frac{h}{2}\sum_{k=1}^{N-n}\frac{1}{\epsilon+hk}  + (N-n)\frac{h^2L^2}{8}\cp.
\end{align*}
We bound the sums using integrals. As $t \mapsto1/t^2$ is decreasing, and $h\leq \epsilon/2$,
$$
h\sum_{k=0}^{N-n-1}\frac{R^2}{(\epsilon+ hk)^2} \leq R^2\int_{\epsilon-h}^{\epsilon + h(N - n - 2)} \frac{1}{t^2} \leq R^2\int_{\epsilon/2}^{+\infty} \frac{1}{t^2} = \frac{2R^2}{\epsilon},
$$
and $t \mapsto -1/t$ is increasing, and $h\leq \epsilon$, hence,
$$
-h\sum_{k=1}^{N-n} \frac{1}{\epsilon+hk} \leq -\int_{\epsilon+h}^{\epsilon + h(N - n) +h} \frac{1}{t} =  -\int_{\epsilon+h}^{T - t_{n-1}} \frac{1}{t} = \log\op\frac{\epsilon +h}{T-t_{n-1}}\cp \leq \log\op\frac{2\epsilon}{T-t_{n-1}}\cp.
$$
It leads to
\begin{equation*}
\tag{\ref{eq:bound_propagation_error:Heun}}
K_n \leq  \sqrt{\frac{2\epsilon}{T-t_{n-1}}} \exp\op\frac{R^2}{\epsilon} +\frac{hR^2}{4\epsilon^2}\cp.
\end{equation*}
Note that it is also valid for $n=N$, as $1\leq \sqrt{\frac{2\epsilon}{\epsilon+h}}\exp\op \frac{R^2}{\epsilon}+\frac{hR^2}{4\epsilon^2}\cp$.
Similarly, as $t \mapsto 1/t^3$ is decreasing, and $h\leq \epsilon/2$,
\begin{align*}
h \sum_{n=1}^{N} \frac{1}{\sqrt{T-t_{n-1}}(T-t_n)^{5/2}} 
&\leq h \sum_{n=1}^{N} \frac{1}{(T-t_n)^{3}} 
= h \sum_{k=0}^{N-1} \frac{1}{(\epsilon+kh)^{3}} \\
&\leq \int_{\epsilon-h}^{T-h}\frac{1}{t^3}dt
\leq \int_{\epsilon/2}^{+\infty}\frac{1}{t^3}dt = \frac{1}{2\epsilon^2},
\end{align*}
hence,
\begin{equation*}
\tag{\ref{eq:bound_propagation_error:Heun_sum}}
h\sum_{n=1}^{N}\frac{K_n}{(T-t_n)^{5/2}}  \leq \frac{1}{\sqrt{2}\epsilon^{3/2}} \exp\op\frac{R^2}{\epsilon} +\frac{hR^2}{4\epsilon^2}\cp.
\end{equation*}
\end{proof}

\subsection{Proof of Lemma \ref{lemma:bound_spatial_regularity}}

Using the expression of $\nabla^2\log p_t$ given by Lemma~\ref{lemma:score_moment} in Appendix~\ref{sct:score_moment}, we have
$$
\nabla^2\log p_t(x) = - \frac{1}{t}I + \frac{1}{t^2}\cov(X|X_t=x).
$$
As $X$ is supported in $B(0,R)$, for all $x$, $0 \preccurlyeq  \cov(X|X_{T-t}=x) \preccurlyeq  R^2 I$, hence we get (\ref{eq:bound_hessian}):
\begin{equation*}
\tag{\ref{eq:bound_hessian}}
 - \frac{1}{t} I\preccurlyeq  \nabla^2\log p_t(x) \preccurlyeq  \op - \frac{1}{t} + \frac{R^2}{t^2}\cp I.   
\end{equation*}
The Jacobian of the function $f_{t,h}:x \in \R^d \mapsto x + h\nabla\log p_t(x)$ is given by
$$
J_f^{t,h}(x) = I + h\nabla^2\log p_t(x).
$$
With (\ref{eq:bound_hessian}) we get that
$$
\op 1 - \frac{h}{t}\cp I
\preccurlyeq  J_f^{t,h}(x) \preccurlyeq
\op 1 + h \op\frac{R^2}{t^2}-\frac{1}{t} \cp\cp I.
$$
For $h\leq t$, $1 -\frac{h}{t} \geq 0$, hence the Lipchitz constant $L_{t,h}$ of $f_{t,h}$ is given by the second inequality, leading to 
$$
L_{t,h} = 1 + h\op\frac{R^2}{t^2} - \frac{1}{t}\cp.
$$

\subsection{Proof of Lemma \ref{lemma:bound_discretization_ODE}}

We start by noticing that:
\begin{align*}
\int_{t}^{t+h} \nabla\log \rv p_s(x_s)ds - h\nabla\log \rv p_{t}(x_{t}) 
&= \int_{t}^{t+h}\op \nabla\log \rv p_s(x_s) - \nabla\log \rv p_{t}(x_{t})\cp ds \\
&= \int_{t}^{t+h} \int_{t}^{s} \frac{d}{du}\nabla\log \rv p_u(x_u)duds.
\end{align*}

\noindent\textbf{Expressing $\frac{d}{dt}\nabla \log \rv p_t(x_t)$.}
Using the chain rule and equation (\ref{eq:reverse_ODE}), we have
\begin{align*}
\frac{d}{dt}\nabla \log \rv p_t(x_t) &= \nabla^2\log \rv p_t(x_t)\cdot\frac{d}{dt}x_t + [\partial_t\nabla \log\rv p_t](x_t) \\
&= \frac{1}{2}\nabla^2\log \rv p_t(x_t)\cdot\nabla\log \rv p_t(x_t) + [\nabla \partial_t\log \rv p_t](x_t).
\end{align*} 

\noindent Moreover, the Fokker-Planck equation for $p_t$ is
$$
\partial_t p_t = \frac{1}{2} \Delta  p_t = \frac{1}{2} \nabla\cdot\nabla p_t =\frac{1}{2} \nabla\cdot\op p_t \nabla\log p_t\cp = \frac{1}{2} \nabla p_t \cdot \nabla\log p_t + \frac{1}{2} p_t \Delta  \log p_t,
$$
hence,
$$
\partial_t \log p_t = \frac{1}{2} \Vert \nabla \log p_t\Vert^2 + \frac{1}{2}\Delta  \log p_t.
$$
Taking the gradient in $x$ gives
$$
\nabla\partial_t \log p_t = \nabla^2 \log p_t \cdot \nabla \log p_t + \frac{1}{2}\nabla\Delta  \log p_t,
$$
hence
\begin{equation}
\label{eq:dt_score}
\nabla\partial_t \log \rv{p}_t = -\nabla^2 \log \rv{p}_t \cdot \nabla \log \rv{p}_t - \frac{1}{2}\nabla\Delta  \log \rv{p}_t.  
\end{equation}

\noindent This finally leads to 
\begin{equation}
\label{eq:dt_score2}
\frac{d}{dt}\nabla \log \rv p_t(x_t) = -\frac{1}{2}\nabla^2\log \rv p_t(x_t)\cdot\nabla\log \rv p_t(x_t) - \frac{1}{2}\nabla\Delta  \log \rv p_t (x_t).
\end{equation}

\noindent\textbf{Controlling the error.} We use the expression of $\nabla\log p_t$, $\nabla^2\log p_t$ and $\nabla\Delta  \log p_t$ given by Lemma~\ref{lemma:score_moment} in Appendix~\ref{sct:score_moment}:
\begin{align*}
\nabla\log \rv p_t(x) &=  -\frac{1}{T-t}\E[B_{T-t}|X_{T-t} = x],\\
\nabla^2\log \rv p_t(x) &= - \frac{1}{T-t}I + \frac{1}{(T-t)^2}\cov(X|X_{T-t}=x),\\
\nabla\Delta  \log  \rv p_t(x) &=
\frac{1}{(T-t)^3} \E[\Vert X\Vert^2(X-\E[X|X_{T-t}=x] )|X_{T-t}=x] \\
&\qquad\qquad- \frac{2}{(T-t)^3} \cov(X|X_{T-t}=x) \cdot\E[X|X_{T-t} = x].
\end{align*}

\noindent With Lemma~\ref{lemma:bound_spatial_regularity}, for all $x\in \R^d$,
$$
\Vert \nabla^2\log \rv p_t(x)\Vert_\textnormal{op} \leq C_{T-t} = \max \op \frac{1}{T-t}, \left|\frac{R^2}{(T-t)^2}- \frac{1}{T-t}\right|\cp.
$$
As we assume that $\epsilon \leq R^2$, and $t \leq T-\epsilon$, we have $C_{T-t} \leq \frac{R^2}{\epsilon(T-t)}$.
Noticing that $x_t \sim X_{T-t}$, it leads to 
\begin{align*}
\E[\Vert\nabla^2\log \rv p_t(x_t)\cdot\nabla\log \rv p_t(x_t)\Vert^2]
&\leq \frac{R^4}{\epsilon^2(T-t)^2} \E[\Vert\nabla\log \rv p_t(x_t) \Vert^2] \\
&= \frac{R^4}{\epsilon^2}\frac{1}{(T-t)^4}\E[\Vert\E[B_{T-t}|X_{T-t}] \Vert^2] \\
\text{(Jensen's inequality for the conditional expectation)} 
&\leq  \frac{R^4}{\epsilon^2}\frac{1}{(T-t)^4} \E[\E[\Vert B_{T-t}\Vert^2|X_{T-t}]]\\
&= \frac{R^4}{\epsilon^2} \frac{1}{(T-t)^4} \E[\E[\Vert B_{T-t}\Vert^2]]\\
&=\frac{dR^4}{\epsilon^2(T-t)^3}.
\end{align*}
Similarly, as $X$ is supported in $B(0,R)$,  we get that  for all $x$, $0 \preccurlyeq  \cov(X|X_{T-t}=x) \preccurlyeq  R^2 I$, hence 
\begin{align*}
\E[\Vert\cov(X|X_{T-t}) \cdot\E[X|X_{T-t}]\Vert^2] 
&\leq  R^4 \E[\Vert \E[X|X_{T-t}] \Vert^2]\\
&\leq  R^4 \E[\Vert X\Vert^2] \leq R^6,
\end{align*}
and
\begin{align*}
\E[\Vert\E[\Vert X\Vert^2(X-\E[X|X_{T-t}] )|X_{T-t}]\Vert^2] 
&\leq  R^4 \E[\Vert X - \E[X|X_{T-t}]\Vert^2] \\
\text{(conditional expectation minimizes least square error)} &\leq  R^4 \E[\Vert X\Vert^2] \leq R^6.
\end{align*}
Combining these bounds gives
\begin{align*}
\Vert &\nabla\Delta  \log  \rv p_t(x)\Vert_{L_2} \\
&\leq \frac{1}{(T-t)^3}\Vert\E[\Vert X\Vert^2(X-\E[X|X_{T-t}] )|X_{T-t}]\Vert_{L_2}+ \frac{2}{(T-t)^3}\Vert\cov(X|X_{T-t}) \cdot \E[X|X_{T-t}]\Vert_{L_2} \\
&\leq \frac{3}{(T-t)^3}R^3,
\end{align*}
and, as $\epsilon\leq R^2$ and $t\leq T-\epsilon$,
\begin{align*}
\left\Vert \frac{d}{dt}\nabla \log \rv p_t(x_t)\right\Vert_{L_2} 
&\leq \frac{1}{2}\Vert\nabla^2\log \rv p_t(x_t)\cdot\nabla\log \rv p_t(x_t)\Vert_{L_2} + \frac{1}{2}\Vert \nabla\Delta  \log  \rv p_t(x)\Vert_{L_2}\\
&\leq \frac{\sqrt{d}}{2}\frac{R^2}{\epsilon(T-t)^{3/2}} +\frac{1}{2} \frac{3R^3}{(T-t)^3} \leq 2\sqrt{d}\frac{R^3}{\epsilon^{3/2}(T-t)^{3/2}}.
\end{align*}
This finally leads to 
\begin{align*}
\left\Vert \int_{t}^{t+h} \nabla\log \rv p_s(x_s)ds - h\nabla\log \rv p_{t}(x_{t}) \right\Vert_{L_2} 
&\leq \int_{t}^{t+h} \int_{t}^{s} \left\Vert\frac{d}{du}\nabla\log \rv p_u(x_u)\right\Vert_{L_2} duds\\
&\leq  2\sqrt{d}\frac{R^3}{\epsilon^{3/2}}\int_{t}^{t+h} \int_{t}^{s} \frac{1}{(T-u)^{3/2}} duds\\
&\leq2\sqrt{d}\frac{R^3}{\epsilon^{3/2}}h \int_{t}^{t+h} \frac{1}{(T-u)^{3/2}} du. 
\end{align*}

\subsection{Proof of Lemma \ref{lemma:bound_discretization_Heun}}

\noindent\textbf{Rewriting the difference.} We start by denoting
$$
f : h \mapsto \int_{t}^{t+h} \nabla\log \rv p_s(x_s)ds - \frac{h}{2}\op\nabla\log \rv p_{t}(x_{t})+\nabla\log \rv p_{t+h}(x_{t+h})\cp.
$$
We have $f(0) = 0$,
$$
f'(h) = \nabla\log \rv p_{t+h}(x_{t+h}) - \frac{1}{2}\op\nabla\log \rv p_{t}(x_{t})+\nabla\log \rv p_{t+h}(x_{t+h})\cp  - \frac{h}{2} \frac{d}{dh}\op\nabla\log \rv p_{t+h}(x_{t+h})\cp,
$$
with $f'(0) = 0$, and,
\begin{align*}
f''(h) 
&= \frac{d}{dh}\op\nabla\log \rv p_{t+h}(x_{t+h})\cp - \frac{1}{2}\frac{d}{dh}\op\nabla\log \rv p_{t+h}(x_{t+h})\cp  \\
&\quad- \frac{1}{2} \frac{d}{dh}\op\nabla\log \rv p_{t+h}(x_{t+h})\cp - \frac{h}{2} \frac{d^2}{dh^2}\op\nabla\log \rv p_{t+h}(x_{t+h})\cp \\
&= - \frac{h}{2} \frac{d^2}{dh^2}\op\nabla\log \rv p_{t+h}(x_{t+h})\cp.    
\end{align*}
Integrating back two times gives
$$
f(h) = \int_0^h\int_0^s - \frac{u}{2} \frac{d^2}{du^2}\op\nabla\log \rv p_{t+u}(x_{t+u})\cp duds,
$$
i.e.,
\begin{multline*}
\int_{t}^{t+h} \nabla\log \rv p_s(x_s)ds - \frac{h}{2}\op\nabla\log \rv p_{t}(x_{t})+\nabla\log \rv p_{t+h}(x_{t+h})\cp =\\ \int_t^{t+h}\int_{t}^{t+s} - \frac{(u-t)}{2} \frac{d^2}{du^2}\op\nabla\log \rv p_{u}(x_{u})\cp duds.
\end{multline*}

\noindent We then use the following Lemma, whose proof is given below:
\begin{lemma}
\label{lemma:bound_second_derivative_score}
For $t\geq 0$, $t\leq T - \epsilon$ and  $\epsilon \leq R^2$, we have
$$\left\Vert\frac{d^2}{dt^2}\nabla \log \rv p_t(x_t)\right\Vert_{L_2}\leq 66d\frac{R^5}{\epsilon^{5/2}(T-t)^{5/2}}.$$
\end{lemma}

\noindent This finally leads to:
\begin{multline*}
\left\Vert \int_{t}^{t+h} \nabla\log \rv p_s(x_s)ds - \frac{h}{2}\op\nabla\log \rv p_{t}(x_{t})+\nabla\log \rv p_{t+h}(x_{t+h})\cp \right\Vert_{L_2} \\
\leq \int_t^{t+h}\int_{t}^{t+s} \frac{(u-t)}{2} \left\Vert\frac{d^2}{du^2}\op\nabla\log \rv p_{u}(x_{u})\cp\right\Vert_{L_2} duds
\leq 33d\frac{R^5}{\epsilon^{5/2}}h^2\int_{t}^{t+h} \frac{1}{(T-u)^{5/2}} du.
\end{multline*}

\begin{proof}[Proof (Lemma~\ref{lemma:bound_second_derivative_score})]
As the computations involve spatial derivatives up to order five, for ease of notations, we only tackle here the case $d=1$.
From (\ref{eq:dt_score2}) in the proof of Lemma~\ref{lemma:bound_discretization_ODE}, we know that
\begin{equation*}
\tag{\ref{eq:dt_score2}}
\frac{d}{dt}\nabla \log \rv p_t(x_t) = \frac{d}{dt}\partial_x \log \rv p_t(x_t) = -\frac{1}{2}\partial_x^2\log \rv p_t(x_t)\partial_x\log \rv p_t(x_t) - \frac{1}{2}\partial_x^3  \log \rv p_t (x_t).    
\end{equation*}
Then we get that
\begin{align*}
\frac{d^2}{dt^2}\nabla \log \rv p_t(x_t) 
&= 
-\frac{1}{2}\frac{d}{dt}\op\partial_x^2\log \rv p_t(x_t)\cp\partial_x\log \rv p_t(x_t) 
-\frac{1}{2}\partial_x^2\log \rv p_t(x_t)\frac{d}{dt}\op\partial_x\log \rv p_t(x_t)\cp\\
&\quad
- \frac{1}{2}\frac{d}{dt}\partial_x^3  \log \rv p_t (x_t) \\
&= 
-\frac{1}{2}\partial_x^3\log \rv p_t(x_t)\frac{d}{dt}\op x_t\cp\partial_x\log \rv p_t(x_t) 
-\frac{1}{2}\partial_t\partial_x^2\log \rv p_t(x_t)\partial_x\log \rv p_t(x_t)\\
&\quad
-\frac{1}{2}\partial_x^2\log \rv p_t(x_t)\partial_x^2\log \rv p_t(x_t)\frac{d}{dt}\op x_t\cp
-\frac{1}{2}\partial_x^2\log \rv p_t(x_t)\partial_t\partial_x\log \rv p_t(x_t)\\
&\quad 
- \frac{1}{2}\partial_x^4  \log \rv p_t (x_t) \frac{d}{dt}\op x_t\cp
- \frac{1}{2}\partial_t\partial_x^3  \log \rv p_t (x_t)
\\
&= 
-\frac{1}{4}\partial_x^3\log \rv p_t(x_t)\op\partial_x\log \rv p_t(x_t)\cp^2
-\frac{1}{2}\partial_t\partial_x^2\log \rv p_t(x_t)\partial_x\log \rv p_t(x_t)\\
&\quad
-\frac{1}{4}\op\partial_x^2\log \rv p_t(x_t)\cp^2\partial_x\log \rv p_t(x_t)
-\frac{1}{2}\partial_x^2\log \rv p_t(x_t)\partial_t\partial_x\log \rv p_t(x_t)\\
&\quad 
- \frac{1}{4}\partial_x^4  \log \rv p_t (x_t)\partial_x\log \rv p_t(x_t)
- \frac{1}{2}\partial_t\partial_x^3  \log \rv p_t (x_t).
\end{align*}
From (\ref{eq:dt_score}) in the proof of Lemma~\ref{lemma:bound_discretization_ODE}, we know that
\begin{equation*}
\tag{\ref{eq:dt_score}}
\partial_t \partial_x\log \rv{p}_t = -\partial_x^2 \log \rv{p}_t \partial_x \log \rv{p}_t - \frac{1}{2}\partial_x^3  \log \rv{p}_t,
\end{equation*}
from which we also deduce,
$$
\partial_t \partial_x^2\log \rv{p}_t = 
-\partial_x^3 \log \rv{p}_t  \partial_x \log \rv{p}_t 
-\op\partial_x^2 \log \rv{p}_t\cp^2
-\frac{1}{2}\partial_x^4  \log \rv{p}_t,
$$
and 
$$
\partial_t \partial_x^3\log \rv{p}_t = 
-\partial_x^4 \log \rv{p}_t  \partial_x \log \rv{p}_t
-3\partial_x^3 \log \rv{p}_t  \partial_x^2 \log \rv{p}_t 
-\frac{1}{2}\partial_x^5  \log \rv{p}_t.
$$
Combining these expressions gives:
\begin{align}
\label{eq:bound_second_derivative_score:decompdt2}
\frac{d^2}{dt^2}\nabla \log \rv p_t(x_t) 
&=
\frac{1}{4}\partial_x^5\log \rv p_t(x_t)
+ \frac{1}{2}\partial_x^4  \log \rv p_t (x_t)\partial_x\log \rv p_t(x_t)
+ \frac{7}{4}\partial_x^3\log \rv p_t(x_t)\partial_x^2\log \rv p_t(x_t) \notag\\
&\quad
+\frac{1}{4}\partial_x^3\log \rv p_t(x_t)\op\partial_x\log \rv p_t(x_t)\cp^2
+\frac{3}{4}\op\partial_x^2\log \rv p_t(x_t)\cp^2\partial_x\log \rv p_t(x_t).
\end{align}
The expressions of $\partial_x\log p_t$, $\partial_x^2\log p_t$ and $\partial_x^3 \log p_t$ are given by Lemma~\ref{lemma:score_moment} in Appendix~\ref{sct:score_moment}:
\begin{align*}
\partial_x\log \rv p_t(x) &=  -\frac{1}{T-t}\E[B_{T-t}|X_{T-t} = x],\\
\partial_x^2\log \rv p_t(x) &= - \frac{1}{T-t} + \frac{1}{(T-t)^2}\cov(X|X_{T-t}=x),\\
\partial_x^3 \log  \rv p_t(x) &= \frac{1}{(T-t)^3} \E[X^2(X-\E[X|X_{T-t}=x] )|X_{T-t}=x]\\
&\quad- \frac{2}{(T-t)^3} \cov(X|X_{T-t}=x) \cdot\E[X|X_{T-t} = x]\\
&= \frac{1}{(T-t)^3} \E[X^3|X_{T-t}=x] - \frac{3}{(T-t)^3} \E[X^2|X_{T-t}=x]\E[X|X_{T-t} = x]\\
&\quad- \frac{2}{(T-t)^3} \E[X|X_{T-t}=x]^3.
\end{align*}
and similarly the expressions of of $\partial_x^4\log p_t$ and $\partial_x^5 \log p_t$ are given by Lemma~\ref{lemma:score_moment_higher_order}:
\begin{align*}
\partial_x^4 \log   p_t(x)
&= \frac{1}{t^4}(\E[X^4|X_{t}=x] -4\E[X^3|X_{t}=x]\E[X|X_{t} = x] -3\E[X^2|X_{t}=x]^2
\\&\quad 
+ 12\E[X^2|X_{t}=x]\E[X|X_{t} = x]^2
-6\E[X|X_{t}=x]^4),\\
\partial_x^5 \log   p_t(x)
&= \frac{1}{t^5}(
\E[X^5|X_{t}=x]
-5 \E[X^4|X_{t}=x] \E[X|X_{t}=x]
-10 \E[X^3|X_{t}=x] \E[X^2|X_{t}=x]
\\&\quad
+20 \E[X^3|X_{t}=x] \E[X|X_{t}=x]^2
+30 \E[X^2|X_{t}=x]^2 \E[X|X_{t}=x]
\\&\quad
-60 \E[X^2|X_{t}=x] \E[X|X_{t}=x]^3
+24 \E[X|X_{t}=x]^5
).
\end{align*}
As $X \in B(0,R)$ almost surely, all the conditional moments are bounded by the corresponding power of $R$, hence
\begin{align*}
|\partial_x^3 \log \rv p_{t}(x)| &\leq \frac{6R^3}{(T-t)^3},\\
|\partial_x^3 \log \rv p_{t}(x)| &\leq \frac{26R^4}{(T-t)^4},\\
|\partial_x^3 \log \rv p_{t}(x)| &\leq \frac{150R^5}{(T-t)^5}.
\end{align*}
With $\epsilon \leq R^2$ and $t \leq T-\epsilon$, we also have
$$
|\partial_x^2 \log \rv p_{T-t}(x)| \leq \max \op \frac{1}{T-t}, \left|\frac{R^2}{(T-t)^2}- \frac{1}{T-t}\right|\cp \leq \frac{R^2}{\epsilon(T-t)}.
$$
Finally, we have, 
\begin{align*}
\left\Vert\partial_x \log \rv p_{T-t}(x_{T-t})\right\Vert_{L_2} 
&=\op\E\ob\left\Vert-\frac{1}{T-t}\E[B_{T-t}|X_{T-t} = x] \right\Vert^2\cb\cp^{1/2}\\
&\leq \frac{1}{T-t} \op\E\ob\Vert B_{T-t} \Vert^2\cb\cp^{1/2} = \frac{1}{\sqrt{T-t}} \sqrt{d},
\end{align*}
and
\begin{align*}
\left\Vert(\partial_x \log \rv p_{T-t}(x_{T-t}))^2\right\Vert_{L_2} 
&=\op\E\ob\left\Vert-\frac{1}{T-t}\E[B_{T-t}|X_{T-t} = x] \right\Vert^4\cb\cp^{1/2}\\
&\leq \frac{1}{(T-t)^2} \op\E\ob\Vert B_{T-t} \Vert^4\cb\cp^{1/2} = \frac{1}{T-t}\sqrt{3}d.
\end{align*}

\noindent Combining these bounds in (\ref{eq:bound_second_derivative_score:decompdt2}), with $T-t\geq \epsilon$ and $\epsilon \leq R^2$ leads to 
\begin{align*}
\left\Vert\frac{d^2}{dt^2}\nabla \log \rv p_t(x_t)\right\Vert_{L_2}
&\leq \frac{1}{4} \frac{150R^5}{(T-t)^5} + \frac{1}{2} \frac{26R^4}{(T-t)^4}\frac{1}{\sqrt{T-t}} \sqrt{d} + \frac{7}{4} \frac{6R^3}{(T-t)^3} \frac{R^2}{\epsilon(T-t)}\\
&\quad + \frac{1}{2}\frac{6R^3}{(T-t)^3}\frac{1}{T-t}\sqrt{3}d + \frac{3}{4} \op\frac{R^2}{\epsilon(T-t)}\cp^2  \frac{1}{\sqrt{T-t}} \sqrt{d}\\
&\leq d\frac{R^5}{\epsilon^{5/2}(T-t)^{5/2}} \op \frac{150}{4} + \frac{26}{2} + \frac{42}{4} + \frac{6\sqrt{3}}{4} + \frac{3}{4}\cp \leq 66d\frac{R^5}{\epsilon^{5/2}(T-t)^{5/2}}.    
\end{align*}
\end{proof}

\subsection{Proof of Lemma \ref{lemma:bound_discretization_SDE}}

We start by noticing that
\begin{align*}
\int_t^{t+h} \nabla\log\rv{p}_s(\rv X_s)ds - h\nabla\log\rv{p}_t(\rv X_t) 
&=\int_t^{t+h} \op\nabla\log\rv{p}_t(\rv X_s) - \nabla\log\rv{p}_t(\rv X_t)\cp ds\\
&= \int_t^{t+h}\int_t^s d\op\nabla\log\rv{p}_u(\rv X_u)\cp ds.    
\end{align*}

\noindent\textbf{Expressing $d\nabla \log \rv{p}_t(\rv X_t)$.}
Using Itō's formula, (\ref{eq:reverse_SDE}) and (\ref{eq:dt_score}), we get:
\begin{align*}
d\nabla \log \rv{p}_t(\rv X_t) &= \nabla^2\log \rv{p}_t(\rv X_t)\cdot d\rv X_t + [\partial_t\nabla \log \rv{p}_t](\rv X_t)dt + \frac{1}{2}\Delta\nabla \log \rv{p}_t(X_t) d\langle \rv X,\rv X\,\rangle_t \\
&= \nabla^2\log \rv{p}_t(\rv X_t)\cdot \nabla\log\rv{p}_t(\rv X_t)dt + \nabla^2\log \rv{p}_t(\rv X_t)\cdot dW_t\\
&\quad -\nabla^2 \log \rv{p}_t(\rv X_t) \cdot \nabla \log \rv{p}_t(X_t)dt - \frac{1}{2}\nabla\Delta  \log \rv{p}_t(\rv X_t)dt \\
&\quad + \frac{1}{2}\Delta\nabla \log \rv{p}_t(\rv X_t) dt \\
&= \nabla^2\log \rv{p}_t(\rv X_t)\cdot dW_t.
\end{align*}
We deduce that
$$
\int_t^{t+h} \nabla\log\rv{p}_s(\rv X_s)ds - h\nabla\log\rv{p}_t(\rv X_t)= \int_t^{t+h}\int_t^s \nabla^2\log \rv{p}_u(\rv X_u)\cdot dW_u ds.
$$
\noindent\textbf{Controlling the error.} Using Ito's isometry,
\begin{align*}
\left\Vert \int_t^{t+h} \nabla\log\rv{p}_s(\rv X_s)ds - h\nabla\log\rv{p}_t(\rv X_t)\right\Vert_{L_2}
&\leq \int_t^{t+h} \left\Vert\int_t^s \nabla^2\log \rv{p}_u(\rv X_u)\cdot dW_u\right\Vert_{L_2}\\
&= \int_t^{t+h}  \op \E\ob \int_t^s\left\Vert\nabla^2\log \rv{p}_u(\rv X_u)\right\Vert_{\textnormal{F}}^2du\cb\cp^{1/2}ds.
\end{align*}
With (\ref{eq:bound_hessian}) from Lemma~\ref{lemma:bound_spatial_regularity}, we know that all eigenvalues of $\nabla^2\log \rv p_u(x)$ are bounded by $\max \op \frac{1}{T-u}, \left|\frac{R^2}{(T-u)^2}- \frac{1}{T-u}\right|\cp$, which is less than $\frac{R^2}{\epsilon(T-t)},$ as $\epsilon \leq R^2$ and $u \leq (t+h) \leq T-\epsilon$.
It implies that
$$\Vert \nabla^2\log \rv p_u(x)\Vert_{\textnormal{F}}^2 \leq d \frac{R^4}{\epsilon^2(T-t)^2},$$
so finally,
\begin{align*}
\left\Vert \int_t^{t+h} \nabla\log\rv{p}_s(\rv X_s)ds - h\nabla\log\rv{p}_t(\rv X_t)\right\Vert_{L_2}
&\leq  \sqrt{d}\frac{R^2}{\epsilon}\int_t^{t+h}  \op \int_t^s \frac{1}{(T+u)^2}du\cp^{1/2}ds\\
&= \sqrt{d}\frac{R^2}{\epsilon}\int_t^{t+h}  \op \frac{t-s}{2(T-t)(T-s)}\cp^{1/2}ds\\
&\leq \frac{\sqrt{2d}}{2}\frac{R^2}{\epsilon} \sqrt{h} \int_t^{t+h}\frac{1}{T-s}ds.
\end{align*}

\subsection{Proof of Proposition \ref{prop:init_error}}

\noindent\textbf{First bound.} We bound the Wasserstein distance by exhibiting specific couplings between $Y$ and $\hat{Y}$. 
We start by taking $\gamma Z' = \sqrt{\gamma^2 - \beta^2}Z'' + \beta Z$, with $Z'' \sim \N(0,I)$ independent form $Z$, leading to 
$$
Y - \hat{Y} = \alpha X + \sqrt{\gamma^2 - \beta^2}Z'' = \alpha \op X -\frac{\sqrt{\gamma^2 -\beta^2}}{\alpha}Z''\cp.
$$
Then choosing the optimal coupling between $X$ and $Z''$ such that 
$$
\left\Vert X -\frac{\sqrt{\gamma^2 -\beta^2}}{\alpha}Z''\right\Vert_{L_2} = W_2\op X,\N\op 0,\frac{\gamma^2 -\beta^2}{\alpha^2}I\cp\cp,
$$
we get that 
$$
W_2(\L(Y),\L(\hat{Y})) \leq \Vert Y - \hat{Y}\Vert_{L_2} = \alpha W_2\op X,\N\op 0,\frac{\gamma^2 -\beta^2}{\alpha^2}I\cp\cp.
$$

\noindent\textbf{Second bound.} To get the second bound, we write:
$$
Y = \alpha \op X + \frac{\beta}{\alpha}Z\cp,
$$
and,
$$
\hat{Y} = \alpha \op \sqrt{\frac{\gamma^2 -\beta^2}{\alpha^2}}Z'' + \frac{\beta}{\alpha}Z'''\cp,
$$
with $Z'',Z'''\sim  \N(0,I)$, $Z''\bot Z'''$.
In other words, denoting $\mu$ the distribution of $X$ and $\nu = \N\op 0,\frac{\gamma^2 -\beta^2}{\alpha^2}I\cp$, $\rho_t = \N (0,tI)$, we have:
$$
W_2(\L(Y),\L(\hat{Y})) = \alpha W_2(\mu *\rho_{\beta^2/\alpha^2},\nu * \rho_{\beta^2/\alpha^2}).$$
To conclude, with the assumption that as $\E[X] =0$ and $\E\ob e^{\xi X^2}\cb < \infty$, for some $\xi>0$, as $\beta/\alpha \rightarrow+\infty$, Therorem~2.1 of \citet{chenAsymptoticsSmoothedWasserstein2022} gives the following asymptotic behavior:
\begin{align*}
W_2(\mu *\rho_{\alpha^2/\beta^2},\nu * \rho_{\alpha^2/\beta^2}) &\sim \frac{\alpha}{\beta} \op\frac{1}{4}\sum_{i\in \{1,...,n\}}\op\E[X_i^2]- \frac{\gamma^2 -\beta^2}{\alpha^2}\cp^2 +  \frac{1}{4}\sum_{i\neq j} \E[X_iX_j]^2 \cp^{1/2}\\
&=\frac{\alpha^2}{2\beta}\left\Vert \Sigma - \frac{\gamma^2 -\beta^2}{\alpha^2}I\right\Vert_\textnormal{F},
\end{align*}
with $\Sigma = \E\ob XX^\top\cb$.

\subsection{Proof of Corollary \ref{cor:init_error}}

The first bound is a direct application of Proposition~\ref{prop:init_error} with $\alpha = 1$, $\beta = \sqrt{T}$ and $\gamma = \sqrt{T}$.
To get the asymptotic behavior, we need to prove that there exists $\xi >0$ such that $\E\ob e^{\xi \Vert X\Vert^2}\cb < \infty$. Under Assumption 1, for any $\xi >0$, we have $\E\ob e^{\xi \Vert X\Vert^2}\cb \leq  e^{\xi R^2}< \infty$. Then, with $\E[X]=0$, Proposition~\ref{prop:init_error} gives
$$
W_2(\L(X_T),\L(\hat{X}_0)) \sim \frac{1}{2} \frac{\Vert \Sigma\Vert_\textnormal{F}}{\sqrt{T}},
$$
hence
$$
\frac{W_2(\L(X_T),\L(\hat{X}_0))}{\Vert \Sigma\Vert_\textnormal{F}/\sqrt{T}} \xrightarrow[T\rightarrow\infty]{} \frac{1}{2} < 1.
$$
In particular, for $T$ large enough, i.e., $T \geq C$ with $C$ that only depends on $\L(X)$, we have 
$$
W_2(\L(X_T),\L(\hat{X}_0)) \leq \frac{\Vert \Sigma\Vert_\textnormal{F}}{\sqrt{T}}.
$$
We conclude by noticing that under Assumption 1, as $\Sigma$ is symmetric positive semi-definite, denoting $\lambda_i$ its eigenvalues, we have
$$
\Vert \Sigma\Vert_\textnormal{F} = \op\sum_i \lambda_i^2\cp^{1/2}
\leq \sum_i |\lambda_i| =\sum_i \lambda_i = \tr[\Sigma] = \E[\tr[XX^\top]] = \Vert X \Vert_{L_2}^2 
\leq R^2.
$$

\subsection{Proof of Lemma \ref{lemma:bound_early_stopping_error}}

We write that, as $(X_0,X_\epsilon)$ is a particular coupling between $\L(X)$ and $\L(X_{\epsilon})$,
$$
W_2(\L(X),\L(X_{\epsilon})) = W_2(\L(X_0),\L(X_{\epsilon}))
\leq \Vert X_0 - X_{\epsilon}\Vert_{L_2}  
= \Vert B_\epsilon\Vert_{L_2} 
= \sqrt{d \epsilon} .
$$

\subsection{Proof of Proposition \ref{prop:convergence_ODE_emp_score}}

To bound the Wasserstein distance between $X$ and $\hat{X} = \hat{X}_N$ the output of the algorithm, we construct a specific coupling between the two variables. Here, the sampler is deterministic, so we only choose the coupling between $x_0 = X_T$ and $\hat{X}_0 \sim \N(0,TI)$, such that,
\begin{equation}
\label{eq:convergence_ODE_emp_score:init}
\Vert x_0 - \hat{X}_0\Vert_{L_2} =\Vert X_T - \hat{X}_0\Vert_{L_2} = W_2(\L(X_T),\L(\hat{X}_0)),
\end{equation}
and follow the evolution of the ODE and its discretization to get a coupling between  $X$ and $\hat{X}$.

\noindent\textbf{Bounding the error at step $n$.} We first look at the error at each discretization step. For $0 < n \leq N$, we have,
$$
x_{t_{n}}= x_{t_{n-1}}+ \frac{1}{2}\int_{t_{n-1}}^{t_{n}} \nabla\log \rv p_t(x_t)dt,
$$
and,
$$
\hat{X}_n = \hat{X}_{n-1} +  \frac{h}{2}s_\theta(T-t_{n-1},\hat{X}_{n-1}),
$$
hence,
\begin{align*}
x_{t_{n}} - \hat{X}_n &= x_{t_{n-1}} - \hat{X}_{n-1} + \frac{1}{2}\int_{t_{n-1}}^{t_{n}} \nabla\log \rv p_t(x_t)dt  - \frac{h}{2}s_\theta(T-t_{n-1},\hat{X}_{n-1})\\
&= x_{t_{n-1}} - \hat{X}_{n-1} \\
&\quad + \frac{1}{2}\op \int_{t_{n-1}}^{t_{n}} \nabla\log \rv p_t(x_t)dt - h\nabla\log \rv p_{t_{n-1}}(x_{t_{n-1}})\cp \\
&\quad + \frac{h}{2}\op\nabla\log \rv p_{t_{n-1}}(x_{t_{n-1}}) - s_\theta(T-t_{n-1},x_{t_{n-1}})\cp\\
&\quad + \frac{h}{2}\op s_\theta(T-t_{n-1},x_{t_{n-1}}) - s_\theta(T-t_{n-1},\hat{X}_{n-1})\cp\\
&= \op \op I +\frac{h}{2}s_\theta(T-t_{n-1},\cdot)\cp(x_{t_{n-1}}) - \op I +\frac{h}{2}s_\theta(T-t_{n-1},\cdot)\cp(\hat{X}_{n-1})\cp\\
&\quad + \frac{1}{2}\op \int_{t_{n-1}}^{t_{n}} \nabla\log \rv p_t(x_t)dt - h\nabla\log \rv p_{t_{n-1}}(x_{t_{n-1}})\cp \\
&\quad + \frac{h}{2}\op\nabla\log \rv p_{t_{n-1}}(x_{t_{n-1}}) - s_\theta(T-t_{n-1},x_{t_{n-1}})\cp.
\end{align*}
As $x_{t_{n-1}} \sim X_{T-t_{n-1}}$, we have,
\begin{align*}
\Vert\nabla\log \rv p_{t_{n-1}}&(x_{t_{n-1}}) - s_\theta(T-t_{n-1},x_{t_{n-1}})\Vert_{L_2}\\
&= \Vert \nabla \log p_{T-t_{n-1}}(X_{T-t_{n-1}}) - s_\theta(T-t_{n-1},X_{T-t_{n-1}})\Vert_{L_2}
= \ep_{\text{score}}(T-t_{n-1}).
\end{align*}
and with Assumption 2 on $x \mapsto x + h s_\theta(t,x)$ , we get that 
\begin{multline*}
\left\Vert\op I +\frac{h}{2}s_\theta(T-t_{n-1},\cdot)\cp(x_{t_{n-1}}) - \op I +\frac{h}{2}s_\theta(T-t_{n-1},\cdot)\cp(\hat{X}_{n-1})\right\Vert_{L_2} 
\\\leq  L_{n-1}\Vert x_{t_{n-1}} - \hat{X}_{n-1}\Vert_{L_2},
\end{multline*}
with $L_{n-1} = L_{T-t_{n-1},h/2} = 1 + \frac{h}{2}\op\frac{R^2}{(T-t_{n-1})^2} - \frac{1}{T-t_{n-1}}\cp.$
With Lemma~\ref{lemma:bound_discretization_ODE}, we get that 
$$
\left\Vert \int_{t_{n-1}}^{t_{n}} \nabla\log \rv p_t(x_t)dt - h\nabla\log \rv p_{t_{n-1}}(x_{t_{n-1}})\right\Vert_{L_2} \leq 2\sqrt{d}\frac{R^3}{\epsilon^{3/2}}h \int_{t_{n-1}}^{t_n} \frac{1}{(T-u)^{3/2}} du,
$$
leading to,
\begin{equation}
\label{eq:convergence_ODE_emp_score:bound_step_n}
\Vert x_{t_{n}} - \hat{X}_n\Vert_{L_2} \leq L_{n-1}\Vert x_{t_{n-1}} - \hat{X}_{n-1}\Vert_{L_2}  +\sqrt{d}\frac{R^3}{\epsilon^{3/2}}h \int_{t_{n-1}}^{t_n} \frac{1}{(T-u)^{3/2}} du + \frac{h}{2}\ep_{\text{score}}(T-t_{n-1}).    
\end{equation}

\noindent\textbf{Bounding the error at time $T-\epsilon.$} With (\ref{eq:convergence_ODE_emp_score:bound_step_n}), by induction, we get that,
\begin{align*}
\Vert x_{T-\epsilon} - \hat{X}_N\ \Vert_{L_2} 
&\leq \op\prod_{n=0}^{N-1}L_{{n}}\cp\Vert x_{0} - \hat{X}_{0}\Vert_{L_2} 
+ \sqrt{d}\frac{R^3}{\epsilon^{3/2}}h\sum_{n=1}^{N}\op\prod_{m=n}^{N-1}L_{m}\cp\int_{t_{n-1}}^{t_n} \frac{1}{(T-u)^{3/2}} du \\
&\quad+ \frac{h}{2}\sum_{n=1}^{N}\op\prod_{m=n}^{N-1}L_{m}\cp\ep_{\text{score}}(T-t_{n-1}).
\end{align*}

\sloppy As $(x_{T-\epsilon},\hat{X}_N)$ above is a specific coupling between the two variables, we have $W_2(\L(x_{T-\epsilon}),\L(\hat{X}_N)) \leq \Vert x_{T-\epsilon} - \hat{X}_N\ \Vert_{L_2}$, and with (\ref{eq:convergence_ODE_emp_score:init}), we get,
\begin{multline*}
W_2(\L(x_{T-\epsilon}),\L(\hat{X}_N)) \leq \op\prod_{n=0}^{N-1}L_{{n}}\cp W_2(\L(X_T),\L(\hat{X}_0)) \\
\quad+\sqrt{d}\frac{R^3}{\epsilon^{3/2}}h\sum_{n=1}^{N}\op\prod_{m=n}^{N-1}L_{m}\cp\int_{t_{n-1}}^{t_n} \frac{1}{(T-u)^{3/2}} du
+ \frac{h}{2}\sum_{n=1}^{N}\op\prod_{m=n}^{N-1}L_{m}\cp\ep_{\text{score}}(T-t_{n-1}).
\end{multline*}

\noindent\textbf{Bounding the propagation of error.} Lemma~\ref{lemma:bound_propagation_error:ODE} gives
$$
\op\prod_{m=n}^{N-1}L_{{m}}\cp \leq \sqrt{\frac{2\epsilon}{T-t_{n-1}}} \exp\op\frac{R^2}{2\epsilon}\cp,
$$
and
$$
\sum_{n=1}^{N}\op\prod_{m=n}^{N-1}L_{m}\cp\int_{t_{n-1}}^{t_n} \frac{1}{(T-u)^{3/2}} du \leq  \frac{1}{\sqrt{2\epsilon}}\exp\op\frac{R^2}{2\epsilon}\cp.
$$
For the error on the score, this leads to
\begin{align*}
\op\prod_{n=0}^{N-1}L_{{n}}\cp W_2(\L(X_T),\L(\hat{X}_0))
&\leq \sqrt{\frac{2\epsilon}{T+h}} \exp\op\frac{R^2}{2\epsilon}\cp W_2(\L(X_T),\L(\hat{X}_0))\\
&\leq \sqrt{\frac{2\epsilon}{T}} \exp\op\frac{R^2}{2\epsilon}\cp W_2(\L(X_T),\L(\hat{X}_0)),
\end{align*}
and for the error on the score, 
\begin{align*}
\sum_{n=1}^{N}\op\prod_{m=n}^{N-1}L_{m}\cp\ep_{\text{score}}(T-t_{n-1}) 
&\leq \sqrt{2\epsilon}\exp\op\frac{R^2}{2\epsilon}\cp\sum_{n=1}^{N}\frac{\ep_{\text{score}}(T-t_{n-1}) }{\sqrt{T-t_{n-1}}} \\
&= \sqrt{2\epsilon}\exp\op\frac{R^2}{2\epsilon}\cp\sum_{k=1}^{N}\frac{\ep_{\text{score}}(\epsilon +hk) }{\sqrt{\epsilon +hk}}.
\end{align*}
Combining the bounds above leads to,
\begin{multline*}
W_2(\L(x_{T-\epsilon}),\L(\hat{X}_N))
\leq \sqrt{\frac{2\epsilon}{T}} \exp\op\frac{R^2}{2\epsilon}\cp W_2(\L(X_T),\L(\hat{X}_0))
+ \sqrt{d}\frac{R^3}{\sqrt{2}\epsilon^{2}}\exp\op\frac{R^2}{2\epsilon}\cp h \\
+ \sqrt{\frac{\epsilon}{2}}\exp\op\frac{R^2}{2\epsilon}\cp h\sum_{k=1}^{N}\frac{\ep_{\text{score}}(\epsilon +hk) }{\sqrt{\epsilon +hk}}. 
\end{multline*}

\noindent\textbf{Early stopping error.} With the triangular inequality and Lemma~\ref{lemma:bound_early_stopping_error}, we have 
\begin{align*}
W_2(\L(X),\L(\hat{X})) = W_2(\L(x_T),\L(\hat{X}_N))
&\leq W_2(\L(x_T),\L(x_{T-\epsilon})) + W_2(\L(x_{T-\epsilon}),\L(\hat{X}_N))  \\
&= W_2(\L(X_0),\L(X_{\epsilon})) + W_2(\L(x_{T-\epsilon}),\L(\hat{X}_N)) \\
&\leq \sqrt{d\epsilon} + W_2(\L(x_{T-\epsilon}),\L(\hat{X}_N)).
\end{align*}

\noindent\textbf{Initialization error.} Corollary~\ref{cor:init_error} gives, for $T$ large enough (depending only on $\L(X)$)
$$
W_2(\L(X_T),\L(\hat{X}_0))\leq \frac{R^2}{\sqrt{T}},
$$
hence finally,
\begin{align*}
W_2(\L&(X),\L(\hat{X})) \leq \sqrt{d\epsilon}+\frac{\sqrt{2\epsilon}}{T} \exp\op\frac{R^2}{2\epsilon}\cp R^2\\
&\quad+ \sqrt{d}\frac{R^3}{\sqrt{2}\epsilon^{2}}\exp\op\frac{R^2}{2\epsilon}\cp h 
+ \sqrt{\frac{\epsilon}{2}}\exp\op\frac{R^2}{2\epsilon}\cp h\sum_{k=1}^{N}\frac{\ep_{\text{score}}(\epsilon +hk)) }{\sqrt{\epsilon +hk}}. \tag{\ref{eq:bound_ODE_emp}}
\end{align*}

\subsection{Proof of Proposition \ref{prop:convergence_Heun_emp_score}}

To bound the Wasserstein distance between $X$ and $\hat{X} = \hat{X}_N$ the output of the algorithm, we construct a specific coupling between the two variables. Here, the sampler is deterministic, so we only choose the coupling between $x_0 = X_T$ and $\hat{X}_0 \sim \N(0,TI)$, such that,
\begin{equation}
\label{eq:convergence_Heun_emp_score:init}
\Vert x_0 - \hat{X}_0\Vert_{L_2} =\Vert X_T - \hat{X}_0\Vert_{L_2} = W_2(\L(X_T),\L(\hat{X}_0)),
\end{equation}
and follow the evolution of the ODE and its discretization to get a coupling between  $X$ and $\hat{X}$.

\noindent\textbf{Bounding the error at step $n$.} We first look at the error at each discretization step. For $0 < n \leq N$, we have,
$$
x_{t_{n}}= x_{t_{n-1}}+ \frac{1}{2}\int_{t_{n-1}}^{t_{n}} \nabla\log \rv p_t(x_t)dt,
$$
and,
\begin{align*}
\hat{Y}_n &= \hat{X}_{n-1} + \frac{h}{2} s_\theta(T-t_{n-1},\hat{X}_{n-1}),\\
\hat{X}_n &= \hat{X}_{n-1} + \frac{h}{4} \op s_\theta(T-t_{n-1},\hat{X}_{n-1}) + s_\theta(T-t_{n},\hat{Y}_{n})\cp.
\end{align*}
We also introduce
$$
\td Y_n = x_{t_{n-1}} + \frac{h}{2}s_\theta\op T - t_{n-1};\hat X_{n-1}\cp,
$$
such that 
$$
\td Y_n - \hat Y_n = x_{t_{n-1}} - \hat X_{n-1}.
$$
Then we have,
\begin{align*}
x&_{t_{n}} - \hat{X}_n \\&= x_{t_{n-1}} - \hat{X}_{n-1} + \frac{1}{2}\int_{t_{n-1}}^{t_{n}} \nabla\log \rv p_t(x_t)dt  - \frac{h}{4} \op s_\theta(T-t_{n-1},\hat{X}_{n-1}) + s_\theta(T-t_{n},\hat{Y}_{n})\cp\\
&= x_{t_{n-1}} - \hat{X}_{n-1} \\
&\quad + \underbrace{\frac{1}{2}\op \int_{t_{n-1}}^{t_{n}} \nabla\log \rv p_t(x_t)dt - \frac{h}{2}\op\nabla\log \rv p_{t_{n-1}}(x_{t_{n-1}})+ \nabla\log \rv p_{t_{n}}(x_{t_{n}})\cp\cp}_{D_n} \\
&\quad + \underbrace{\frac{h}{4}\op\nabla\log \rv p_{t_{n-1}}(x_{t_{n-1}}) - s_\theta(T-t_{n-1},x_{t_{n-1}})\cp + \frac{h}{4}\op\nabla\log \rv p_{t_{n}}(x_{t_{n}}) - s_\theta(T-t_{n},x_{t_{n}})\cp}_{S_n}\\
&\quad + \underbrace{\frac{h}{4}\op s_\theta(T-t_{n},x_{t_{n}}) - s_\theta(T-t_{n-1},\td Y_n)\cp}_{R_n}\\
&\quad + \frac{h}{4}\op s_\theta(T-t_{n},\td Y_n) - s_\theta(T-t_{n},\hat Y_n)\cp + \frac{h}{4}\op s_\theta(T-t_{n-1},x_{t_{n-1}}) - s_\theta(T-t_{n-1},\hat{X}_{n-1})\cp\\
&= \frac{1}{2}\op \op I +\frac{h}{2}s_\theta(T-t_{n},\cdot)\cp(\td Y_n) - \op I +\frac{h}{2}s_\theta(T-t_{n},\cdot)\cp(\hat Y_n)\cp\\
&\quad +\frac{1}{2}\op \op I +\frac{h}{2}s_\theta(T-t_{n-1},\cdot)\cp(x_{t_{n-1}}) - \op I +\frac{h}{2}s_\theta(T-t_{n-1},\cdot)\cp(\hat{X}_{n-1})\cp\\
&\quad + D_n + S_n + R_n.
\end{align*}
As $x_{t_{n-1}} \sim X_{T-t_{n-1}}$, we have,
\begin{multline*}
\Vert\nabla\log \rv p_{t_{n-1}}(x_{t_{n-1}}) - s_\theta(T-t_{n-1},x_{t_{n-1}})\Vert_{L_2}\\
= \Vert \nabla \log p_{T-t_{n-1}}(X_{T-t_{n-1}}) - s_\theta(T-t_{n-1},X_{T-t_{n-1}})\Vert_{L_2} 
= \ep_{\text{score}}(T-t_{n-1}),
\end{multline*}
and as $x_{t_{n}} \sim X_{T-t_{n}}$, 
\begin{multline*}
\Vert\nabla\log \rv p_{t_{n}}(x_{t_{n}}) - s_\theta(T-t_{n},x_{t_{n}})\Vert_{L_2}\\
= \Vert \nabla \log p_{T-t_{n}}(X_{T-t_{n}}) - s_\theta(T-t_{n},X_{T-t_{n}})\Vert_{L_2} 
= \ep_{\text{score}}(T-t_{n}),
\end{multline*}
hence
$$
\Vert S_n\Vert_{L_2} \leq \frac{h}{4}\op\ep_{\text{score}}(T-t_{n})+\ep_{\text{score}}(T-t_{n-1})\cp.
$$

\noindent With Assumption 2 on $x \mapsto x + h s_\theta(t,x)$ , we get that 
$$
\left\Vert\op I +\frac{h}{2}\nabla\log \rv p_{t_{n-1}}\cp(x_{t_{n-1}}) - \op I +\frac{h}{2}\nabla\log \rv p_{t_{n-1}}\cp(\hat{X}_{n-1})\right\Vert_{L_2} \leq  L_{n-1}\Vert x_{t_{n-1}} - \hat{X}_{n-1}\Vert_{L_2},
$$
and
\begin{align*}
\left\Vert\op I +\frac{h}{2}\nabla\log \rv p_{t_{n}}\cp(\td Y_n) - \op I +\frac{h}{2}\nabla\log \rv p_{t_{n}}\cp(\hat Y_{n})\right\Vert_{L_2}
&\leq  L_{n}\Vert \td Y_n -\hat Y_n\Vert_{L_2} \\
&= L_{n}\Vert x_{t_{n-1}} - \hat{X}_{n-1}\Vert_{L_2},
\end{align*}
with $L_{n-1} = L_{T-t_{n-1},h/2} = 1 + \frac{h}{2}\op\frac{R^2}{(T-t_{n-1})^2} - \frac{1}{T-t_{n-1}}\cp$ and $L_{n} = L_{T-t_{n},h/2} = 1 + \frac{h}{2}\op\frac{R^2}{(T-t_{n})^2} - \frac{1}{T-t_{n}}\cp.$
Finally Lemma~\ref{lemma:bound_discretization_Heun} gives 
$$
\Vert D_n \Vert_{L_2} \leq \frac{33d}{2}\frac{R^5}{\epsilon^{5/2}}h^2\int_{t_{n-1}}^{t_n} \frac{1}{(T-u)^{5/2}} du \leq \frac{33d}{2}\frac{R^5}{\epsilon^{5/2}}h^3 \frac{1}{(T-t_n)^{5/2}},
$$
leading to,
\begin{multline}
\label{eq:convergence_Heun_emp_score:bound_step_n_without_R_n}
\Vert x_{t_{n}} - \hat{X}_n\Vert_{L_2} \leq \frac{L_{n-1}+L_n}{2}\Vert x_{t_{n-1}} - \hat{X}_{n-1}\Vert_{L_2}  + \frac{33d}{2}\frac{R^5}{\epsilon^{5/2}}h^3 \frac{1}{(T-t_n)^{5/2}}\\
+ \frac{h}{4}\op\ep_{\text{score}}(T-t_{n})+\ep_{\text{score}}(T-t_{n-1})\cp + \Vert R_n\Vert_{L_2}.
\end{multline}

\noindent\textbf{Bounding $\Vert R_n\Vert_{L_2}$.} Under Assumption 2, $x\mapsto s_\theta(T-t_n,x)$ is $C_{T-t_n}$-Lipchitz with
$$
C_{T-t_n} = \max \op \frac{1}{T-t_n}, \left|\frac{R^2}{(T-t_n)^2}- \frac{1}{T-t_n}\right|\cp 
\leq \frac{R^2}{\epsilon(T-t_n)},
$$
hence
\begin{align*}
\Vert R_n\Vert_{L_2} &\leq \frac{hR^2}{4\epsilon(T-t_n)} \Vert x_{t_n} - \td Y_n\Vert_{L_2}\\
&= \frac{hR^2}{4\epsilon(T-t_n)} \left\Vert \frac{1}{2}\op \int_{t_{n-1}}^{t_{n}} \nabla\log \rv p_t(x_t)dt- hs_\theta\op T - t_{n-1};\hat X_{n-1}\cp\cp\right\Vert_{L_2}\\
&\leq \frac{hR^2}{8\epsilon(T-t_n)} \left\Vert \int_{t_{n-1}}^{t_{n}} \nabla\log \rv p_t(x_t)dt- h  \nabla\log \rv p_{t_{n-1}}(x_{t_{n-1}})\right\Vert_{L_2}\\
&\quad+ \frac{h^2R^2}{8\epsilon(T-t_n)} \left\Vert  \nabla\log \rv p_{t_{n-1}}(x_{t_{n-1}})- s_\theta\op T - t_{n-1};x_{t_{n-1}}\cp\right\Vert_{L_2}\\
&\quad+ \frac{h^2R^2}{8\epsilon(T-t_n)} \left\Vert s_\theta\op T - t_{n-1};x_{t_{n-1}}\cp - s_\theta\op T - t_{n-1};\hat X_{n-1}\cp\right\Vert_{L_2}\\
&\leq \frac{hR^2}{8\epsilon(T-t_n)} \left\Vert \int_{t_{n-1}}^{t_{n}} \nabla\log \rv p_t(x_t)dt- h  \nabla\log \rv p_{t_{n-1}}(x_{t_{n-1}})\right\Vert_{L_2}\\
&\quad+ \frac{h^2R^2}{8\epsilon(T-t_n)} \ep_\text{score}(T-t_{n-1}) + \frac{h^2R^4}{8\epsilon^2(T-t_n)(T-t_{n-1})} \Vert x_{t_{n-1}} - \hat X_{n-1}\Vert_{L_2}.
\end{align*}
With Lemma~\ref{lemma:bound_discretization_ODE}, we have
\begin{align*}
\left\Vert \int_{t_{n-1}}^{t_{n}} \nabla\log \rv p_t(x_t)dt- h  \nabla\log \rv p_{t_{n-1}}(x_{t_{n-1}})\right\Vert_{L_2} 
&\leq 2\sqrt{d}\frac{R^3}{\epsilon^{3/2}}h \int_{t_{n-1}}^{t_n} \frac{1}{(T-u)^{3/2}} du\\
&\leq 2\sqrt{d}\frac{R^3}{\epsilon^{3/2}(T-t_n)^{3/2}}h^2.
\end{align*}
hence
\begin{align*}
\Vert R_n\Vert_{L_2} 
&\leq \frac{\sqrt{d}R^5}{4\epsilon^{5/2}(T-t_n)^{5/2}}h^3+ \frac{h^2R^2}{8\epsilon(T-t_n)} \ep_\text{score}(T-t_{n-1})\\
&\quad + \frac{h^2R^4}{8\epsilon^2(T-t_n)(T-t_{n-1})} \Vert x_{t_{n-1}} - \hat X_{n-1}\Vert_{L_2}\\
&\leq \frac{\sqrt{d}R^5}{4\epsilon^{5/2}(T-t_n)^{5/2}}h^3+ \frac{h^2R^2}{8\epsilon^2} \ep_\text{score}(T-t_{n-1}) + \frac{h^2R^4}{8\epsilon^2(T-t_n)^2} \Vert x_{t_{n-1}} - \hat X_{n-1}\Vert_{L_2}.
\end{align*}
This leads to
\begin{multline}
\label{eq:convergence_Heun_emp_score:bound_step_n}
\Vert x_{t_{n}} - \hat{X}_n\Vert_{L_2} 
\leq 
\op \frac{L_{n-1}+L_n}{2} + \frac{h^2R^4}{8\epsilon^2(T-t_n)^2}\cp\Vert x_{t_{n-1}} - \hat{X}_{n-1}\Vert_{L_2}  
+ \frac{17dR^5}{\epsilon^{5/2}(T-t_n)^{5/2}}h^3 \\
+ \frac{h}{4}\ep_{\text{score}}(T-t_{n})
+ \frac{h}{4}\op 1 + \frac{hR^2}{2\epsilon^2}\cp \ep_\text{score}(T-t_{n-1}).
\end{multline}

\noindent\textbf{Bounding the error at time $T-\epsilon.$} With (\ref{eq:convergence_Heun_emp_score:bound_step_n}), by induction, we get that,
\begin{multline*}
\Vert x_{T-\epsilon} - \hat{X}_N\ \Vert_{L_2} 
\leq \op\prod_{n=0}^{N-1}\frac{L_{n}+L_{n+1}}{2} + \frac{h^2R^4}{8\epsilon^2(T-t_{n+1})^2}\cp\Vert x_{0} - \hat{X}_{0}\Vert_{L_2} \\
+ \frac{17dR^5}{\epsilon^{5/2}}h^3\sum_{n=1}^{N}\op\prod_{m=n}^{N-1}\frac{L_{m}+L_{m+1}}{2} + \frac{h^2R^4}{8\epsilon^2(T-t_{m+1})^2}\cp\frac{1}{(T-t_n)^{5/2}} \\
+ \frac{h}{4}\sum_{n=1}^{N}\op\prod_{m=n}^{N-1}\frac{L_{m}+L_{m+1}}{2} + \frac{h^2R^4}{8\epsilon^2(T-t_{m+1})^2}\cp\ep_{\text{score}}(T-t_{n})\\
+ \frac{h}{4}\op 1 + \frac{hR^2}{2\epsilon^2}\cp\sum_{n=1}^{N}\op\prod_{m=n}^{N-1}\frac{L_{m}+L_{m+1}}{2} + \frac{h^2R^4}{8\epsilon^2(T-t_{m+1})^2}\cp\ep_{\text{score}}(T-t_{n-1}).    
\end{multline*}

Denote $K_n = \op\prod_{m=n}^{N-1}\frac{L_{m}+L_{m+1}}{2} + \frac{h^2R^4}{8\epsilon^2(T-t_{m+1})^2}\cp$. As $(x_{T-\epsilon},\hat{X}_N)$ above is a specific coupling between the two variables, we have $W_2(\L(x_{T-\epsilon}),\L(\hat{X}_N)) \leq \Vert x_{T-\epsilon} - \hat{X}_N\ \Vert_{L_2}$, and with (\ref{eq:convergence_Heun_emp_score:init}), we get,
\begin{multline*}
W_2(\L(x_{T-\epsilon}),\L(\hat{X}_N))
\leq K_0 W_2(\L(X_T),\L(\hat{X}_0))
+ \frac{17dR^5}{\epsilon^{5/2}}h^3\sum_{n=1}^{N}\frac{K_n}{(T-t_n)^{5/2}} \\
+ \frac{h}{4}\sum_{n=1}^{N}K_n\ep_{\text{score}}(T-t_{n})
+ \frac{h}{4}\op 1 + \frac{hR^2}{2\epsilon^2}\cp\sum_{n=1}^{N}K_n\ep_{\text{score}}(T-t_{n-1}).    
\end{multline*}

\noindent\textbf{Bounding the propagation of error.} 
With Lemma~\ref{lemma:bound_propagation_error:Heun}, we get that
$$
K_n \leq  \sqrt{\frac{2\epsilon}{T-t_{n-1}}} \exp\op\frac{R^2}{\epsilon} +\frac{hR^2}{4\epsilon^2}\cp,
$$
and
$$
h\sum_{n=1}^{N}\frac{K_n}{(T-t_n)^{5/2}}  \leq \frac{1}{\sqrt{2}\epsilon^{3/2}} \exp\op\frac{R^2}{\epsilon} +\frac{hR^2}{4\epsilon^2}\cp.
$$
For the initialization error, this leads to
\begin{align*}
K_0 W_2(\L(X_T),\L(\hat{X}_0)) &\leq \sqrt{\frac{2\epsilon}{T+h}} \exp\op\frac{R^2}{\epsilon} +\frac{hR^2}{4\epsilon^2}\cp W_2(\L(X_T),\L(\hat{X}_0))\\
&\leq \sqrt{\frac{2\epsilon}{T}} \exp\op\frac{R^2}{\epsilon} +\frac{hR^2}{4\epsilon^2}\cp W_2(\L(X_T),\L(\hat{X}_0)),
\end{align*}
and for the error on the score,
\begin{align*}
\sum_{n=1}^{N}K_n\ep_{\text{score}}(T-t_{n-1}) 
&\leq \sqrt{2\epsilon}\exp\op\frac{R^2}{\epsilon} +\frac{hR^2}{4\epsilon^2}\cp \sum_{n=1}^{N}\frac{\ep_{\text{score}}(T-t_{n-1}) }{\sqrt{T-t_{n-1}}} \\
&= \sqrt{2\epsilon}\exp\op\frac{R^2}{\epsilon} +\frac{hR^2}{4\epsilon^2}\cp\sum_{k=1}^{N}\frac{\ep_{\text{score}}(\epsilon +hk) }{\sqrt{\epsilon +hk}},
\end{align*}
and similarly,
$$
\sum_{n=1}^{N}K_n\ep_{\text{score}}(T-t_{n}) \leq \sqrt{2\epsilon}\exp\op\frac{R^2}{\epsilon} +\frac{hR^2}{4\epsilon^2}\cp \sum_{k=0}^{N-1}\frac{\ep_{\text{score}}(\epsilon +hk) }{\sqrt{\epsilon +h(k+1)}}.
$$
Combining the bounds above leads to,
\begin{multline*}
W_2(\L(x_{T-\epsilon}),\L(\hat{X}_N))
\leq  \sqrt{\frac{2\epsilon}{T}} \exp\op\frac{R^2}{\epsilon} +\frac{hR^2}{4\epsilon^2}\cp  W_2(\L(X_T),\L(\hat{X}_0)) \\
+ \frac{17dR^5}{\sqrt{2}\epsilon^{4}}\exp\op\frac{R^2}{\epsilon} +\frac{hR^2}{4\epsilon^2}\cp h^2 \\
+ \frac{\sqrt{\epsilon}}{2\sqrt{2}}\exp\op\frac{R^2}{\epsilon} +\frac{hR^2}{4\epsilon^2}\cp  h\sum_{k=0}^{N-1}\frac{\ep_{\text{score}}(\epsilon +hk) }{\sqrt{\epsilon +h(k+1)}}\\
+ \frac{\sqrt{\epsilon}}{2\sqrt{2}}\op 1 + \frac{hR^2}{2\epsilon^2}\cp\exp\op\frac{R^2}{\epsilon} +\frac{hR^2}{4\epsilon^2}\cp  h\sum_{k=1}^{N}\frac{\ep_{\text{score}}(\epsilon +hk) }{\sqrt{\epsilon +hk}}.    
\end{multline*}

\noindent\textbf{Early stopping error.} With the triangular inequality and Lemma~\ref{lemma:bound_early_stopping_error}, we have 
\begin{align*}
W_2(\L(X),\L(\hat{X})) = W_2(\L(x_T),\L(\hat{X}_N))
&\leq W_2(\L(x_T),\L(x_{T-\epsilon})) + W_2(\L(x_{T-\epsilon}),\L(\hat{X}_N))  \\
&= W_2(\L(X_0),\L(X_{\epsilon})) + W_2(\L(x_{T-\epsilon}),\L(\hat{X}_N)) \\
&\leq \sqrt{d\epsilon} + W_2(\L(x_{T-\epsilon}),\L(\hat{X}_N)).
\end{align*}

\noindent\textbf{Initialization error.} Corollary~\ref{cor:init_error} gives, for $T$ large enough (depending only on $\L(X)$)
$$
W_2(\L(X_T),\L(\hat{X}_0))\leq \frac{R^2}{\sqrt{T}},
$$
hence finally,
\begin{multline*}
\tag{\ref{eq:bound_Heun_emp}}
W_2(\L(X),\L(\hat{X})) \leq \sqrt{d\epsilon}
+  \frac{\sqrt{2\epsilon}}{T} \exp\op\frac{R^2}{\epsilon} +\frac{hR^2}{4\epsilon^2}\cp  R^2
+ \frac{17dR^5}{\sqrt{2}\epsilon^{4}}\exp\op\frac{R^2}{\epsilon} +\frac{hR^2}{4\epsilon^2}\cp h^2 \\
+ \frac{\sqrt{\epsilon}}{2\sqrt{2}}\exp\op\frac{R^2}{\epsilon} +\frac{hR^2}{4\epsilon^2}\cp \op h\sum_{k=0}^{N-1}\frac{\ep_{\text{score}}(\epsilon +hk) }{\sqrt{\epsilon +h(k+1)}}
+  \op 1 + \frac{hR^2}{2\epsilon^2}\cp h\sum_{k=1}^{N}\frac{\ep_{\text{score}}(\epsilon +hk) }{\sqrt{\epsilon +hk}}\cp.  
\end{multline*}

\subsection{Proof of Proposition \ref{prop:convergence_SDE_emp_score}}

To bound the Wasserstein distance between $X$ and $\hat{X} = \hat{X}_N$ the output of the algorithm, we construct a specific coupling between the two variables.
Contrary to the deterministic samplers, the Euler-Maruyama sampler adds noise at each step, so we need to choose a specific representation of this noise to get a coupling between $(\hat{X}_n)_n$ and $(\rv X_{t})_t$.
We start by choosing the coupling between $\rv X_0 = X_T$ and $\hat{X}_0 \sim \N(0,TI)$, such that,
\begin{equation}
\label{eq:convergence_SDE_emp_score:init}
\Vert X_0 - \hat{X}_0\Vert_{L_2} =\Vert X_T - \hat{X}_0\Vert_{L_2} = W_2(\L(X_T),\L(\hat{X}_0)).
\end{equation}

\noindent\textbf{Bounding the error at step $n$.} We first look at the error at each discretization step. For $0 < n \leq N$, we have,
$$
\rv X_{t_{n}}= \rv X_{t_{n-1}}+ \int_{t_{n-1}}^{t_{n}} \nabla\log \rv p_t(\rv X_t)dt + \int_{t_{n-1}}^{t_{n}} d W_t,
$$
and,
$$
\hat{X}_n = \hat{X}_{n-1} + h s_\theta(T-t_{n-1},\hat{X}_{n-1}) + Z_n,
$$
with $Z_n\sim \N(0, hI)$ independent from $\hat{X}_{n-1}$.
To define the coupling between $(\rv X_{t_n})_,$ and $(\hat{X}_n)_n$, we take 
$$Z_n=\int_{t_{n-1}}^{t_{n}} d W_t,$$  such that the Gaussian noises will cancel out when computing the difference $x_{t_{n}} - \hat{X}_n$:
\begin{align*}
\rv X_{t_{n}} - \hat{X}_n 
&= \rv X_{t_{n-1}} - \hat{X}_{n-1} + \int_{t_{n-1}}^{t_{n}} \nabla\log \rv p_t(\rv X_t)dt  - hs_\theta(T-t_{n-1},\hat{X}_{n-1})\\
&= \rv X_{t_{n-1}} - \hat{X}_{n-1} \\
&\quad + \int_{t_{n-1}}^{t_{n}} \nabla\log \rv p_t(\rv X_t)dt - h\nabla\log \rv p_{t_{n-1}}(\rv X_{t_{n-1}}) \\
&\quad + h\op\nabla\log \rv p_{t_{n-1}}(\rv X_{t_{n-1}}) - s_\theta(T-t_{n-1},\rv X_{t_{n-1}})\cp\\
&\quad + h\op s_\theta(T-t_{n-1},\rv X_{t_{n-1}}) - s_\theta(T-t_{n-1},\hat{X}_{n-1})\cp\\
&= \op \op I +hs_\theta(T-t_{n-1},\cdot)\cp(\rv X_{t_{n-1}}) - \op I + hs_\theta(T-t_{n-1},\cdot)\cp(\hat{X}_{n-1})\cp\\
&\quad + \int_{t_{n-1}}^{t_{n}} \nabla\log \rv p_t(\rv X_t) - \nabla\log \rv p_{t_{n-1}}(\rv X_{t_{n-1}})dt\\
&\quad + h\op\nabla\log \rv p_{t_{n-1}}(\rv X_{t_{n-1}}) - s_\theta(T-t_{n-1},\rv X_{t_{n-1}})\cp.
\end{align*}
As $\rv X_{t_{n-1}} \sim X_{T-t_{n-1}}$, we have,
\begin{multline*}
\Vert\nabla\log \rv p_{t_{n-1}}(\rv X_{t_{n-1}}) - s_\theta(T-t_{n-1},\rv X_{t_{n-1}})\Vert_{L_2}\\
= \Vert \nabla \log p_{T-t_{n-1}}(X_{T-t_{n-1}}) - s_\theta(T-t_{n-1},X_{T-t_{n-1}})\Vert_{L_2} 
= \ep_{\text{score}}(T-t_{n-1}), 
\end{multline*}
and with Assumption 2 on $x \mapsto x + h s_\theta(t,x)$ , we get that
$$
\left\Vert\op I +h\nabla\log\rv p_{t_{n-1}}\cp(\rv X_{t_{n-1}}) - \op I + h\nabla\log \rv p_{t_{n-1}}\cp(\hat{X}_{n-1})\right\Vert_{L_2} \leq  L_{n-1}\Vert \rv X_{t_{n-1}} - \hat{X}_{n-1}\Vert_{L_2},
$$
with $L_{n-1} = L_{T-t_{n-1},h} = 1 + h\op\frac{R^2}{(T-t_{n-1})^2} - \frac{1}{T-t_{n-1}}\cp.$ Finally Lemma~\ref{lemma:bound_discretization_SDE} gives 
$$
 \left\Vert \int_{t_{n-1}}^{t_{n}} \nabla\log \rv p_t(\rv X_t) - \nabla\log \rv p_{t_{n-1}}(\rv X_{t_{n-1}})dt \right\Vert_{L_2} \leq  \frac{\sqrt{2d}}{2}\frac{R^2}{\epsilon} \sqrt{h} \int_{t_{n-1}}^{t_n}\frac{1}{T-s}ds,
$$
leading to,
\begin{equation}
\label{eq:convergence_SDE_emp_score:bound_step_n}
\Vert \rv X_{t_{n}} - \hat{X}_n\Vert_{L_2}
\leq L_{n-1}\Vert \rv X_{t_{n-1}} - \hat{X}_{n-1}\Vert_{L_2}
+ \frac{\sqrt{2d}}{2}\frac{R^2}{\epsilon} \sqrt{h} \int_{t_{n-1}}^{t_n}\frac{1}{T-s}ds
+ h\ep_{\text{score}}(T-t_{n-1}).    
\end{equation}

\noindent\textbf{Bounding the error at time $T-\epsilon.$} With (\ref{eq:convergence_SDE_emp_score:bound_step_n}), by induction, we get that,
\begin{multline*}
\Vert \rv X_{T-\epsilon} - \hat{X}_N\ \Vert_{L_2} 
\leq \op\prod_{n=0}^{N-1}L_{{n}}\cp\Vert \rv X_{0} - \hat{X}_{0}\Vert_{L_2}\\
+ \frac{\sqrt{2d}}{2}\frac{R^2}{\epsilon} \sqrt{h} \sum_{n=1}^{N}\op\prod_{m=n}^{N-1}L_{m}\cp\int_{t_{n-1}}^{t_n}\frac{1}{T-s}ds
+ h\sum_{n=1}^{N}\op\prod_{m=n}^{N-1}L_{m}\cp\ep_{\text{score}}(T-t_{n-1}).
\end{multline*}
As $(\rv X_{T-\epsilon},\hat{X}_N)$ above is a specific coupling between the two variables, we have $W_2(\L(\rv X_{T-\epsilon}),\L(\hat{X}_N)) \leq \Vert \rv X_{T-\epsilon} - \hat{X}_N\ \Vert_{L_2}$, and with (\ref{eq:convergence_SDE_emp_score:init}), we get,
\begin{multline*}
W_2(\L(\rv X_{T-\epsilon}),\L(\hat{X}_N)) \leq \op\prod_{n=0}^{N-1}L_{{n}}\cp W_2(\L(X_T),\L(\hat{X}_0))\\
+  \frac{\sqrt{2d}}{2}\frac{R^2}{\epsilon} \sqrt{h} \sum_{n=1}^{N}\op\prod_{m=n}^{N-1}L_{m}\cp\int_{t_{n-1}}^{t_n}\frac{1}{T-s}ds
+ h\sum_{n=1}^{N}\op\prod_{m=n}^{N-1}L_{m}\cp\ep_{\text{score}}(T-t_{n-1}).
\end{multline*}

\noindent\textbf{Bounding the propagation of error.} Lemma~\ref{lemma:bound_propagation_error:SDE} gives
$$
\prod_{m=n}^{N-1}L_{m} \leq  \frac{2\epsilon}{T-t_{n-1}} \exp\op\frac{R^2}{\epsilon}\cp,
$$
and
$$
\sum_{n=1}^{N}\op\prod_{m=n}^{N-1}L_{m}\cp  \int_{t_{n-1}}^{t_n}\frac{1}{T-s}ds \leq 2\exp\op\frac{R^2}{\epsilon}\cp.
$$
For the initialization error, this leads to
\begin{align*}
\op\prod_{n=0}^{N-1}L_{{n}}\cp W_2(\L(X_T),\L(\hat{X}_0)) 
&\leq  \frac{2\epsilon}{T+h} \exp\op\frac{R^2}{\epsilon}\cp W_2(\L(X_T),\L(\hat{X}_0))\\
&\leq \frac{2\epsilon}{T} \exp\op\frac{R^2}{\epsilon}\cp W_2(\L(X_T),\L(\hat{X}_0)),
\end{align*}
and for the error on the score,
\begin{align*}
\label{eq:convergence_SDE_emp_score:bound_propagated_score_error}
\sum_{n=1}^{N}\op\prod_{m=n}^{N-1}L_{m}\cp\ep_{\text{score}}(T-t_{n-1}) 
&\leq 2\epsilon\exp\op\frac{R^2}{\epsilon}\cp\sum_{n=1}^{N}\frac{\ep_{\text{score}}(T-t_{n-1}) }{{T-t_{n-1}}} \\
&= {2\epsilon}\exp\op\frac{R^2}{\epsilon}\cp\sum_{k=1}^{N}\frac{\ep_{\text{score}}(\epsilon +hk) }{{\epsilon +hk}}.
\end{align*}
Combining the bounds above leads to,
\begin{align*}
W_2(\L(\rv X_{T-\epsilon}),\L(\hat{X}_N))
&\leq \frac{2\epsilon}{T} \exp\op\frac{R^2}{\epsilon}\cp W_2(\L(X_T),\L(\hat{X}_0))
+ \sqrt{2d}\frac{R^2}{\epsilon}\exp\op\frac{R^2}{\epsilon}\cp \sqrt{h}\\
&\quad + {2\epsilon}\exp\op\frac{R^2}{\epsilon}\cp\sum_{k=1}^{N}\frac{\ep_{\text{score}}(\epsilon +hk) }{{\epsilon +hk}}. 
\end{align*}

\noindent\textbf{Early stopping error.} With the triangular inequality and Lemma~\ref{lemma:bound_early_stopping_error}, we have 
\begin{align*}
W_2(\L(X),\L(\hat{X})) = W_2(\L(\rv X_T),\L(\hat{X}_N))
&\leq W_2(\L(\rv X_T),\L(\rv X_{T-\epsilon})) + W_2(\L(\rv X_{T-\epsilon}),\L(\hat{X}_N))  \\
&= W_2(\L(X_0),\L(X_{\epsilon})) + W_2(\L(\rv X_{T-\epsilon}),\L(\hat{X}_N)) \\
&\leq \sqrt{d\epsilon} + W_2(\L(\rv X_{T-\epsilon}),\L(\hat{X}_N)).
\end{align*}

\noindent\textbf{Initialization error.} Corollary~\ref{cor:init_error} gives, for $T$ large enough (depending only on $\L(X)$)
$$
W_2(\L(X_T),\L(\hat{X}_0))\leq \frac{R^2}{\sqrt{T}},
$$
hence finally,
hence finally,
\begin{align*}
W_2&(\L(X),\L(\hat{X})) 
\leq \sqrt{d\epsilon}+\frac{2\epsilon}{T^{3/2}} \exp\op\frac{R^2}{\epsilon}\cp R^2\\
&\quad + \sqrt{2d}\frac{R^2}{\epsilon}\exp\op\frac{R^2}{\epsilon}\cp \sqrt{h}
+ {2\epsilon}\exp\op\frac{R^2}{\epsilon}\cp h\sum_{k=1}^{N}\frac{\ep_{\text{score}}(\epsilon +hk) }{{\epsilon +hk}}. 
\tag{\ref{eq:bound_SDE_emp}}
\end{align*}

\subsection{Proof of Proposition \ref{prop:convergence_SDE_true_score}}

To bound the Wasserstein distance between $X$ and $\hat{X} = \hat{X}_N$ the output of the algorithm, we construct a specific coupling between the two variables.
Contrary to the deterministic samplers, the Euler-Maruyama sampler add noise at each step, so we need to choose a specific representation of this noise to get a coupling between $(\hat{X}_n)_n$ and $(\rv X_{t})_t$. 
Here, to get an order 1 convergence rate in the step size, the crucial part is that the discretization error at each step is uncorrelated from the previous steps.
Therefore, we will only need to by more precise and define the filtration $\F_t$ to control the correlation between the errors at each step.
We start by choosing the coupling between $\rv X_0 = X_T$ and $\hat{X}_0 \sim \N(0,TI)$, such that,
\begin{equation}
\label{eq:convergence_SDE_true_score:init}
\Vert X_0 - \hat{X}_0\Vert_{L_2} =\Vert X_T - \hat{X}_0\Vert_{L_2} = W_2(\L(X_T),\L(\hat{X}_0)).
\end{equation}
We take $(W_t)_{t\geq 0}$ a Brownian motion independent from $(X_0,\hat X_0)$ and we define the filtration $\F_t = \sigma(X_0,\hat X_0, (W_s)_{s\leq t})$, such that for all $t\geq 0$, $\rv X_T$ is $\F_t$-measurable and $(W_t)_{t}$ is a $\F_t$-Brownian motion.

\noindent\textbf{Bounding the error at step $n$.} For $0 < n \leq N$, we have,
$$ 
\rv X_{t_{n}}= \rv X_{t_{n-1}}+ \int_{t_{n-1}}^{t_{n}} \nabla\log \rv p_t(\rv X_t)dt + \int_{t_{n-1}}^{t_{n}} d W_t,
$$
and,
$$
\hat{X}_n = \hat{X}_{n-1} +h\nabla\log \rv p_t(\hat{X}_{n-1}) + Z_n,
$$
with $Z_n\sim \N(0, hI)$ independent from $\hat{X}_{n-1}$.
To define the coupling between $(\rv X_{t_n})_,$ and $(\hat{X}_n)_n$, we take 
$$Z_n=\int_{t_{n-1}}^{t_{n}} d W_t,$$
such that we get recursively that $\hat X_n$ is $\F_{t_n}$-measurable.
Moreover, with this choice, the Gaussian noises will cancel out when computing the difference $\rv X_{t_{n}} - \hat{X}_n$:
\begin{align*}
\rv X_{t_{n}} - \hat{X}_n &= \rv X_{t_{n-1}} - \hat{X}_{n-1} + \int_{t_{n-1}}^{t_{n}} \nabla\log \rv p_t(\rv X_t)dt  -h\nabla\log \rv p_t(\hat{X}_{n-1})\\
&= \rv X_{t_{n-1}} - \hat{X}_{n-1} \\
&\quad + \int_{t_{n-1}}^{t_{n}} \nabla\log \rv p_t(\rv X_t)dt - h\nabla\log \rv p_{t_{n-1}}(\rv X_{t_{n-1}}) \\
&\quad + h\op\nabla\log \rv p_{t_{n-1}}(\rv X_{t_{n-1}}) - \nabla\log \rv p_{t_{n-1}}(\hat{X}_{n-1})\cp\\
&= \op \op I +h\nabla\log\rv p_{t_{n-1}}\cp(\rv X_{t_{n-1}}) - \op I + h\nabla\log \rv p_{t_{n-1}}\cp(\hat{X}_{n-1})\cp\\
&\quad + \int_{t_{n-1}}^{t_{n}} \nabla\log \rv p_t(\rv X_t) - \nabla\log \rv p_{t_{n-1}}(\rv X_{t_{n-1}})dt.
\end{align*}
Lemma~\ref{lemma:bound_spatial_regularity} gives
$$
\left\Vert\op I +h\nabla\log\rv p_{t_{n-1}}\cp(\rv X_{t_{n-1}}) - \op I + h\nabla\log \rv p_{t_{n-1}}\cp(\hat{X}_{n-1})\right\Vert_{L_2} \leq  L_{n-1}\Vert \rv X_{t_{n-1}} - \hat{X}_{n-1}\Vert_{L_2},
$$
with $L_{n-1} = L_{T-t_{n-1},h} = 1 + h\op\frac{R^2}{(T-t_{n-1})^2} - \frac{1}{T-t_{n-1}}\cp.$
Moreover, Lemma~\ref{lemma:bound_discretization_SDE} gives
$$
\int_{t_{n-1}}^{t_{n}} \nabla\log \rv p_t(\rv X_t) - \nabla\log \rv p_{t_{n-1}}(\rv X_{t_{n-1}})dt = \int_{t_{n-1}}^{t_n}\int_{t_{n-1}}^s \nabla^2\log \rv{p}_u(\rv X_u)\cdot dW_u ds,
$$
and
$$
\left\Vert \int_{t_{n-1}}^{t_{n}} \nabla\log \rv p_t(\rv X_t) - \nabla\log \rv p_{t_{n-1}}(\rv X_{t_{n-1}})dt \right\Vert_{L_2} \leq \frac{\sqrt{2d}}{2}\frac{R^2}{\epsilon} \sqrt{h} \int_{t_{n-1}}^{t_n}\frac{1}{T-s}ds.
$$
Using Cauchy–Schwarz inequality, we have
$$
\left\Vert \int_{t_{n-1}}^{t_{n}} \nabla\log \rv p_t(\rv X_t) - \nabla\log \rv p_{t_{n-1}}(\rv X_{t_{n-1}})dt \right\Vert_{L_2}^2 \leq \frac{d}{2}\frac{R^4}{\epsilon^2}h^2 \int_{t_{n-1}}^{t_n}\frac{1}{(T-s)^2}ds.
$$
It follows that
\begin{multline*}
\Vert \rv X_{t_{n}} - \hat{X}_n\Vert_{L_2}^2 \leq L_{n-1}^2\Vert \rv X_{t_{n-1}} - \hat{X}_{n-1}\Vert_{L_2}^2 
+ \frac{d}{2}\frac{R^4}{\epsilon^2}h^2 \int_{t_{n-1}}^{t_n}\frac{1}{(T-s)^2}ds\\
+ \E\Big[\Big\langle\op I +h\nabla\log\rv p_{t_{n-1}}\cp(\rv X_{t_{n-1}}) - \op I + h\nabla\log \rv p_{t_{n-1}}\cp(\hat{X}_{n-1}),\\
\int_{t_{n-1}}^{t_n}\int_{t_{n-1}}^s \nabla^2\log \rv{p}_u(\rv X_u)\cdot dW_u ds\Big\rangle\Big].
\end{multline*}
We tackle the third term by noticing that, as $\rv X_{t_{n-1}}$ and $\hat X_{n-1}$ are $\F_{t_{n-1}}$-measurable,
\begin{multline*}
\E\Big[\Big\langle\op I +h\nabla\log\rv p_{t_{n-1}}\cp(\rv X_{t_{n-1}}) - \op I + h\nabla\log \rv p_{t_{n-1}}\cp(\hat{X}_{n-1}),\\
\int_{t_{n-1}}^{t_n}\int_{t_{n-1}}^s \nabla^2\log \rv{p}_u(\rv X_u)\cdot dW_u ds\Big\rangle\Big] \\
= \E\Big[\Big\langle\op I +h\nabla\log\rv p_{t_{n-1}}\cp(\rv X_{t_{n-1}}) - \op I + h\nabla\log \rv p_{t_{n-1}}\cp(\hat{X}_{n-1}),\\
\E\Big[\int_{t_{n-1}}^{t_n}\int_{t_{n-1}}^s \nabla^2\log \rv{p}_u(\rv X_u )\cdot dW_u ds\Big|\F_{t_{n-1}}\Big]\Big\rangle\Big] \\
=\E\ob\left\langle\op I +h\nabla\log\rv p_{t_{n-1}}\cp(\rv X_{t_{n-1}}) - \op I + h\nabla\log \rv p_{t_{n-1}}\cp(\hat{X}_{n-1}),0\right\rangle\cb = 0,
\end{multline*}
hence finally,
\begin{equation}
\label{eq:convergence_SDE_true_score:bound_step_n}
\Vert \rv X_{t_{n}} - \hat{X}_n\Vert_{L_2}^2 \leq L_{n-1}^2\Vert \rv X_{t_{n-1}} - \hat{X}_{n-1}\Vert_{L_2}^2  + \frac{d}{2}\frac{R^4}{\epsilon^2}h^2 \int_{t_{n-1}}^{t_n}\frac{1}{(T-s)^2}ds.    
\end{equation}

\noindent\textbf{Bounding the error at time $T-\epsilon.$} With (\ref{eq:convergence_SDE_true_score:bound_step_n}), by induction, we get that,
$$
\Vert \rv X_{T-\epsilon} - \hat{X}_N\ \Vert_{L_2}^2 
\leq \op\prod_{n=0}^{N-1}L_{{n}}\cp^2\Vert \rv X_{0} - \hat{X}_{0}\Vert_{L_2} ^2
+ \frac{d}{2}\frac{R^4}{\epsilon^2}h^2 \sum_{n=1}^{N}\op\prod_{m=n}^{N-1}L_{m}\cp^2\int_{t_{n-1}}^{t_n}\frac{1}{(T-s)^2}ds.
$$
As $(\rv X_{T-\epsilon},\hat{X}_N)$ above is a specific coupling between the two variables, we have $W_2(\L(\rv X_{T-\epsilon}),\L(\hat{X}_N)) \leq \Vert \rv X_{T-\epsilon} - \hat{X}_N\ \Vert_{L_2}$, and with (\ref{eq:convergence_SDE_true_score:init}), we get,
\begin{multline*}
W_2(\L(\rv X_{T-\epsilon}),\L(\hat{X}_N))^2 \leq \op\prod_{n=0}^{N-1}L_{{n}}\cp^2 W_2(\L(X_T),\L(\hat{X}_0))^2\\
+ \frac{d}{2}\frac{R^4}{\epsilon^2}h^2 \sum_{n=1}^{N}\op\prod_{m=n}^{N-1}L_{m}\cp^2\int_{t_{n-1}}^{t_n}\frac{1}{(T-s)^2}ds.
\end{multline*}

\noindent\textbf{Bounding the propagation of error.} With Lemma~\ref{lemma:bound_propagation_error:SDE}, we have
$$
\prod_{m=0}^{N-1}L_{{m}} \leq  \frac{2\epsilon}{T+h} \exp\op\frac{R^2}{\epsilon}\cp \leq  \frac{2\epsilon}{T} \exp\op\frac{R^2}{\epsilon}\cp,
$$
and
$$
\sum_{n=1}^{N}\op\prod_{m=n}^{N-1}L_{m}\cp^2\int_{t_{n-1}}^{t_n}\frac{1}{(T-s)^2}ds \leq \frac{4}{3\epsilon}\exp\op\frac{2R^2}{\epsilon}\cp,
$$
leading to 
$$
W_2(\L(\rv X_{T-\epsilon}),\L(\hat{X}_N))^2  \leq \op\frac{2\epsilon}{T} \exp\op\frac{R^2}{\epsilon}\cp\cp^2 W_2(\L(X_T),\L(\hat{X}_0))^2
+ \frac{2d}{3}\frac{R^4}{\epsilon^3}\exp\op\frac{2R^2}{\epsilon}\cp h^2. 
$$
Then, using that for $a,b\geq 0$, $\sqrt{a+b} \leq \sqrt{a}+\sqrt{b}$, we get that,
$$
W_2(\L(\rv X_{T-\epsilon}),\L(\hat{X}_N))  \leq \frac{2\epsilon}{T} \exp\op\frac{R^2}{\epsilon}\cp W_2(\L(X_T),\L(\hat{X}_0))
+ \sqrt{\frac{2d}{3}}\frac{R^2}{\epsilon^{3/2}}\exp\op\frac{R^2}{\epsilon}\cp h. 
$$

\noindent\textbf{Early stopping error.} With the triangular inequality and Lemma~\ref{lemma:bound_early_stopping_error}, we have 
\begin{align*}
W_2(\L(X),\L(\hat{X})) = W_2(\L(\rv X_T),\L(\hat{X}_N))
&\leq W_2(\L(\rv X_T),\L(\rv X_{T-\epsilon})) + W_2(\L(\rv X_{T-\epsilon}),\L(\hat{X}_N))  \\
&= W_2(\L(X_0),\L(X_{\epsilon})) + W_2(\L(\rv X_{T-\epsilon}),\L(\hat{X}_N)) \\
&\leq \sqrt{d\epsilon} + W_2(\L(\rv X_{T-\epsilon}),\L(\hat{X}_N)).
\end{align*}

\noindent\textbf{Initialization error.} Corollary~\ref{cor:init_error} gives, for $T$ large enough (depending only on $\L(X)$)
$$
W_2(\L(X_T),\L(\hat{X}_0))\leq \frac{R^2}{\sqrt{T}},
$$
hence finally,
\begin{equation*}
\tag{\ref{eq:bound_SDE}}
W_2(\L(X),\L(\hat{X})) \leq 
\sqrt{d\epsilon}
+\frac{2\epsilon}{T^{3/2}} \exp\op\frac{R^2}{\epsilon}\cp R^2
+\sqrt{\frac{2d}{3}}\frac{R^2}{\epsilon^{3/2}}\exp\op\frac{R^2}{\epsilon}\cp h.  
\end{equation*}

\section{Explosion of the Reverse ODE in Finite Time for Quadratically Perturbed Score}
\label{sct:explosion_ODE}
\begin{proposition}
\label{prop:explosion_ODE}
Assume that $X\in B(0,R)$ almost surely, and, for $\alpha > 0$, define :
$$
s(t,x) = \nabla\log p_t (x) + \alpha \Vert x\Vert x.
$$
Denote $\rv x_t$ the process defined by:
\begin{equation*}
\left\{\begin{array}{rl}
    \frac{d x_t}{dt} &= \frac{1}{2} s(T-t,x_t),\\
    x_{0} &= X_{T},
\end{array}\right.
\end{equation*}
and  $\tau \in [0,\infty]$ the random time of explosion, i.e., the stopping time such that for $\tau < \infty$, $\Vert x_t\Vert \xrightarrow[t\rightarrow \tau^-]{} \infty$ almost surely\footnote{This explosion time can always be defined, as one can take $\tau = \infty$ when there is no explosion.}. Then for all $\delta > 0$, $\P(\tau \leq \delta) > 0$.
\end{proposition}
\begin{proof}
We start by computing the derivative of $\Vert x_t\Vert^2$:
\begin{align*}
\frac{d}{dt}\Vert x_t\Vert^2 
&= 2 x_t \cdot \frac{d}{dt}x_t\\
&= \alpha \Vert x_t \Vert^3 + x_t \cdot \nabla\log p_{T-t} (x_t)\\
\text{(Lemma~\ref{lemma:score_moment})}&= \alpha \Vert x_t \Vert^3 + x_t\cdot \frac{1}{T-t}(\E[X|X_{T-t} = x_t]  - x_t)\\
&= \alpha \Vert x_t \Vert^3 -  \frac{\Vert x_t\Vert^2}{T-t} + \frac{x_t \cdot\E[X|X_{T-t} = x_t] }{T-t}.\\
\end{align*}
As $X\in B(0,R)$ almost surely, $x_t \cdot\E[X|X_{T-t} = x_t] \geq - \Vert x_t\Vert \Vert\E[X|X_{T-t} = x_t]\Vert \geq -R \Vert x_t\Vert$, hence, 
$$
\frac{d}{dt}\Vert x_t\Vert^2 \geq \alpha \Vert x_t \Vert^3 -  \frac{\Vert x_t\Vert^2}{T-t} - \frac{R \Vert x_t\Vert}{T-t}.
$$
We denote $y_t = \Vert x_t\Vert^2$, and we get that 
$$
\frac{d}{dt}y_t \geq 2\alpha y_t^{3/2} -  \frac{y_t}{T-t} - \frac{R \sqrt{y_t}}{T-t}.
$$
We fix $\epsilon >0$, then for $t \in [0, T-\epsilon]$, we have 
\begin{equation}
\label{eq:explosion_ODE:inequality_dy_dt}
\frac{d}{dt}y_t \geq \alpha y_t^{3/2} -  \frac{y_t}{\epsilon} - \frac{R \sqrt{y_t}}{\epsilon}. 
\end{equation}

\noindent Moreover, the usual computation on limits gives,
$$
\op \alpha y^{3/2} -  \frac{y}{\epsilon} - \frac{R \sqrt{y}}{\epsilon}\cp - \frac{\alpha}{2} y^{3/2} = \frac{\alpha}{2}y^{3/2} -  \frac{y}{\epsilon}- \frac{R \sqrt{y}}{\epsilon} \xrightarrow[y\rightarrow+\infty]{} +\infty,
$$
hence there is some $C > 0$ such that for all $y \geq C$, 
\begin{equation}
\label{eq:explosion_ODE:inequality_y}
\alpha y^{3/2} -  \frac{y}{\epsilon} - \frac{R \sqrt{y}}{\epsilon} \geq \frac{\alpha}{2} y^{3/2} > 0.
\end{equation}

\noindent Under the assumption that $y_0 \geq C$, this inequality, along with (\ref{eq:explosion_ODE:inequality_dy_dt}), ensures that, for $t \in [0, T-\epsilon]$, $y_t$ is increasing, $y_t \geq C > 0$ and 
$$
\frac{d}{dt}y_t \geq \alpha y_t^{3/2} -  \frac{y_t}{\epsilon} - \frac{R \sqrt{y_t}}{\epsilon} \geq \frac{\alpha}{2} y_t^{3/2}.
$$
Using a result by \citet{petrovitschManiereDetendreTheoreme1901} (a generalization of Grönwall's Lemma is the 1D case), we then deduce that for $t \in [0, T-\epsilon]$, $y_t \geq z_t$ with $z_t$ the solution to the ODE:
\begin{equation*}
\left\{\begin{array}{rl}
    \frac{d z_t}{dt} &= \frac{\alpha}{2} z_t^{3/2},\\
    z_{0} &= y_0.
\end{array}\right.
\end{equation*}
We solve this equation in explicit form with 
$$
z_t = \op\frac{1}{y_0^{-1/2} - \frac{\alpha}{4} t}\cp^2,
$$
In particular, $z_t$ explode in time $\tau_z = \frac{4}{\alpha \sqrt{y_0}}$, and as $y_t \geq z_t$, it follows that $\tau \leq \tau_z$.
Then for $\delta >0$,
\begin{align*}
\P(\tau \leq \delta)
&\geq \P(\tau \leq \min(\delta, T-\epsilon),y_0 \geq C)\\
&\geq \P(\tau_z \leq \min(\delta, T-\epsilon), y_0 \geq C)\\
&= \P\op\frac{4}{\alpha \sqrt{y_0}}\leq \max(\delta, T-\epsilon), y_0 \geq C\cp\\
&= \P\op y_0 \geq \max\op C, \op\frac{4}{\alpha\min(\delta, T-\epsilon)}\cp^2\cp\cp.
\end{align*}
Finally, as $y_0 = \Vert x_0\Vert^2=\Vert X_T\Vert^2$, and $X_T = X + B_T$ has a positive density over $\R^d$, $y_0$ has a positive density over $\R_+$, hence 
$$
\P(\tau \leq \delta) \geq \P\op y_0 \geq \max\op C, \op\frac{4}{\alpha\min(\delta, T-\epsilon)}\cp^2\cp\cp > 0.
$$
\end{proof}

\section{Parametrization of Diffusion Models}
\label{sct:param_time}
In this section, we look at other time parametrizations (noise schedules) for diffusion models. The objective of this section is to provide the appropriate tools for comparing the convergence guarantees of diffusion models independently of the parametrization, and to give indications on how to adapt our results to these other contexts and highlight what might or might not change.

Note also that we could try to adapt this work to flow matching or stochastic interpolants \citep{liuFlowStraightFast2023,lipmanFlowMatchingGenerative2023,albergoStochasticInterpolantsUnifying2023}.
When the source distribution is a Gaussian, flow matching corresponds to a particular time parametrization of diffusion models; hence, the tools developed here could be directly used.
Using a more general source distribution distribution will require adapting how we control of the regularity of the velocity fields, under proper assumptions on this distribution.

\subsection{Time Parametrization of Diffusion Models}
The forward process $X_t$ defined by (\ref{eq:sde_forward}) can be generalized as \citep[see, e.g.,][]{songScoreBasedGenerativeModeling2021,karrasElucidatingDesignSpace2022}
\begin{equation}
\label{eq:sde_forward_gen}
\left\{
\begin{array}{l}
dX_t =  -f(t)X_tdt + g(t)dB_t, \\
X_0 = X,
\end{array}\right.
\end{equation}
where $f,g : \R \mapsto\R_+$ are continuous function. This SDE has a closed-form solution,
$$
X_t =  a(t) X +  a(t)\int_0^t \frac{g(u)}{a(u)}dB_u
$$
where
$$
a(t) = \exp\op-\int_0^t f(u)du\cp
$$
In particular, for each time $t$, the marginal distribution $\L(X_t)$ verifies that 
$$
X_t = a(t) X + b(t) Z = s(t)(X + \sigma(t) Z).
$$
with $Z\sim \N(0,I)$ independent from $X$ and
\begin{align*}
b(t) &= a(t)\sqrt{\int _0^t \frac{g(u)^2}{a(u)^2}du},\\
s(t) &= a(t) =  \exp\op-\int_0^t f(u)du\cp,\\
\sigma(t) &= \frac{b(t)}{a(t)} = \sqrt{\int _0^t \frac{g(u)^2}{s(u)^2}du}.
\end{align*}
Note that we can also start for $a(t),b(t)$ or $s(t),\sigma(t)$ to define the corresponding functions $f,g$ by taking
\begin{align*}
f(t) &= - \frac{d}{dt}\log a(t) =- \frac{d}{dt}\log s(t) \\
g(t) &= a(t)\sqrt{\frac{d}{dt} \op a(t)^{-2} b(t)^2\cp} = s(t)\sqrt{\frac{d}{dt}\sigma(t)^2}.
\end{align*}

We can generalize the result of Propositions~\ref{prop:reverse_SDE} and \ref{prop:reverse_ODE}, i.e., for some $T>0$, define a reverse process $\rv X_t$ such that $\rv X_t \sim X_{T-t}$. Denoting $p_t$ the density of $X_t$ and $\rv p_t = p_{T-t}$, we take
$$
\left\{
\begin{array}{ll}
d\rv{X}_t &=  \op-\frac{\dot s(T-t)}{s(T-t)}\rv X_t + (1+\lambda)s(T-t)^2\dot \sigma(T-t)\sigma(T-t)\nabla\log \rv p_{t}(\rv{X}_t)\cp dt \\
&\quad+ \sqrt{\lambda} s(T-t)\sqrt{2\dot\sigma(T-t)\sigma(T-t)}dW_t, \\
\rv{X}_0 &= X_T,
\end{array}\right.
$$
with $W_t$ a Brownian motion. Note that $\nabla\log \rv p_t = \nabla \log p_{T-t}$ depends on the time parametrization, to make this dependency explicit, we define $p\op x; \sigma^2\cp$   the density of the normalized variable at noise level $\sigma^2$: $X + \N(0,\sigma^2I)$). We have $X_t \sim s(t)(X+\N(0,\sigma(t)^2I))$, hence
$$
 p_t(x) \propto p\op\frac{x}{s(t)};\sigma(t)^2\cp,
$$
leading to 
\begin{equation}
\label{eq:score_to_normalized_score}
\nabla\log p_t(x) = \frac{1}{s(t)} \nabla \log p\op\frac{x}{s(t)};\sigma(t)^2\cp.
\end{equation}
Then, the reverse process can be rewritten as:
\begin{equation}
\label{eq:reverse_process_gen}
\left\{
\begin{array}{ll}
d\rv{X}_t &=   d^\lambda (\rv X_t,T-t) dt + \sqrt{\lambda} s(T-t)\sqrt{2\dot\sigma(T-t)\sigma(T-t)}dW_t, \\
\rv{X}_0 &= X_T,
\end{array}\right.
\end{equation}
with 
\begin{align}
\label{eq:generalized_drift}
d^\lambda (x,T-t)
&= -\frac{\dot s(T-t)}{s(T-t)} x + (1+\lambda) s(T-t)\dot \sigma(T-t)\sigma(T-t)\nabla\log p \op \frac{x}{s(T-t)};\sigma(T-t)^2\cp\notag\\
&= f(T-t) x + \frac{1+\lambda}{2} \frac{g(T-t)^2}{s(T-t)}\nabla\log p \op \frac{x}{s(T-t)};\sigma(T-t)^2\cp.
\end{align}

For $s(t) = 1$ and $\sigma(t) = \sqrt{t}$, we recover our reverse ODE (\ref{eq:reverse_ODE}) with $\lambda = 0$ and  reverse SDE (\ref{eq:reverse_SDE})  with $\lambda = 1$. Algorithms are obtained similarly as before by taking discretization (Euler-Maruyama, Euler or Heun).

\subsection{Impact of Parametrization on Different Quantities}
\label{sct:param_time:changes}
\subsubsection{Score Regularity and Propagation of Errors}
We can adapt the results of Lemma~\ref{lemma:bound_spatial_regularity} in order to control the spacial regularity of the drift, which is essential to control how errors are propagated from steps to steps.
We have
$$
\nabla_x d^\lambda (x,t) = -\frac{\dot s(t)}{s(t)} I + (1+\lambda) \dot \sigma(t)\sigma(t)\nabla^2\log p \op \frac{x}{s(t)};\sigma(t)^2\cp.
$$
As with our original time parametrization, $\nabla \log p(x;t) = \nabla\log p_t(x)$, from Lemma~\ref{lemma:bound_spatial_regularity}, we deduce that
$$
- \frac{1}{\sigma^2} I\preccurlyeq  \nabla^2\log p(x;\sigma^2) \preccurlyeq  \op - \frac{1}{\sigma^2} + \frac{R^2}{\sigma^4}\cp I,
$$
hence
$$
\op -\frac{\dot s(t)}{s(t)}-  (1+\lambda) \frac{\dot \sigma(t)}{\sigma(t)} \cp I
\preccurlyeq 
\nabla d^\lambda (x,t)
\preccurlyeq  \op -\frac{\dot s(t)}{s(t)} + (1+\lambda) \op-\frac{\dot\sigma(t)}{\sigma(t)} + \frac{\dot \sigma(t)R^2}{\sigma(t)^3}\cp\cp I.
$$
We get that, for $h$ small enough,  $f_{t,h}: x \mapsto x + h d^\lambda (x,t)$,  $f_{t,h}$ is $L_{t,h}$-Lipchitz, with
$$
L_{t,h} = 1 + h \op -\frac{\dot s(t)}{s(t)} + (1+\lambda) \op-\frac{\dot\sigma(t)}{\sigma(t)} + \frac{\dot \sigma(t)R^2}{\sigma(t)^3}\cp\cp.
$$
This constant is important as it controls how error are propagated through iterations (see proofs of Propositions~\ref{prop:convergence_ODE_emp_score}-\ref{prop:convergence_SDE_true_score}). More precisely, error at step $n$ are propagated to the last step $N$ by being multiplied by $\prod_{m=n}^{N-1}L_{T-t_m,h}$ with $t_m = hm$.
We have given precise bounds for this quantity with our original time parametrization in Lemmas~\ref{lemma:bound_propagation_error:ODE}-\ref{lemma:bound_propagation_error:Heun}.
We give here the order of magnitude for the general case:
\begin{align}
\label{eq:propagation_coef_gen}
\prod_{m=n}^{N-1}L_{T-t_m,h} 
&= \prod_{m=n}^{N-1} \op 1 + h \op -\frac{\dot s(T-t_m)}{s(T-t_m)} + (1+\lambda) \op-\frac{\dot\sigma(T-t_m)}{\sigma(T-t_m)} + \frac{\dot \sigma(T-t_m)R^2}{\sigma(T-t_m)^3}\cp\cp\cp \notag\\
&\leq \exp\op h\sum_{m=n}^{N-1} \op -\frac{\dot s(T-t_m)}{s(T-t_m)} + (1+\lambda) \op-\frac{\dot\sigma(T-t_m)}{\sigma(T-t_m)} + \frac{\dot \sigma(T-t_m)R^2}{\sigma(T-t_m)^3}\cp\cp\cp\notag\\
&\approx \exp\op \int_\epsilon^T -\frac{\dot s(t)}{s(t)} + (1+\lambda) \op-\frac{\dot\sigma(t)}{\sigma(t)} + \frac{\dot \sigma(t)R^2}{\sigma(t)^3}\cp dt\cp \notag\\
&\leq \frac{1}{s(T)} \op \frac{\sigma(\epsilon)}{\sigma(T)}\cp^{1+\lambda}\exp\op\frac{(1+\lambda)R^2}{2\sigma(\epsilon)^2}\cp.
\end{align}

\subsubsection{Initialization Error}
The initialization $X_T \sim s(T)(X +\sigma(T)\N(0,I))$ is approximated by taking $\hat X_0 = s(T)\sigma(T)Z$ with $Z\sim \N(0,I)$. Proposition~\ref{prop:init_error} gives 
$$
W_2(\L(Y),\L(\hat{X}_0)) \leq s(T) W_2\op \L(X),\delta_0\cp = s(T) \E[\Vert X\Vert^2].
$$
Assuming moreover that $\E[X] = 0$, and that for some some $\xi>0$, $\E\ob e^{\xi X^2}\cb < \infty$ as $\sigma(T) \rightarrow+\infty$, we also have
$$
W_2(\L(Y),\L(\hat{X}_0)) \sim  \frac{s(T)}{2\sigma(T)} \left\Vert \Sigma\right\Vert_\textnormal{F}
$$
with $\Sigma = \E\ob XX^\top\cb$.

At first glance, it can seem that taking a time parametrization such that $s(T)\xrightarrow[T\rightarrow\infty]{} 0$ (for example $s(T) = e^{-T}$ for the Ornstein–Uhlenbeck process) will lead to a smaller initialization error.
This idea is incorrect because the gain from the scaling is exactly compensated by the multiplicative factor corresponding to the propagation of error that scales as $s(T)^{-1}$:
$$
\textnormal{propagated initialization error}  \lesssim \op \frac{\sigma(\epsilon)}{\sigma(T)}\cp^{1+\lambda}\exp\op\frac{(1+\lambda)R^2}{2\sigma(\epsilon)^2}\cp\E[\Vert X\Vert^2].
$$
This bound does not depend on a specific time parametrization $t\mapsto(s(t),\sigma(t))$, but only on the noise levels $(\sigma(\epsilon),\sigma(T))$.
Note however that our second bound in Proposition~\ref{prop:init_error} gives a better asymptomatic dependency in $\sigma(T)$ that the bounds found in the literature:
$$
\textnormal{propagated initialization error}  \lesssim \frac{1}{\sigma(T)}\op \frac{\sigma(\epsilon)}{\sigma(T)}\cp^{1+\lambda}\exp\op\frac{(1+\lambda)R^2}{2\sigma(\epsilon)^2}\cp \left\Vert \Sigma\right\Vert_\textnormal{F}.
$$

\subsubsection{Score Error}
The $L_2$-error on the score also depends on the time parametrization:
\begin{align*}
\ep_{\text{score}}(t) &= \Vert\nabla\log p_t(X_t) - f_\theta(X_t,t)\Vert_{L_2} \\
&= \frac{1}{s(t)}\left\Vert\nabla\log p\op \frac{X_t}{s(t)};\sigma(t)^2\cp - \td f_\theta\op \frac{X_t}{s(t)},\sigma(t)^2\cp\right\Vert_{L_2}\\
&= \frac{1}{s(t)}\left\Vert\nabla\log p\op X + \sigma(t)Z;\sigma(t)^2\cp - \td f_\theta( X + \sigma(t)Z,\sigma(t)^2)\right\Vert_{L_2}\\
&= \frac{1}{s(t)}\td\ep_{\text{score}}(\sigma(t)^2),
\end{align*}
where $\td f_\theta(\cdot, \sigma^2)$ is the normalized score predictor at noise level $\sigma^2$ defined by $\td f_\theta(x, \sigma(t)^2) = s(t) f_\theta(s(t)x,t)$.
Note that with our initial scaling, $s(t)=1$ and $\sigma(t) = \sqrt{t}$, we simply have $\ep_{\text{score}}(t) = \td\ep_{\text{score}}(t)$. 

Theoretical works on the convergence of diffusion models often assume that the quantity $\ep_{\text{score}}(t)$ is bounded. However, to compare works using different time parametrization, it is important to translate these hypotheses in term of $\td\ep_{\text{score}}(\sigma(t)^2)$. Typically, a bound of the form $\forall t, \ep_{\text{score}}(t) \leq \alpha$ corresponds to $\td\ep_{\text{score}}(\sigma(t)^2)\leq s(t) \alpha$.

In the proofs, what matters is the error on the drift coefficient of the process at time $t_n$, which can be expressed using the error on the score:
\begin{align*}
\ep_\textnormal{approximation}(t_n) 
&= \Vert d^\lambda (\rv X_t,T-t_n) - \hat d^\lambda (\rv X_t,T-t_n)\Vert_{L_2}\\
&= (1+\lambda) s(T-t_n)\dot \sigma(T-t_n)\sigma(T-t_n) \td\ep_{\text{score}}(\sigma(t_n)^2)
\end{align*}

Once again, it can seem at first glance that some time parametrizations $t\mapsto(s(t),\sigma(t))$ lead to smaller approximation errors. However, when adding these errors along steps and accounting from the multiplicative propagation coefficient (\ref{eq:propagation_coef_gen}), we get an integral over noise levels that does not depend on the specific time parametrization:
\begin{align*}
\textnormal{propagated }&\textnormal{approximation error}\\
&\lesssim h\sum_{n=0}^{N-1}\op \frac{\sigma(\epsilon)}{\sigma(T-t_n)}\cp^{1+\lambda}\exp\op\frac{(1+\lambda)R^2}{2\sigma(\epsilon)^2}\cp\dot \sigma(T-t_n)\sigma(T-t_n) \td\ep_{\text{score}}(\sigma(t_n)^2)\\
&\approx \int_\epsilon^T \op \frac{\sigma(\epsilon)}{\sigma(t)}\cp^{1+\lambda}\exp\op\frac{(1+\lambda)R^2}{2\sigma(\epsilon)^2}\cp\dot \sigma(t)\sigma(t) \td\ep_{\text{score}}(\sigma(t)^2)dt\\
&= \frac{1}{2}\int_{\sigma(\epsilon)^2}^{\sigma(T)^2} \op \frac{\sigma(\epsilon)}{\sqrt{\sigma^2}}\cp^{1+\lambda}\exp\op\frac{(1+\lambda)R^2}{2\sigma(\epsilon)^2}\cp\td\ep_{\text{score}}(\sigma^2)d\sigma^2.
\end{align*}

\subsubsection{Discretization Error}
Similarly to section~\ref{sct:discretization_error}, we can control the discretization error by looking at the time derivative of the drift $d^\lambda(\rv X_T,T-t)$.
As before, we can express derivatives of $\nabla \log p(x;\sigma^2)$ in term of conditional moments of the distribution, all of which can be controlled under Assumption 1 or 1'. But looking at the expression of $d^\lambda(\rv X_T,T-t)$ (\ref{eq:generalized_drift}), we see that it will also be needed to control time derivation of $s(t)$ and $\sigma(t)$ (or equivalently of $f,g$) which will be possible under the proper assumptions on the forward process.
However, inspection of equation (\ref{eq:generalized_drift}) reveals that computing this derivative entails differentiating compositions and products, resulting in a substantial increase in the number of terms. While each of these terms can, in principle, be controlled using the tools developed here, deriving explicit bounds in the general case lies beyond the scope of the present work and is deferred to future research.

\vskip 0.2in
\bibliography{biblio.bib}

\end{document}